%% file: V4_andrea.tex
\documentclass{article}
\usepackage[utf8]{inputenc}
\input{macro.tex}

\title{Adversarial Examples in Random Neural Networks\\
 with General Activations}
\author{
	Andrea Montanari\footnotemark[2]
	\thanks{Department of Electrical Engineering, Stanford University} 
	\;\;
	and  
	\;\;
	Yuchen Wu\thanks{Department of Statistics, Stanford University}
}
\date{March 31, 2022}
\begin{document}
\maketitle

\begin{abstract}
	A substantial body of empirical work documents the lack of robustness in
	deep learning models to adversarial examples. 
	Recent theoretical work proved that adversarial examples are ubiquitous in 
	two-layers networks with sub-exponential width and ReLU or smooth
	activations, and multi-layer ReLU networks with sub-exponential width. 
	We present a result of the same type, with no restriction on width and 
	for general locally Lipschitz continuous activations. 
	
	More precisely, given a neural network $f(\,\cdot\,;\btheta)$ with random weights $\btheta$,
	and feature vector $\bx$, we show that an adversarial example
	$\bx'$ can be found with high probability along the direction of the gradient 
	$\nabla_{\bx}f(\bx;\btheta)$.
	Our proof is based on a Gaussian conditioning technique. Instead of proving that $f$
	is approximately linear in a neighborhood of $\bx$, we characterize the joint distribution of
	$f(\bx;\btheta)$ and $f(\bx';\btheta)$ for $\bx' = \bx-s(\bx)\nabla_{\bx}f(\bx;\btheta)$, where $s(\bx) = \sign(f(\bx; \btheta)) \cdot s_d$ for some positive step size $s_d$. 
\end{abstract}

\section{Introduction}

The output of a neural network at test time can be significantly
changed by an imperceptible but carefully chosen perturbation of its input.
Such perturbed inputs are referred to as \emph{adversarial examples}. 
In the context of deep learning, 
the existence  of adversarial examples was first discovered experimentally in \cite{szegedy2013intriguing}. 
A rapidly expanding literature developed algorithms to produce adversarial
examples \cite{goodfellow2014explaining,carlini2018audio,kos2018adversarial,papernot2017practical,yuan2019adversarial},
as well as techniques to increase model robustness
\cite{kurakin2016adversarial,tramer2017ensemble,qin2019adversarial,wong2020fast,
      salman2019provably,bui2022unified}.
      
Throughout this paper, we will focus on the standard supervised learning setting,
whereby a data sample takes the form $(\bx,y)$, with $\bx\in\reals^d$ a covariates vector
and $y\in\reals$ the corresponding label. A model is a function
$f(\,\cdot\,;\btheta):\reals^d\to \reals$ parametrized by weights $\btheta\in\reals^p$.  
In this setting, given a test point $\bx\in\reals^d$, an adversary constructs
$\bx^{\sadv}=\bx^{\sadv}(\bx;\btheta)\in \reals^d$. The adversary is successful if, with high probability
\begin{align}
\sign\big(f(\bx^{\sadv};\btheta)\big) = -\sign\big(f(\bx;\btheta)\big)\, , 
\;\;\;\;\;\;\;\;\; \|\bx^{\sadv}-\bx\|\ll \|\bx\|\, .
\end{align}
In this paper we will interpret `with high probability' as
with probability converging to one as $d\to\infty$ with respect to a certain
distribution over the random weights $\btheta$,
for any fixed $\bx\in\reals^d$ (with normalization $\|\bx\|=\sqrt{d}$).
 Results in the literature differ in the choice of the point $\bx$ 
(e.g., random according to the test distribution, or an arbitrary training point, or a fixed $\bx$
as in the present paper),
and the norm $\|\,\cdot\,\|$ (empirical work often adopts $\ell_{\infty}$ norm, but we will follow
earlier theoretical papers and use $\ell_2$ norm.). We refer to the next sections
for formal statements.

Among the earlier hypotheses about the origins and ubiquity of adversarial examples
was the idea, put forward in \cite{goodfellow2014explaining}, that they are related to the fact that
$f(\,\cdot\,;\btheta)$ is approximately linear (better, affine) over large regions of the input space. 
This hypothesis has several consequences that match empirical observations at a qualitative level:
\begin{enumerate}
 \item Prevalence of adversarial examples. Indeed, if   $f(\bx;\btheta)\approx a(\btheta)+\<\bb(\btheta),\bx\>$,
 then 
 \begin{align}
 f(\bx^{\sadv};\btheta)- f(\bx;\btheta) \approx \<\bb(\btheta),\bx^{\sadv}-\bx\>\, .
 \label{eq:ApproxLinearity}
 \end{align}
 Assuming without loss of generality $a(\btheta) = 0$, we have $|f(\bx;\btheta)|=\Theta(\|\bb(\btheta)\|_2)$
 for most $\|\bx\|_2=\sqrt{d}$. By choosing $\bx^{\sadv}-\bx = \pm \bb(\btheta)/\|\bb(\btheta)\|_2$,
 we obtain that $|f(\bx^{\sadv};\btheta)- f(\bx;\btheta)|$ is of order    $|f(\bx;\btheta)|$,
 while $\|\bx^{\sadv}-\bx\|_2=1\ll \|\bx\|_2$. With appropriate choice of sign and stepsize, such perturbation also flips the sign of $f$. 
 
\item Adversarial examples can be found by efficient algorithms. 
Indeed, the above argument suggests to take
\begin{align}
\label{eq:FirstFGSM}
\bx^{\sadv}(\bx;\btheta) = \bx-s(\bx)\nabla_{\bx}f(\bx;\btheta)\, ,
\end{align}
for a suitable $s(\bx)$. This approach was successfully implemented in 
\cite{goodfellow2014explaining}, who referred\footnote{The original proposal
attempted to minimize $\|\bx^{\sadv}-\bx\|_{\infty}$, and consequently takes
$\bx^{\sadv}-\bx\propto \sign(\nabla_{\bx}f)$.} 
to it as the `fast gradient sign method'
(FGSM).
\end{enumerate}
The main result of this paper is a proof that this procedure indeed 
produces adversarial examples when $f(\bx;\btheta)$ is a two-layer or multi-layer
fully connected neural network with random weights. This can be interpreted 
as the function implemented by the network at initialization.

Several groups obtained theoretical results on the existence of adversarial examples. 
One basic remark is that, if the distribution of the covariates $\bx$ satisfies
an isoperimetry property, and $\P_{\bx}(f(\bx;\btheta)>0)\in [\delta,1-\delta]$
for some constant $\delta>0$, then a random $\bx$ will be close to the decision boundary 
(and to an adversarial example) with high probability. This is the case
---for instance--- when $\bx$ is a uniformly random vector on a high-dimensional
sphere, or a standard Gaussian vector. 
Increasingly sophisticated incarnations of this argument were given in 
\cite{gilmer2018adversarial,shafahi2018adversarial,fawzi2018adversarial}.

The isoperimetry argument clarifies why adversarial examples are ubiquitous, but
does not explain why they can be found so easily, for instance via FGSM.
In the other direction, \cite{bubeck2019adversarial} proved that learning robust classifiers
can be computationally hard.

A somewhat different point of view was developed in \cite{ilyas2019adversarial},
which proposed that non-robustness is related to the presence of 
non-robust features in the data. These functions $h(\bx)$ of the data are used by a normally trained
classifier (minimizing the training error), but can be significantly changed by an imperceptible
perturbation of $\bx$.  By itself, this is not incompatible with the `approximate linearity' 
hypothesis described above. However, \cite{ilyas2019adversarial} emphasized the 
existence of robust features alongside non-robust ones.

Our work is most closely related to a recent sequence of  papers
analyzing the brittleness of fully connected neural networks to the FGSM-style attack
\eqref{eq:FirstFGSM}
\cite{daniely2020most,bubeck2021single,bartlett2021adversarial}. 
In particular, \cite{daniely2020most} showed that random ReLU networks are vulnerable 
if the width of each layer is small relative to  
 the width of the previous layer. For the case of two-layer networks, this result was improved in
  \cite{bubeck2021single} which considered either smooth activations
and  width subexponential in the input dimension $m = \exp(o(d))$, or ReLU activations, 
and width $m\le \exp(d^{0.24})$. 
Finally, \cite{bartlett2021adversarial} generalized the latter analysis to 
multi-layer networks with maximal width $m \le  \exp(d^c)$ for some small $c$.

We also point to the recent paper \cite{vardi2022gradient} which studied trained
two-layer ReLU networks (under the assumption that gradient flow converges to 
a network that perfectly classifies the training set). These are shown to be non-robust (in a stronger
sense than above) for sample size $n\le \sqrt{d}$.

In this paper, we prove that the FGSM-like attack \eqref{eq:FirstFGSM}
indeed finds adversarial examples for neural networks with random Gaussian 
weights. We present the following novel contributions, with respect to earlier work:
\begin{description}
\item[Arbitrary width.] Our results apply to arbitrary diverging width, without upper 
bounds on the growth rate. This question posed as an open problem in  \cite{bubeck2021single}
and is not merely academic. A large body of literature connects wide random
neural networks to Gaussian processes and kernel methods, see 
\cite{neal1994priors,lee2019wide,arora2019exact,bartlett2021deep} for a few pointers.
For instance \cite{mei2022generalization,montanari2022interpolation} prove that the generalization
properties of two-layer networks linearized around their initialization approach
the one of the associated infinite-width kernel as soon as the
number of parameters become larger than the number of
samples.

Within this context, the upper bounds on width assumed
in \cite{bubeck2021single,bartlett2021adversarial} are somewhat puzzling. 
A priori, they could suggest that exponentially wide networks are more robust 
than sub-exponentially wide, although their generalization properties are similar.
Here we prove that this is not the case. In particular, our results apply to Gaussian 
processes as well.   
\item[General activation.] Both for two-layer and multi-layer networks, our proof 
applies for a general class of activation functions $\sigma(x)$. We only require 
$\sigma'(x)$ to exist almost everywhere, continuous, and bounded by a polynomial. 

While most activation function of practical use are more regular than this (e.g., Lipschitz continuous),
this generalization clarifies that the approximate linearity 
property \eqref{eq:ApproxLinearity} is not a naive consequence of the smoothness of the 
activation functions.  
\item[Weak linearity condition.] Our proofs are based on a weaker notion of linearity than 
\cite{daniely2020most,bubeck2021single,bartlett2021adversarial}. Namely,
instead of proving that $f(\,\cdot\,;\btheta)$ is approximately linear in a neighborhood of $\bx$,
we only prove that it is approximately linear along
 the direction of interest $\nabla_{\bx} f(\bx;\btheta)$. 
 \item[Gaussian conditioning.] Establishing approximate linearity along a specific direction
 poses an obvious mathematical challenge: The direction $\nabla_{\bx} f(\bx;\btheta)$ 
 is correlated with the function $f(\,\cdot\,;\btheta)$ itself. We deal with this difficulty 
 by introducing a Gaussian conditioning technique that was not used before in this context,
 and we believe can be useful to study other attacks.
\end{description}

For clarity of exposition, we will treat separately the two-layer 
and multi-layer cases. This allows the reader to understand the proof
strategy in a simpler example, before diving into the notational intricacies of multi-layer
networks.  In the case of two-layer networks, we prove two theorems. The first one is 
stated in Section \ref{sec:2layer}, and 
establishes that the attack succeeds with probability converging to one,
but do not provide explicit probability bounds. On the other hand, this theorems holds for very 
general activation functions. We present the proof of these results in \cref{sec:2layer-proof}. We then state a complementary result in Section
\ref{sec:NonAsymp} which explicit non-asymptotic probability bounds, but limited
to Lipschitz activations. 
The result for multi-layer network is stated in Section \ref{sec:multilayer}
with proofs in Section \ref{sec:multilayer-proof}.
Several technical lemmas are deferred to the appendices.

\subsection{Notations}

We generally use lower case letters for scalars, lower case bold for vectors and
upper case bold for matrices. The ordinary scalar product of vectors $\bu,\bv\in\RR^n$ is
denoted by $\<\bu,\bv\>$, and we let $\|\bu\|_2:= \<\bu,\bu\>^{1/2}$. 
For $n \in \NN_+$, we let $[n] = \{1,2,\dots,n\}$. We denote by 
$\plim$ convergence in probability. For random variables $X, Y$, we denote by $X \perp Y$ if
 $X$ and $Y$ are independent of each other. We denote by $\plim$ convergence in probability.

\section{Random two-layer networks}
\label{sec:2layer}

We begin by studying the following random two-layer neural network:
\begin{align}
	f(\bx;\btheta) = \sum_{i = 1}^m a_i \sigma(\bw_i^{\top} \bx).\label{eq:TwoLayerFirst}
\end{align}
Here, $\sigma: \RR \to \RR$ is a fixed activation function, 
$\btheta = \{(a_i,\bw_i)\}_{i\le m}$,
the weight vectors $\bw_i \in \RR^d$ are i.i.d. generated from $\normal(\textbf{0}, \id_d / d)$, 
and $(a_i)_{i\le m} \iidsim \normal(0,1/m)$. We denote by $\bW \in \RR^{m \times d}$ the random weight
 matrix with the $i$-th row equal to $\bw_i$, $\ba \in \RR^m$ the vector with the $i$-th coordinate 
 equal to $a_i$. We assume that $\bW$ is independent of $\ba$. 
 In what follows, we shall typically drop the argument $\btheta$ from $f$,
 and write $f(\bx) = f(\bx;\btheta)$.

Let $\tau := \sign(f(\bx))$, and $s_d \in \RR^+$ be the step size which depends on $(m,d)$. We define
\begin{align*}
	\bx^s := \bx - \tau s_d \nabla f(\bx).\label{eq:Xs}
\end{align*} 
Our main result on two-layer networks establishes
 that there exists a sequence of step sizes $\{s_d\}_{d \geq 1}$, such that for 
 $d, m=m(d) \rightarrow \infty$, $\|\bx - \bx^s\|_2 / \|\bx\|_2 \toP 0$, while with high 
 probability $\sign(f(\bx)) \neq \sign(f(\bx^s))$. 
\begin{theorem}\label{thm:2-layer}
	Let $\bx \in \RR^d$ be a deterministic vector with $\|\bx\|_2 = \sqrt{d}$. 
	Assume that $\sigma(x)$ is not a constant, $\sigma$ is continuous, almost everywhere 
	differentiable, $\sigma'$ is almost everywhere continuous, 
	and $|\sigma'(x)| \leq C_{\sigma}(1 + |x|^{k - 1})$, where
	$k \in \NN_+$ is a fixed positive integer and $C_{\sigma} > 0$ is a constant 
	depending only on $\sigma$. 
	
	Then the following hold:
%
	\begin{enumerate}
		\item There exists a constant $C > 0$ 
		depending only on $\sigma$, such that for any $\delta \in (0,1)$, with probability 
		at least $1 - \delta$, 
		\begin{align*}
			\frac{\|\bx - \bx^s\|_2}{\|\bx\|_2} \leq \frac{Cs_d}{\sqrt{d}} (1 + d^{-1/2} \log(1 / \delta))(1 + (m\delta)^{-1/2}).
		\end{align*}
		\item Let $\{\xi_d\}_{d \in \NN_+} \subseteq \RR^+$ be an increasing sequence 
		such that $\xi_d \to \infty$ as $d \to \infty$.
		Then there exists $\{s_d\}_{d \in \NN_+} \subseteq \RR^+$, such that $s_d \leq \xi_d$
		and the following hold: 
		\begin{align}
	\plim_{m,d\to\infty}\frac{\|\bx - \bx^s\|_2}{\|\bx\|_2} =  0\, ,\;\;\;	
	\lim_{m,d\to\infty}\P(\sign(f(\bx)) \neq \sign(f(\bx^s))) = 1\, .
\end{align}
	\end{enumerate}
\end{theorem}

\begin{remark}
Note that this theorem provides a completely quantitative non-asymptotic 
upper bound on the size of the perturbation $\|\bx - \bx^s\|_2$. On
the other hand, it does not provide convergence rates for the success probability 
$\P(\sign(f(\bx)) \neq \sign(f(\bx^s)))$. 

This can be traced to the use of a non-quantitative uniform central limit theorem 
in our proof (see below). 
In Section \ref{sec:NonAsymp} we obtain an explicit rate by a more careful handling of that
 step, at the price of assuming Lipschitz activations.
\end{remark}

\begin{remark}
The scale of the input is chosen so that the output after
the first layer $\sigma(\bw_i^{\top}\bx)$ are of order one when 
$\bw_i\sim\normal(\textbf{0}, \id_d / d)$.
As a consequence, the output after the second layer $f(\bx;\btheta)$ is also of order one for
$a_i \iidsim \normal(0,1/m)$.
	
 The proof of the theorem applies with barely any change to all inputs satisfying 
 $c \sqrt{d} \leq \|\bx\|_2 \leq C \sqrt{d}$ for some positive constants $c, C$. Here, we 
 choose not to state this general version to simplify notations. 

We note that the scaling is unimportant when the activation function is positively homogeneous 
(that is to say, for $z > 0$, $\sigma(zx) = z \sigma(x)$). 
On the other hand, for general activations the sensitivity 
to input perturbations is necessarily dependent of the scale.
For instance, if $\sigma(t) = (t - 1)_+$ and $\|\bx\|_2 \ll \sqrt{d}$
then, with high probability we have $|\bw_i^\top\bx|\ll 1$ for all $i\le m$.
In other words, the inputs to the hidden neurons lie in the region in which the 
activation vanishes, and therefore a small 
perturbation will not change the network output.
\end{remark}

\subsection{Proof technique}

As mentioned in the introduction, our proofs are based on Gaussian calculus. 
Before dwelling into the actual calculation, it is perhaps useful to describe a 
simple calculation along the same lines. Let $\bx\in\reals^d$, $\|\bx\|_2=\sqrt{d}$ 
be fixed and the adversarial example be given by Eq.~\eqref{eq:Xs}, with
$f(\bx) = f(\bx;\btheta)$, the two-layer network of Eq.~\eqref{eq:TwoLayerFirst}. 

The gradient of $f$ is given by
\begin{align}
	\nabla f(\bx) = \bW^{\top}\bD_{\sigma} \ba,\label{eq:Gradient2Layer}
\end{align}
where $\bD_{\sigma} \in \RR^{m \times m}$ is a diagonal matrix: $\bD_{\sigma} = 
\diag(\{\sigma'(\bw_i^{\top} \bx)\}_{i \leq m})$. 

Note that $\nabla f(\bx)$ is a random vector, because $\bW,\ba$ are random.
It is clearly useful to understand the distribution of this random vector: this will tell 
us about the properties of the adversarial perturbation. 
It is elementary that, if $\bb$ is a random vector independent of $\bW$, 
then conditional on the norm $\|\bb\|_2$, $\bW^{\top}\bb$ is a Gaussian random vector:
\begin{align}
\bW^{\top}\bb\big|_{\|\bb\|_2}\sim\normal(\mathbf{0}, (\|\bb\|_2^2/d) \cdot \id_d)\,. 
\end{align}
Equivalently, we can express the same fact by saying 
that $\bW^{\top}\bb = \|\bb\|_2 \bz/\sqrt{d}$, where $\bz$ is a standard Gaussian vector 
that is independent of 
$\bb$.

This fact might suggest that 
\begin{align}
\nabla f(\bx) \approx \bz \cdot \frac{1}{\sqrt{d}}\big\|\bD_{\sigma} \ba\big\|_2 \, ,\label{eq:Explanation}
\end{align}
where $\bz$ is a standard Gaussian vector 
that is independent of  everything else. Since (as is easy to see)
$\big\|\bD_{\sigma} \ba\big\|_2$ concentrates, this would further imply that the gradient is
approximately Gaussian with i.i.d. entries, whose variances we can easily compute.

However, the implication is not straightforward
because in this case, $\bb = \bD_{\sigma} \ba$ is not independent of $\bW$
(because the diagonal elements of $\bD_{\sigma}$ are $\sigma'(\bw_i^{\top}\bx)$ and are therefore 
a deterministic function of $\bW$).
As a consequence, Eq.~\eqref{eq:Explanation} is at best an approximate equality, and quantifying 
the error requires an argument.

Luckily, the Gaussian distribution allows for a particularly elegant such argument.
Let  $\Pi_{\bx} \in \RR^{d \times d}$ denote the orthogonal projector onto the linear space spanned by
 $\bx$, $\Pi_{\bx} =\bx\bx^{\top}/d$ and  $\Pi_{\bx}^{\perp} := \id_d - \Pi_{\bx}$.
 Further, define $\bg = \bW\bx$ (the input first-layer neurons). Then we have
\begin{align*}
\bW &= \bW\Pi_{\bx}  + \bW\Pi_{\bx}^\perp\\
 &= \frac{1}{d}\bg\bx^{\top}+ \bW\Pi_{\bx}^\perp\, .
 \end{align*}
 This decomposition has several useful properties: $(i)$~$\bW\Pi_{\bx}^\perp$ is independent
 of $\bg$ (because two orthogonal projections of a standard normal are independent); 
 $(ii)$~$\bD_{\sigma}$ is a function uniquely of $\bg$; $(iii)$~the distribution of $\bg$
 is simple, namely $\bg\sim\normal(0,\id_d)$. 
 
 Using this decomposition, we obtain
\begin{align}
\nabla f(\bx) &= \frac{1}{d}\bx\bg^\top\bD_{\sigma}\ba+ \Pi^{\perp}_{\bx}\bW^{\top}
\bD_{\sigma}\ba\nonumber\\
&\stackrel{(*)}{=} \alpha_{\parallel} \bx + \alpha_{\perp}  \Pi^{\perp}_{\bx}\bz\, ,
\label{eq:GradientDecomposition}
\end{align}
where $\bz\sim\normal(\mathbf{0},\id_d)$ is independent of $\bg$ and
\begin{align*}
\alpha_{\parallel}:= \frac{1}{d}\sum_{i=1}^m g_i a_i\sigma'(g_i)\, ,\;\;
\alpha_{\perp}^2 :=  \frac{1}{d}\sum_{i=1}^m a_i^2\sigma'(g_i)^2\, .
\end{align*}
Note that the crucial step $(*)$ is correct because $\Pi^{\perp}_{\bx}\bW^{\top}$
is independent of $\bg$ as discussed above. 
By the law of large numbers, both $\alpha_{\parallel}=O_P(1/d)$ and $\alpha_{\perp}$
concentrates around its expectation, which is of order $1/\sqrt{d}$.
Therefore, Eq.~\eqref{eq:GradientDecomposition} provides an exact version of \cref{eq:Explanation}.

The basic intuition in the decomposition \eqref{eq:GradientDecomposition}
is quite simple. Even if $\bD_{\sigma}$ depends on $\bW$, it only depends on a low-dimensional 
projection of this matrix. We can condition on this projection, and 
resample the orthogonal component of $\bW$ independently from it.

Our proofs push forward the same type of reasoning. 
Instead of computing the distribution of $\nabla f(\bx)$, we now have to compute the
distribution of $f(\bx^s) = f(\bx-\tau s_d\nabla f(\bx))$. By conditioning on
suitable projections of $\bW$, we write this quantity (better, its difference from the 
first order Taylor approximation)
in terms of a certain numbers of empirical averages similar to $\alpha_{\parallel}$,
$\alpha_{\perp}$ above. Thanks to  such representations, 
the proofs of our  main theorems reduce to controlling these empirical averages
and this can be achieved by standard empirical process theory.

\section{Proof of \cref{thm:2-layer}}
\label{sec:2layer-proof}

As discussed above, our proof strategy is based on conditioning on 
low-dimensional projections of $\bW$. 
We state a  Gaussian conditioning lemma that will 
be used repeatedly throughout the paper. 

We say that $\bY$ depends on $\bX$ only through
 $g(\bX)$ if and only if there exists a deterministic function $h$ and a random vector 
 $\bZ$ that is independent of $\bX$, such that $\bY = h(g(\bX), \bZ)$.
\begin{lemma}\label{lemma:gaussian-conditioning}
	Let $\bX \in \RR^{m \times d}$ be a matrix with i.i.d. standard Gaussian entries, and 
	$\bA_1 \in \RR^{k_1 \times m}, \bA_2 \in \RR^{d \times k_2}$ be other random matrices. 
	Let $\bY = h_1(\bA_1 \bX, \bX \bA_2, \bZ_1)$ and $\bA_2 = h_2(\bA_1 \bX, \bZ_2)$ for 
	deterministic functions $h_1$ and $h_2$.
	 Further assume that $(\bX, \bA_1, \bZ_1, \bZ_2)$ are mutually independent.
	  Then there exists $\tilde{\bX} \in \RR^{m \times d}$ which has the same distribution with $\bX$ and is independent of $\bY$, such that
\begin{align*}
	\bX = \Pi_{\bA_1}^{\perp}\tilde{\bX} \Pi_{\bA_2}^{\perp} + \Pi_{\bA_1}^{\perp}{\bX} \Pi_{\bA_2} + \Pi_{\bA_1}{\bX} \Pi_{\bA_2}^{\perp} + \Pi_{\bA_1}{\bX} \Pi_{\bA_2},
\end{align*} 
	where $\Pi_{\bA_1} \in \RR^{m \times m}$, $\Pi_{\bA_2} \in \RR^{d \times d}$ are 
	 the orthogonal projectors onto the subspaces
	 spanned by the rows of $\bA_1$, $\bA_2$, respectively. 
	Further, $\Pi_{\bA_1}^{\perp} := \id_m - \Pi_{\bA_1}$, $\Pi_{\bA_2}^{\perp} := \id_d - \Pi_{\bA_2}$. 
\end{lemma}	
The proof of \cref{lemma:gaussian-conditioning} is a straightforward application of the
 properties of Gaussian ensembles, and we defer it to 
 Appendix \ref{sec:proof-of-lemma:gaussian-conditioning}.

\subsubsection*{Proof of the first claim}  

We first prove claim 1 of the theorem. Recall that the gradient of $f$ is given by 
Eq.~\eqref{eq:Gradient2Layer}.
The next lemma provides non-asymptotic control over the Euclidean norm of $\nabla f(\bx)$. 
\begin{lemma}\label{lemma:control-gd-norm}
	Under the conditions of \cref{thm:2-layer}, there exists a constant $C > 0$ that depends only on $\sigma$, such that for any $\delta > 0$, with probability at least $1 - \delta$, we have
	\begin{align*}
		\|\nabla f(\bx)\|_2 \leq C (1 + d^{-1/2} \log(1 / \delta))(1 + (m\delta)^{-1/2}).
	\end{align*}
\end{lemma}
The proof of \cref{lemma:control-gd-norm} is deferred to  \cref{sec:proof-of-lemma-control-gd-norm} in the appendix. Recall that $\|\bx\|_2 = \sqrt{d}$, thus the first claim of the theorem follows directly from \cref{lemma:control-gd-norm}.

\subsubsection*{Proof of the second claim}  

We next invoke \cref{lemma:gaussian-conditioning} to prove the second claim of \cref{thm:2-layer}. 
For the sake of simplicity, we define $\bg := \bW \bx$, $\bg^s := \bW \bx^s$, then $\bg \overset{d}{=} \normal(\mathbf{0}, \id_m)$.  By definition, we have
	\begin{align}\label{eq:1}
		\bg^s = &\ \bg - \tau s_d \bW \nabla f(\bx) \nonumber \\
		      = &\ \bg - \tau s_d \bW \bW^{\top} \bD_{\sigma} \ba.
	\end{align}
Recall that  $\Pi_{\bx} \in \RR^{d \times d}$ is the orthogonal projector onto the linear subspace spanned by $\bx$, and let $\Pi_{\bx}^{\perp} := \id_d - \Pi_{\bx}$. Using \cref{lemma:gaussian-conditioning}, we can decompose the weight matrix $\bW$ as $\bW =  \bg \bx^{\top} / d + \tilde\bW\Pi_{\bx}^{\perp}$, where $\tilde\bW$ has the same marginal distribution as $\bW$ and is independent of $(\bg, \ba)$. We then substitute this result into \cref{eq:1}, which gives 
	\begin{align}\label{eq:decomp-gs}
		\bg^s = & \ \bg - \tau s_d(\bg \bg^{\top} / d + \tilde\bW\Pi_{\bx}^{\perp} \tilde\bW^{\top}) \bD_{\sigma} \ba \nonumber\\
		= & \ \bg(1 - \tau s_d \bg^{\top} \bD_{\sigma} \ba / d) - \tau s_d \bar\bW \bar\bW^{\top} \bD_{\sigma} \ba \nonumber \\
		{=} & \ \bg(1 - \tau s_d \bg^{\top} \bD_{\sigma} \ba / d) - \tau s_d \Pi_{\bD_{\sigma} \ba}\bar\bW \bar\bW^{\top} \bD_{\sigma} \ba - \tau s_d \Pi_{\bD_{\sigma} \ba}^{\perp}\bar\bW_c \bar\bW^{\top} \bD_{\sigma} \ba \nonumber \\
		= & \ \bg(1 - \tau s_d \bg^{\top} \bD_{\sigma} \ba / d)  -\tau s_d \cdot \frac{\|\bar\bW^{\top} \bD_{\sigma} \ba\|_2^2 - \langle \bar{\bW}_c^{\top} \bD_{\sigma} \ba, \bar\bW^{\top} \bD_{\sigma} \ba \rangle}{\|\bD_{\sigma} \ba\|_2^2} \bD_{\sigma} \ba -  \frac{s_d}{\sqrt{d}} \|\bar\bW^{\top} \bD_{\sigma} \ba\|_2 \bu,
	\end{align}
	where $\bar\bW, \bar{\bW}_c \in \RR^{m \times (d - 1)}$ are matrices which have i.i.d. Gaussian entries with mean zero and variance $1 / d$. Furthermore, $\bar{\bW}$ is independent of $(\bg, \ba)$ and $\bar{\bW}_c$ is independent of $(\bg, \ba, \bar{\bW}^{\top} \bD_{\sigma} \ba)$. Such independence is established by \cref{lemma:gaussian-conditioning}. In the last line above, $\bu = \sqrt{d}\tau \bar{\bW}_c\bar{\bW}^{\top} \bD_{\sigma} \ba \|\bar{\bW}^{\top} \bD_{\sigma} \ba\|_2^{-1} \in \RR^m$, which by the property of the Gaussian distribution and the independence result we have just established has i.i.d. standard Gaussian entries and is independent of $(\bg, \ba, \bar{\bW}^{\top} \bD_{\sigma} \ba)$.  
	
	We introduce the following notations for the sake of simplicity. 	
	\begin{align}\label{eq:mu-beta-gamma}
		& \mu := \frac{1}{d}\tau s_d \bg^{\top} \bD_{\sigma} \ba, \nonumber \\
		& \beta := \tau s_d \cdot \frac{\|\bar\bW^{\top} \bD_{\sigma} \ba\|_2^2 - \langle \bar{\bW}_c^{\top} \bD_{\sigma} \ba, \bar\bW^{\top} \bD_{\sigma} \ba \rangle}{\sqrt{m}\|\bD_{\sigma} \ba\|_2^2}, \\
		& \gamma := \frac{s_d}{\sqrt{d}}\|\bar\bW^{\top} \bD_{\sigma} \ba\|_2.\nonumber
	\end{align}
	The following lemma states that under the current setting, the above quantities are small in probability. 
	\begin{lemma}\label{lemma:coefficients-op1}
		Under the conditions of \cref{thm:2-layer}, if we further assume that there exists a constant $S_0 > 0$, such that $s_d \rightarrow S_0$ as $d \to \infty$, then  as $d$ goes to infinity we have  
		\begin{align*}
			\mu = o_P(1), \qquad \beta = o_P(1), \qquad \gamma = o_P(1). 
		\end{align*}
	\end{lemma}
We postpone the proof of \cref{lemma:coefficients-op1} to Appendix \ref{sec:proof-of-lemma:coefficients-op1}, which constitutes merely of standard applications of concentration inequalities. In the following parts of the proof, we will assume $\{s_d\}_{d \in \NN_+}$ satisfies the condition stated in \cref{lemma:coefficients-op1}. 

Let $F: \RR^m \rightarrow \RR$ be a random function, such that $F(\by) = \sum_{i = 1}^m a_i \sigma(y_i)$. Then the quantity of interest $f(\bx^s) - f(\bx)$ can be expressed as $F(\bg^s) - F(\bg)$. Furthermore, by \cref{eq:decomp-gs,eq:mu-beta-gamma}, we have
	\begin{align*}
		& F(\bg^s) - F(\bg) - \langle \nabla F(\bg), \bg^s - \bg \rangle \\
		= & \sum_{i = 1}^m \left\{ a_i\big( \sigma((1 - \mu)g_i - \beta \sqrt{m}a_i\sigma'(g_i) - \gamma u_i) - \sigma(g_i) \big)+ \mu a_i\sigma'(g_i)g_i + \beta \sqrt{m} a_i^2 \sigma'(g_i)^2 + \gamma a_i u_i \sigma'(g_i)\right\} .
	\end{align*}
Then we proceed to show that $F(\bg^s)$ can be well approximated by the corresponding first order Taylor expansion at $\bg$. Namely, we will show that $|F(\bg^s) - F(\bg) - \langle \nabla F(\bg), \bg^s - \bg \rangle| = o_P(1)$. 

Let $b_i = \sqrt{m} a_i$, then $b_i \iidsim \normal(0,1)$ for $i \in [m]$. For $\btheta = (\theta_1, \theta_2, \theta_3) \in \RR^3$, we define
\begin{align*}
	h_{\btheta}(b,g,u) := b\sigma((1 - \theta_1)g - \theta_2 b\sigma'(g) - \theta_3 u) - b\sigma(g).
\end{align*}
{Notice that $F(\bg_s) = \sum_{i = 1}^mh_{(\mu, \beta, \gamma)}(b_i, g_i, u_i) / \sqrt{m}$ and $F(\bg) = \sum_{i = 1}^m h_{(0,0,0)}(b_i, g_i, u_i) / \sqrt{m}$. Given these expressions, it is a natural reflex to apply the central limit theorem to study $F(\bg_s) - F(\bg)$. However, the fact that $(\mu, \beta, \gamma)$ are random and depend on $(\ba, \bg, \bu)$ raises doubts about such application. To fix this issue, we resort to the uniform central limit theorem to present a valid result. }

For $\btheta \in \RR^3$, we define the empirical process $\GG_m$ evaluated at $\btheta$ as 
\begin{align}\label{eq:empirical-process-Gm}
	\GG_m(\btheta) := \frac{1}{\sqrt{m}}\sum_{i = 1}^m (h_{\btheta}(b_i, g_i, u_i) - \E[h_{\btheta}(b_i, g_i, u_i)]),
\end{align}
where the expectation is taken over $\{(b_i, g_i, u_i)\}_{i \leq m} \iidsim \normal(\textbf{0}, \id_3)$.  For $\btheta, \barbtheta \in \RR^3$, we define the covariance function $c: \RR^3 \times \RR^3 \to \RR$ as
\begin{align}\label{eq:cov}
	c(\btheta, \barbtheta) := \E[h_{\btheta}(b,g,u)h_{\barbtheta}(b,g,u)] - \E[h_{\btheta}(b,g,u)]\E[h_{\barbtheta}(b,g,u)]. 
\end{align}
Application of the regularity assumptions imposed on $\sigma$ and the dominated convergence theorem evidently reveals the continuity of $c(\cdot, \cdot)$. We denote by $\GG$ the Gaussian process indexed by $\btheta$ with mean zero and covariance function $c(\cdot, \cdot)$. The following lemma establishes that $\GG_m$ converges weakly to $\GG$.
\begin{lemma}\label{lemma:Donsker}
	Let $\Omega := \{\bx \in \RR^3: \|\bx\|_{\infty} \leq 1\}$, and $C(\Omega)$ be the space of continuous functions on $\Omega$ endowed with the supremum norm. Under the conditions of \cref{thm:2-layer}, if we further assume that there exists a constant $S_0 > 0$, such that $s_d \rightarrow S_0$ as $d \to \infty$, then $\{\GG_m\}_{m \geq 1}$ converges weakly in $C(\Omega)$ to $\GG$, which is a Gaussian process with mean zero and covariance defined in \cref{eq:cov}. 
\end{lemma}
The proof of \cref{lemma:Donsker} is deferred to Appendix \ref{sec:proof-of-lemma-Donsker}. 

\begin{remark}
Lemma \ref{lemma:Donsker} is the main step in which we loose quantitative control of
success probability for the FGSM attack. We prove this lemma by an application
of the uniform central limit theorem. A more explicit approach should be able to provide
concrete probability bounds.
\end{remark}

For $\btheta, \barbtheta \in \Omega$, we define $\rho(\btheta, \barbtheta) :=\E[(\GG(\btheta) - \GG(\barbtheta))^2]^{1/2}$. Then by \cite[Lemma 18.15]{van2000asymptotic}, without any loss, we can and will assume that $\GG$ almost surely has $\rho$-continuous sample path. 

In the following parts, we fix some positive $\epsilon$ a priori, and define $S_{\epsilon}(\GG) := \sup_{\|\btheta\|_{\infty} \leq \ep} |\GG(\btheta)|$. Note that $S_{\epsilon}$ is a continuous function with respect to the supremum norm on $\Omega$, thus $S_{\epsilon}(\GG_m)$ converges weakly to $S_{\epsilon}(\GG)$ due to the continuous mapping theorem. Recall that we showed in  \cref{lemma:coefficients-op1} that $\mu, \beta, \gamma \toP 0$ as $d \to \infty$, thus for any $\epsilon > 0$, 
\begin{align*}
	& \left|F(\bg^s) - F(\bg) + \sqrt{m}\beta \E[\sigma'(g)\sigma'((1 - \mu)g - \beta b \sigma'(g) - \gamma u)] \right|\\
	 \overset{(i)}{=} & \left| \frac{1}{\sqrt{m}}\sum_{i = 1}^m b_i\big( \sigma((1 - \mu)g_i - \beta b_i\sigma'(g_i) - \gamma u_i) - \sigma(g_i) \big) - \sqrt{m} \E\big[b\big( \sigma((1 - \mu)g - \beta b \sigma'(g) - \gamma u) - \sigma(g) \big)\big] \right| \\
	 {=} & \left|\frac{1}{\sqrt{m}}\sum_{i = 1}^m \big(h_{(\mu, \beta, \gamma)}(b_i, g_i, u_i) - \E[h_{(\mu, \beta, \gamma)}(b_i, g_i, u_i)] \big) \right| \\
	 \overset{(ii)}{\leq} & S_{\ep}(\GG_m) + \delta_{\ep}(m),
\end{align*}
where  $\delta_{\ep}(m) \toP 0$ as $d  \rightarrow \infty$. In the above display,  \emph{(i)} is by Stein's lemma (see \cref{lemma:Stein} below), \emph{(ii)} is by \cref{lemma:coefficients-op1}, and the expectations are taken over $\{b,b_i, g, g_i, u, u_i\}_{i \in [m]} \iidsim \normal(0,1)$. 
\begin{lemma}[Stein's lemma]\label{lemma:Stein}
	Suppose $Z$ is a normally distributed random variable with expectation $x_1$ and variance $x_2$. Further suppose $g$ is a function for which the two expectations $\E[g(Z)(Z - x_1)]$ and $\E[g'(Z)]$ both exist. Then 
	\begin{align*}
		\E[g(Z)(Z - x_1)] = x_2 \E[g'(Z)].
	\end{align*}
\end{lemma}
We next prove that $S_{\ep}(\GG_m)$ is small, which consists of two major steps. In the first step, we show that $S_{\ep}(\GG)$ is small, then we establish that $S_{\ep}(\GG)$ and $S_{\ep}(\GG_m)$ are close for large $d$.

Note that $c(\mathbf{0}, \mathbf{0}) = 0$. Since $\GG$ has $\rho$-continuous sample path and the covariance function $c$ is continuous, we then obtain that $S_{\ep}(\GG) \toP 0$ as $\ep \to 0^+$. 
For all $\ep' > 0$, we first choose $\ep > 0$ small enough, such that $\P(S_{\ep}(\GG) \geq \ep'/3) \leq \ep' / 3$. Since $S_{\ep}(\GG_m) \overset{d}{\to} S_{\ep}(\GG)$, $S_{\ep}(\GG)$ obviously has continuous cumulative distribution function, and $\delta_{\ep}(m) \toP 0$ as $m,d \to \infty$, putting these together we conclude that there exists $m_{\ep, \ep'} \in \NN_+$, such that for all $m \geq m_{\ep, \ep'}$, $\P(|\delta_{\ep}(m)| \geq \ep' / 3) \leq \ep' / 3$ and $\P(|S_{\ep}(\GG_m)| \geq \ep' / 3) \leq \ep'/ 3 + \P(|S_{\ep}(\GG)| \geq \ep' / 3)$. In summary, for all $m \geq m_{\ep, \ep'}$, 
\begin{align*}
	\P\left(\left|F(\bg^s) - F(\bg) + \sqrt{m}\beta \E[\sigma'(g)\sigma'((1 - \mu)g - \beta b \sigma'(g) - \gamma u)] \right| \geq \ep' \right) \leq \ep'.
\end{align*}
As the choice of $\ep'$, we then have
\begin{align}\label{eq:3}
	\left|F(\bg^s) - F(\bg) + \sqrt{m}\beta \E[\sigma'(g)\sigma'((1 - \mu)g - \beta b \sigma'(g) - \gamma u)] \right| = o_P(1).
\end{align}
Note that in the above equation the expectation is taken over $(b,g,u) \sim \normal(\mathbf{0}, \id_3)$, and $\E[\sigma'(g)\sigma'((1 - \mu)g - \beta b \sigma'(g) - \gamma u)]$ is a random variable which depends on the values of the random vector $(\mu, \beta, \gamma)$. 
 
Next, we consider $\langle \nabla F(\bg), \bg - \bg^s \rangle$. Notice that this formula admits the following decomposition:
\begin{align*}
	\langle \nabla F(\bg), \bg - \bg^s \rangle =& \mu {\sum_{i = 1}^m a_i \sigma'(g_i) g_i} + \beta \sqrt{m} { \sum_{i = 1}^m a_i^2 \sigma'(g_i)^2}+ \gamma {\sum_{i = 1}^m a_i u_i \sigma'(g_i)} \\
	 = & \mu T_1 + \beta\sqrt{m} T_2 + \gamma T_3.
\end{align*}
Since $\{\sqrt{m}a_i,u_i, g_i\}_{i \in [m]}$ are i.i.d. standard Gaussian random variables, as an immediate consequence of the law of large numbers and the central limit theorem, we can conclude that $T_1$ and $T_3$ are both $O_P(1)$, and $T_2 = \E[\sigma'(g)^2] + O_P(m^{-1/2})$. Using this result and \cref{lemma:coefficients-op1}, we further deduce that
\begin{align}\label{eq:4}
	\langle \nabla F(\bg), \bg - \bg^s \rangle = \beta \sqrt{m} \E[\sigma'(g)^2] + o_P(1). 
\end{align}
The law of large numbers implies  that as $d \to \infty$, we have $\beta \sqrt{m} = \tau S_0 + o_P(1)$. As a result, $\langle \nabla F(\bg), \bg - \bg^s \rangle = \tau S_0 \E[\sigma'(g)^2] + o_P(1)$. 

Recall that according to \cref{lemma:coefficients-op1},  $\mu, \beta, \gamma = o_P(1)$, thus intuitively we would expect that the expectations displayed in \cref{eq:3} and \cref{eq:4} are close to each other. 
To make such heuristic rigorous, we notice that by assumption,  $\sigma'$ is almost everywhere continuous and $|\sigma'(x)| \leq C_{\sigma}(1 + |x|^{k - 1})$. Then we apply the dominated convergence theorem, and obtain that as  $d \to \infty$,
\begin{align}\label{eq:5}
	\E[\sigma'(g)\sigma'((1 - \mu)g - \beta b \sigma'(g) - \gamma u)] = \E[\sigma'(g)^2] + o_P(1). 
\end{align}
Substituting \cref{eq:4,eq:5} into \cref{eq:3} gives $|F(\bg^s) - F(\bg) - \langle \nabla F(\bg), \bg^s - \bg \rangle| = o_P(1)$, which further leads to $F(\bg^s) = F(\bg) - \tau S_0 \E[\sigma'(g)^2] + o_P(1)$. Furthermore, the central limit theorem implies that $F(\bg) \overset{d}{\rightarrow} \normal(0, \E[\sigma(g)^2])$, thus as $d \to \infty$, 
\begin{align*}
	\P(\sign(F(\bg)) \neq \sign(F(\bg^s))) \to \P(\sign(z) \neq \sign(z - \sign(z)S_0\E[\sigma'(g)^2])), 
\end{align*}
where $z \sim \normal(0,\E[\sigma(g)^2])$, $g \sim \normal(0,1)$, and are independent of each other. Since $S_0$ is arbitrary, using a standard diagonal argument, we derive that there exists a sequence of step sizes $\{s_d\}_{d \in \NN_+}$, such that as $d \to \infty$, $\P(\sign(F(\bg)) \neq \sign(F(\bg^s))) \rightarrow 1$ and $\|\bx^s - \bx\|_2 / \|\bx\|_2 \toP 0$. 
More precisely, for all $n \in \NN_+$, there exists $S_0^n > 0$ and $d_n \in \NN_+$, such that if we set $s_d = S^n_0$ for all $d \in \NN_+$, then for all $d \geq d_n$, 
\begin{align*}
	\P\big(\sign(F(\bg)) \neq \sign(F(\bg^s))\big) \geq 1 - \frac{1}{n}.
\end{align*} 
Without loss of generality, we can assume $d_n < d_{n + 1}$, $S_0^n \leq \xi_{d_n}$, and $S_0^n / \sqrt d_n < n^{-1}$ by fixing $S_0^n$ and taking $d_n$ large enough. We set $s_d = S_0^n$ if and only if $d_n \leq d < d_{n + 1}$. Under such choice of $\{s_d\}_{d \geq 1}$, for all $d_{n + 1} > d \geq d_n$, we have
\begin{align*}
	\frac{s_d}{\sqrt{d}} = \frac{S_0^n}{\sqrt{d}} \leq \frac{S_0^n}{\sqrt{d_n}} \leq \frac{1}{n}, \qquad \P\big(\sign(F(\bg)) \neq \sign(F(\bg^s))\big)  \geq 1 -  \frac{1}{n}, \qquad s_d \leq \xi_{d_n} \leq \xi_d.
\end{align*} 
Since $n$ is arbitrary, then we combine the equations above with the first claim of the theorem and conclude the proof of the second claim.

\section{Non-asymptotic result for two-layer networks}
\label{sec:NonAsymp}

As emphasized above, Theorem \ref{thm:2-layer} does not provide
a quantitative convergence rate for the probability that the attack succeeds,
namely, $\P(\sign(f(\bx)) \neq \sign(f(\bx^s)))$.
We remedy to this by establishing a non-asymptotic bound below,
at the price of assuming that the activation function $\sigma$ is Lipschitz continuous. 
%
\begin{theorem}\label{thm:non-asymptotic-2-layer}
Consider the random two-layer neural network  of
	Eq.~\eqref{eq:TwoLayerFirst}. Let $\bx \in \RR^d$ be a deterministic vector 
	with $\|\bx\|_2 = \sqrt{d}$. Assume that $\sigma$ is $L$-Lipschitz over $\RR$ 
	for some $L \geq 1$, $\sigma$ is not a constant, and $\sigma'$ is almost everywhere continuous. Then there exist numerical constants $c, C_0 > 0$, such that for all $\xi > 0$, if the following conditions hold:
\begin{align}\label{eq:conditions}
\begin{split}
	& d \geq \max \left\{  \frac{C_{\xi}^2C_0(\sigma(0)^2 + L^2)}{\xi}, \frac{4L^4C_{\xi}^2}{c^2}\left( \log \frac{C_0}{\xi} \right)^2, 16C_{\xi}^4\left(1 + \E_{g \sim \normal(0,1)}[\sigma(g)^2] \right)^2 \right\}, \\
	& m \geq  C_{\xi}^4,  \qquad \tilde Q_{d, m} \geq \E_{g \sim \normal(0,1)}[\sigma'(g)^2] / 2, \qquad \xi \leq C_0 e^{-9c},  \qquad s_d = C_{\xi}, 
\end{split}
\end{align}
where 
\begin{align*}
	& C_{\xi} = \frac{4  \sqrt{\log \frac{C_0}{\xi} \cdot (\E_{g \sim \normal(0,1)}[\sigma(g)^2] + 1)}}{\sqrt{c}\,\E_{g \sim \normal(0,1)}[\sigma'(g)^2]}, \\
	& \tilde Q_{d,m} = \min_{{|\theta|_1 \leq d^{-1/2}, |\theta_2| \leq 2m^{-1/4}, |\theta_3| \leq d^{-1/4}}} \E_{g, b, u \sim_{i.i.d.} \normal(0,1)}[\sigma'(g) \sigma'((1 - \theta_1)g - \theta_2 b \sigma'(g) - \theta_3 u)].
\end{align*}
Then with probability at least 
\begin{align*}
	1 - 3\xi - C_0 \left( \exp(-cd) + m^{-1}(\sigma(0)^4 + L^4) + 2L(d^{-1/4} + m^{-1/4}) \right) ,
\end{align*}
it holds that $\sign(f(\bx)) \neq \sign(f(\bx^s))$. 
	
\end{theorem}
\begin{remark}
In order to parse the last statement, think of $\xi$ as a small constant controlling the probability
that the adversarial attack fails. 
Then the conditions \eqref{eq:conditions} require $d$ and $m$ to be larger than some
constants $d_0(\xi)$ and $m_0(\xi)$, where $d_0(\xi)$ is of order $\log(1/\xi)/\xi$,
and $m_0(\xi)$ is only polylogarithmic in $1/\xi$. 
\end{remark}
In order to prove \cref{thm:non-asymptotic-2-layer}, it suffices to prove the following lemma, which can be regarded as a more general version of the statement. 
\begin{lemma}\label{lemma:non-asymptotic-2-layer}
	Under the conditions of \cref{thm:non-asymptotic-2-layer}, 
	for $\eta_1, \eta_2 > 0$, we define
	\begin{align}
	& \cS_{d, m} := \left\{ \btheta:  | \theta_1  | \leq \frac{s_d\eta_1}{d}, \,\, \left|\theta_2 - \frac{\tau s_d}{\sqrt{m}} \right| \leq \frac{2s_d \eta_2}{\sqrt{dm}}, \,\, |\theta_3 | \leq \frac{2s_d}{\sqrt d} \cdot \sqrt{ 1 + \E_{g \sim \normal(0,1)}[\sigma(g)^2] }\right\}, \label{eq:Sdm}\\
	& Q_{d, m} := \min_{\btheta \in \cS_{d, m}} \E_{g, b, u \sim_{i.i.d.} \normal(0,1)}[\sigma'(g) \sigma'((1 - \theta_1)g - \theta_2 b \sigma'(g) - \theta_3 u)],  \label{eq:Qdm}\\
	& \delta_{d, m} := \max \left\{\frac{s_d\eta_1}{d},\, \frac{2s_d \eta_2}{\sqrt{dm}} +  \frac{\tau s_d}{\sqrt{m}},\, \frac{2s_d}{\sqrt d} \cdot \sqrt{ 1 + \E_{g \sim \normal(0,1)}[\sigma(g)^2] }\right\}. \label{eq:deltadm}
\end{align}
In the above definitions, we ignore the dependence on $(\eta_1, \eta_2)$ for the sake of simplicity.  Then there exist numerical constants $c, C_0 > 0$, such that with probability at least 
$$1 - C_0\left\{\eta_1^{-2}(\sigma(0)^2 + L^2) + \exp(-c\eta_2) +\exp(-cd) +m^{-1}(\sigma(0)^4 + L^4) +\exp(-c\eta_3^2) + \eta^{-1}  L \delta_{d, m} \right\},$$
we have $\sign(f(\bx)) \neq \sign(f(\bx^s))$. In the above display, 
\begin{align*}
	\eta_3 = \frac{s_d Q_{d, m } - \eta - 2d^{-1/2}s_d  L^2 \eta_2 - 1}{\sqrt{\E_{g \sim \normal(0,1)}[\sigma(g)^2] + 1}}.
\end{align*}
\end{lemma}
\begin{remark}\label{rmk:Xi}
	In \cref{lemma:non-asymptotic-2-layer}, if we set 
	\begin{align*}
		\eta = 1, \,\,\,\,\,\, \eta_1 = \sqrt{\frac{C_0(\sigma(0)^2 + L^2)}{\xi}}, \,\,\,\,\,\, \eta_2 = c^{-1} \log \frac{C_0}{\xi}, \,\,\,\,\,\,   s_d = \frac{2\sqrt{\E_{g \sim \normal(0,1)}[\sigma(g)^2] + 1} \cdot \left( \sqrt{c^{-1} \log \frac{C_0}{\xi}} + 1 \right) + 2}{\E_{g \sim \normal(0,1)}[\sigma(g)^2] - 4 d^{-1/2}\eta_2}, 
	\end{align*}
	then \cref{thm:non-asymptotic-2-layer} reduces to a direct corollary of \cref{lemma:non-asymptotic-2-layer}.  
\end{remark}
We will prove \cref{lemma:non-asymptotic-2-layer} in the rest parts of this section. Recall that $\mu, \beta, \gamma$ were defined in \cref{eq:mu-beta-gamma}. We first give a non-asymptotic characterization of these random quantities.  
\begin{lemma}\label{lemma:control-mu-beta-gamma}
	There exist numerical constants $c, C > 0$, such that the following results hold: 
	\begin{enumerate}
		\item For any $\eta_1 > 0$, with probability at least $1 - C\eta_1^{-2}(\sigma(0)^2 + L^2)$,
		\begin{align*}
			| \mu  | \leq \frac{s_d\eta_1}{d}.
		\end{align*} 
		\item For any $\eta_2 \geq 1$, with probability at least $1 - C\exp(-c\eta_2)$,
		\begin{align*}
			\left|\beta - \frac{\tau s_d}{\sqrt{m}} \right| \leq \frac{2s_d \eta_2}{\sqrt{dm}}.
		\end{align*}
		\item With probability at least $1 - 2 \exp(-cd) - Cm^{-1}(\sigma(0)^4 + L^4)$,
		\begin{align*}
			|\gamma \,| \leq \frac{2s_d}{\sqrt d} \cdot \sqrt{ 1 + \E_{g \sim \normal(0,1)}[\sigma(g)^2] }.
		\end{align*}
	\end{enumerate}
\end{lemma} 
We postpone the proof of \cref{lemma:control-mu-beta-gamma} to Appendix \ref{sec:proof-of-lemma:control-mu-beta-gamma}. 
In what follow, we will always assume that the events described in \cref{lemma:control-mu-beta-gamma} 
occur. Namely, we will be working on event $\cS$ defined as follows:
\begin{align*}
	\cS = \left\{| \mu  | \leq \frac{s_d\eta_1}{d}, \, \left|\beta - \frac{\tau s_d}{\sqrt{m}} \right| \leq \frac{2s_d \eta_2}{\sqrt{dm}}, \,  |\gamma \,| \leq \frac{2s_d}{\sqrt d} \cdot \sqrt{ 1 + \E_{g \sim \normal(0,1)}[\sigma(g)^2] }\right\}.
\end{align*}
By \cref{lemma:control-mu-beta-gamma}, $\P(\cS) \geq 1 - C\eta_1^{-2}(\sigma(0)^2 + L^2) - C\exp(-c\eta_2) - 2 \exp(-cd) - Cm^{-1}(\sigma(0)^4 + L^4)$.
Recall from \cref{eq:decomp-gs,eq:empirical-process-Gm} that 
\begin{align*}
	F(\bg^s) - F(\bg) = & \sum_{i = 1}^m a_i\big( \sigma((1 - \mu)g_i - \beta \sqrt{m}a_i\sigma'(g_i) - \gamma u_i) - \sigma(g_i) \big) \\
	= & \, \GG_m((\mu, \beta, \gamma)) \, - \, \sqrt{m}\beta\,\E_{g, b, u \sim_{i.i.d.} \normal(0,1)} \left[   \sigma'(g)\sigma'((1 - \mu) g - \beta b \sigma'(g) - \gamma u)  \right].
\end{align*}
We then show that $\GG_m((\mu, \beta, \gamma))$ is close to zero. More precisely, for $\delta > 0$, we define the set $\Theta_{\delta} := \{\btheta \in \RR^3: \|\btheta\|_{\infty} \leq \delta\}$. We immediately see that if $\max\{|\mu|, |\beta|, |\gamma| \} \leq \delta$, then 
\begin{align*}
	\GG_m((\mu, \beta, \gamma)) \leq \sup _{\btheta \in \Theta_{\delta}} \GG_m(\btheta).
\end{align*}
For $\btheta \in \RR^3$, we define 
\begin{align*}
	\LL_m(\btheta) = \frac{1}{\sqrt{m}} \sum_{i = 1}^m \eps_i b_i \left( \sigma((1 - \theta_1) g_i - \theta_2 b_i \sigma'(g_i) - \theta_3 u_i) - \sigma(g_i) \right),
\end{align*}
where $\eps_i \iidsim \Unif\{-1, +1\}$ and are independent of everything else.
By symmetrization, it holds that 
\begin{align}\label{eq:sym}
	\E\left[\sup_{\btheta \in \Theta_{\delta}} \GG_m(\btheta)\right] \leq 2\, \E \left[ \sup_{\btheta \in \Theta_{\delta}} \LL_m(\btheta)\right].
\end{align} 
Since $\LL_m(\vec{0}) = \GG_m(\vec{0}) = 0$, we see that both $\sup_{\btheta \in \Theta_{\delta}} \GG_m(\btheta)$ and $\sup_{\btheta \in \Theta_{\delta}} \LL_m(\btheta)$ are non-negative. Furthermore, conditioning on $\{(b_i, g_i, u_i): i \in [m]\}$, it is not hard to see that $\{\LL_m(\btheta): \btheta \in \Theta_{\delta}\}$ is a sub-Gaussian process indexed by parameters in $\Theta_{\delta}$. In addition, the sub-Gaussian norm of this process $\|\cdot\|_{\Psi_2}$ satisfies the following inequality: 
\begin{lemma}\label{lemma:upper-bound-sub-gaussian-norm}
	There exists a numerical constant $C > 0$, such that for $\btheta, \btheta' \in \RR^3$,
	\begin{align*}
	\left\|\LL_m(\btheta) - \LL_m(\btheta')\right\|_{\Psi_2} \leq C \cdot \sqrt{\frac{1}{m} \sum_{i = 1}^m M(b_i, g_i, u_i)^2} \cdot \|\btheta - \btheta'\|_2,
	\end{align*}
	where 
	\begin{align*}
		M(b, g, u) = L\cdot |b| \cdot \sqrt{g^2 + L^2 b^2 + u^2}.
	\end{align*}
\end{lemma}
We defer the proof of the lemma to Appendix \ref{sec:proof-of-lemma:upper-bound-sub-gaussian-norm}. Using the Dudley’s integral inequality, we conclude that there exists another numerical constant $C' > 0$, such that
\begin{align}\label{eq:Dudley}
	& \E_{\eps_i \sim_{i.i.d.} \Unif\{\pm 1\}}\left[ \sup_{\btheta \in \Theta_{\delta}} \frac{1}{\sqrt{m}} \sum_{i = 1}^m \eps_i b_i \left( \sigma((1 - \theta_1) g_i - \theta_2 b_i - \theta_3 u_i) - \sigma(g_i) \right) \right] \nonumber \\
	\leq & \, C' \sqrt{\frac{1}{m} \sum_{i = 1}^m M(b_i, g_i, u_i)^2} \cdot \int_0^{\infty} \sqrt{\log \mathcal{N} (\Theta_{\delta}, \|\cdot\|_2, x)} \dd x,
\end{align}
where $\mathcal{N}(\Theta, \|\cdot\|_2, x)$ is the smallest number of closed balls with centers in $\Theta$ and radius $x$ whose union covers $\Theta$. In our case, since $\Theta_{\delta}$ is contained in the ball centered at the origin with radius $\sqrt{3}\delta$, we have $\mathcal{N} (\Theta_{\delta}, \|\cdot\|_2, x) \leq (1 + 4\delta / x)^3$.  Plugging this upper bound into \cref{eq:Dudley} further leads to the following result:
\begin{align}\label{eq:D-integral}
	\E\left[\sup_{\btheta \in \Theta_{\delta}} \LL_m(\btheta) \right] \leq & C''L \int_0^{\sqrt{3}\delta} \sqrt{ \log (1 + 4\delta / x)} \dd x \leq  4C''L {\delta},
\end{align}
where $C'' > 0$ is a numerical constant. 
Combining \cref{eq:sym,eq:D-integral}, we can further upper bound the expectation of the  non-negative random variable $\sup_{\btheta \in \Theta_{\delta}} \GG_m(\btheta)$ with $8C'' L \delta$. By Markov's inequality, with probability at least $1 - 8 \eta^{-1} C'' L \delta$, $\sup_{\btheta \in \Theta_{\delta}} \GG_m(\btheta) \leq \eta$ for any $\eta > 0$. Recall that $\delta_{d,m}$ is defined in \cref{eq:deltadm}. By \cref{lemma:control-mu-beta-gamma}, we see that with probability at least $1 - C\eta_1^{-2}(\sigma(0)^2 + L^2) - C\exp(-c\eta_2) - 2 \exp(-cd) - Cm^{-1}(\sigma(0)^4 + L^4) - 8 \eta^{-1} C'' L \delta_{d, m}$,  
\begin{align*}
	F(\bg^s) - F(\bg) \leq &  \,\sup_{\btheta \in \Theta_{\delta_{d,m}}}\GG_m(\btheta) \, - \, \sqrt{m}\beta\,\E_{g, b, u \sim_{i.i.d.} \normal(0,1)} \left[   \sigma'(g)\sigma'((1 - \mu) g - \beta b \sigma'(g) - \gamma u)  \right]  \\
	 \leq & \, \eta - \sqrt{m}\beta\,\E_{g, b, u \sim_{i.i.d.} \normal(0,1)} \left[   \sigma'(g)\sigma'((1 - \mu) g - \beta b \sigma'(g) - \gamma u)  \right]. 
\end{align*}
Analogously, similar lower bound holds with at least the same amount of probability: 
\begin{align*}
	F(\bg^s) - F(\bg) \geq -\eta - \sqrt{m}\beta\,\E_{g, b, u \sim_{i.i.d.} \normal(0,1)} \left[   \sigma'(g)\sigma'((1 - \mu) g - \beta b \sigma'(g) - \gamma u)  \right]. 
\end{align*}
Furthermore, on the set $\cS$, it holds that
\begin{align*}
	\left|\sqrt{m} \beta - \tau s_d \right|\leq \frac{2s_d \eta_2}{\sqrt{d}}.
\end{align*}
As a result, with probability at least $1 - 2C\eta_1^{-2}(\sigma(0)^2 + L^2) - 2C\exp(-c\eta_2) - 4\exp(-cd) - 2Cm^{-1}(\sigma(0)^4 + L^4) - 16 \eta^{-1} C'' L \delta_{d, m}$ 
\begin{align}
	\left| F(\bg^s) - F(\bg) + \tau s_d \E_{g, b, u \sim_{i.i.d.} \normal(0,1)} \left[   \sigma'(g)\sigma'((1 - \mu) g - \beta b \sigma'(g) - \gamma u)  \right] \right| \leq \eta + \frac{2s_d L^2 \eta_2}{\sqrt{d}}.
\end{align}
Since $\mu, \beta, \gamma$ are all $o_P(1)$, we expect that $\E_{g, b, u \sim_{i.i.d.} \normal(0,1)} \left[   \sigma'(g)\sigma'((1 - \mu) g - \beta b \sigma'(g) - \gamma u)  \right]$ should be approximately equal to $\E_{g \sim \normal(0,1)}[\sigma'(g)^2]$, which is strictly positive and does not depend on $(m,d)$, provided that $\sigma$ is not a constant function. In this case, we only need to choose the step size $s_d$ large enough to flip the sign of $F(\bg)$. This argument can be made rigorous via the following lemma. In this lemma, we upper bound the magnitude of $F(\bg)$. 
\begin{lemma}\label{lemma:upper-bound-Fg}
	For any $\eta_3 \geq 0$, with probability at least $1 - C\exp (-\frac{cm}{\sigma(0)^4 + L^4}) - C\exp(-c\eta_3^2)$ for some numerical constants $c, C > 0$, we have
	\begin{align*}
		|F(\bg)| \leq \eta_3 \cdot \sqrt{\E_{g \sim \normal(0,1)}[\sigma(g)^2] + 1}. 
	\end{align*}
\end{lemma}
We prove the lemma in Appendix \ref{sec:proof-of-lemma:upper-bound-Fg}. 
Let $\cS' := \cS\, \cap \{|F(\bg)| \leq \eta_3 \cdot \sqrt{\E_{g \sim \normal(0,1)}[\sigma(g)^2] + 1}\}$. On $\cS'$, if in addition we have
\begin{align*}
	s_d \, Q_{d, m} - \eta - \frac{2s_d L^2 \eta_2}{\sqrt{d}} \geq \eta_3\cdot \sqrt{\E_{g \sim \normal(0,1)}[\sigma(g)^2] + 1},
\end{align*} 
then $\sign(F(\bg)) \neq \sign(F(\bg^s))$. In the rest parts of the proof, we will always take 
\begin{align*}
	\eta_3 = \frac{s_d Q_{d, m } - \eta - 2d^{-1/2}s_d  L^2 \eta_2 - 1}{\sqrt{\E_{g \sim \normal(0,1)}[\sigma(g)^2] + 1}}. 
\end{align*}
With such choice of $\eta_3$, we can finally put together all above analysis and conclude that the adversarial example succeeds with probability at least 
\begin{align*}
	1 - C_0\left\{\eta_1^{-2}(\sigma(0)^2 + L^2) + \exp(-c\eta_2) +\exp(-cd) +m^{-1}(\sigma(0)^4 + L^4) +\exp(-c\eta_3^2) + \eta^{-1} L \delta_{d, m} \right\}
\end{align*}
for some absolute positive constants $c, C_0$. 

%

%
%
%

\section{Random multi-layer networks}
\label{sec:multilayer}

We generalize the model considered in \cref{sec:2layer} in the current section. More precisely,
we consider a multi-layer neural network with $l + 1$ layers for $l \in \NN_+$:
\begin{align*}
	f(\bx) = \bW_{l + 1} \sigma(\bW_l \sigma(\cdots \sigma(\bW_2 \sigma(\bW_1 \bx)) \cdots)).
\end{align*}
In the above equation, the random weight matrix $\bW_i \in \RR^{d_i \times d_{i - 1}}$ has i.i.d. Gaussian entries: 
$(\bW_{i})_{jj'} \iidsim \normal(0, 1 / d_{i - 1})$ for all $j \in [d_i], j' \in [d_{i - 1}]$, 
and further $\{\bW_i\}_{i \in [l+ 1]}$ are independent of each other. We assume $d_0 = d$, $d_{l + 1} = 1$, 
and $d_i = d_i(d) \rightarrow \infty$ for all $0 \leq i \leq l$. The $d$-dimensional 
input vector $\bx$  is a deterministic vector with Euclidean norm $\sqrt{d}$. 
The activation function $\sigma: \RR \rightarrow \RR$ is understood to act on vectors entrywise.  

For the simplicity of notations, we define recursively the following vectors: $\bh_0 := \bx$, $\bg_1 := \bW_1 \bx$, $\bh_j := \sigma(\bg_j)$ and $\bg_{j + 1} := \bW_{j + 1} \bh_j$ for $j \in [l]$.
 The gradient of $f$ can be expressed as:
\begin{align*}
	\nabla f(\bx) = \bW_1^{\top} \bD_{\sigma}^1 \bW_2^{\top} \bD_{\sigma}^2 \cdots \bW_{l}^{\top} \bD_{\sigma}^l \bW_{l + 1}^{\top},
\end{align*} 
where $\bD_{\sigma}^j = \diag(\{\sigma'(\bg_j)\}) \in \RR^{d_j \times d_j}$. As before, we denote by 
$\tau \in \{\pm 1\}$ the sign of $f(\bx)$, and let $\{s_d\}_{d \in \NN_+} \subseteq \RR^+$ 
be a sequence of step sizes to be determined.

\begin{theorem}\label{thm:multi-layer}
	Assume that $\sigma$ satisfies the conditions in \cref{thm:2-layer}. 
	Then the following results hold:
	\begin{enumerate}
		\item There exists a constant $C > 0$ depending uniquely on $(\sigma, l)$, such that
		 for any $\delta \in (0,1)$, with probability at least $1 - \delta$, 
		\begin{align}\label{eq:BoundPertML}
			\frac{\|\bx - \bx^s\|_2}{\|\bx\|_2} \leq \frac{Cs_d}{\sqrt{d}} ( \sqrt{\log(1 / \delta)} + 1)^{l - 1}(1 + \log(1 / \delta)d^{-1/2}) \prod_{i = 1}^l  \prod_{j = 1}^i \big(1 + \delta^{-1/2} d_j^{-1/2} \big)^{k^{i - j}},
		\end{align}
		where we recall that $k$ is a fixed positive integer such that $|\sigma'(x)| \leq C_{\sigma}(1 + |x|^{k - 1})$.
		\item  Let $\{\xi_d\}_{d \in \NN_+} \subseteq \RR^+$ be an increasing sequence 
		such that $\xi_d \to \infty$ as $d \to \infty$.
		Then there exists $\{s_d\}_{d \in \NN_+} \subseteq \RR^+$, such that $s_d \leq \xi_d$
		and the following limits hold: 
		\begin{align}
	\plim_{d\to\infty}\frac{\|\bx - \bx^s\|_2}{\|\bx\|_2} =  0\, ,\;\;\;	
	\lim_{d\to\infty}\P(\sign(f(\bx)) \neq \sign(f(\bx^s))) = 1\, .
\end{align}
\end{enumerate}
\end{theorem}

\begin{remark}
The proof of \cref{thm:multi-layer} applies without changes to neural networks which 
have different activation functions at different layers, provided they satisfy the assumptions as stated.
We refrain from stating such a generalization to avoid cumbersome notations.
\end{remark}

\begin{remark}
The bound on the perturbation size in Eq.~\eqref{eq:BoundPertML} deteriorates when the 
depth $l$ becomes exponentially large in the input dimension $d$. 
A similar behavior is observed in \cite{bartlett2021adversarial} which also provides 
an example of a random network with exponential depth for which the output is 
nearly constant, and in particular is immune to FGSM attacks.

For general random networks, the example of \cite{bartlett2021adversarial}
implies that subexponential depth is a required assumption. 
On the other hand, the example of \cite{bartlett2021adversarial} is special in that the network is non-balanced:
it takes the same sign for any input. As we pointed out in the introduction, if the output 
takes each sign  on a fraction of the inputs with measure bounded away from zero 
(it is `balanced'), then it must have adversarial examples by isoperimetry.

On the other hand, it is unclear whether subexponential depth is necessary for FGSM attacks to be successful on
random networks, after balancing. As the depth increases, the random function $\bx\mapsto f(\bx)$ becomes
 `rougher,' as it can be seen by computing its covariance function. While such a function will contain 
adversarial examples, it is likely to be more difficult to find them by a single gradient step
as in FGSM. (Of course a special case is the one of linear activations: in that case depth is irrelevant.)
\end{remark}

\section{Proof of \cref{thm:multi-layer}}
\label{sec:multilayer-proof}

\subsubsection*{Proof of the first claim}  

For $m \in [l]$, we define $\boeta_m := \bD_{\sigma}^m \bW_{m + 1}^{\top} \bD_{\sigma}^{m +1} \cdots \bW_l^{\top} \bD_{\sigma}^l \bW_{l + 1}^{\top} \in \RR^{d_m}$ and $\by_m := \bW_m^{\top} \boeta_m \in \RR^{d_{m - 1}}$. The following lemma shows that the normalized Euclidean norms of $\boeta_m, \by_m, \bh_m, \bg_m$ converge in probability to some deterministic constants as $d \rightarrow \infty$. Furthermore, such constants are independent of the choice of $\{s_d\}_{d \in \NN_+}$. 
\begin{lemma}\label{lemma:Pconvergence}
	Under the conditions of \cref{thm:multi-layer}, the following sequences of random variables converge in probability to strictly positive constants as $d \to \infty$:
	\begin{enumerate}
		\item $\{\|\bh_m\|_2^2 / d_m\}_{d \geq 1}$, for all $1 \leq m \leq l$.
		\item $\{\|\bg_m\|_2^2 / d_m \}_{d \geq 1}$, for all $1 \leq m \leq l$. 
		\item $\{\|\boeta_m\|_2^2\}_{d \geq 1}$, for all $1 \leq m \leq l$.
		\item $\{\|\by_m\|_2^2\}_{d \geq 1}$, for all $1 \leq m \leq l$.
	\end{enumerate}   
	Furthermore, 
	\begin{enumerate}
		\item[5.] $\bh_{m - 1}^{\top} \bW_m^{\top} \boeta_m = O_P(1)$, for all  $1 \leq m \leq l$.  
	\end{enumerate} 
\end{lemma}
\begin{remark}
	The above sequences of random variables are independent of the choice of $\{s_d\}_{d \in \NN_+}$.
\end{remark}
The proof of \cref{lemma:Pconvergence} is deferred to Appendix \ref{sec:proof-of-lemma:Pconvergence}. As in the two-layers case, in the next lemma we provide a finite sample upper bound on the Eulidean norm of the gradient $\nabla f(\bx)$:

\begin{lemma}\label{lemma:finite-sample}
	Under the conditions of \cref{thm:multi-layer}, there exists a constant $Q > 0$, which is a function of  $(\sigma, l)$ only, such that for any $\delta > 0$, with probability at least $1 - \delta$, we have
	\begin{align*}
		\|\nabla f(\bx)\|_2 \leq Q( \sqrt{\log(1 / \delta)} + 1)^{l - 1}(1 + \log(1 / \delta)d^{-1/2}) \prod_{i = 1}^l  \prod_{j = 1}^i \big(1 + \delta^{-1/2} d_j^{-1/2} \big)^{k^{i - j}}.
	\end{align*}
\end{lemma}
The proof of \cref{lemma:finite-sample} is deferred to Appendix \ref{sec:proof-of-lemma:finite-sample}. Recall that $\|\bx\|_2 = \sqrt{d}$, thus the first claim of the theorem is just a straightforward consequence of  \cref{lemma:finite-sample}.

\subsubsection*{Proof of the second claim}  

Our proof of the second claim proceeds by induction. Before stating our induction hypothesis, we analyze the first layer to gain some intuition.  

 We define 
 \begin{align*}
 	& \bg_1^s := \bW_1 \bx^s = \bg_1 - \tau s_d \bW_1 \bW_1^{\top} \boeta_1, \\
 	& \cF_1 := \sigma\{\bg_1, \boeta_1, \{\bW_i\}_{2 \leq i \leq l + 1}, \bx\}.
 \end{align*}
Notice that $\bg_1$, $\bg_1^s$ can be regarded as the outputs of the first layer with the inputs being $\bx$ and $\bx^s$, respectively. 
 
Since $\bW_1$ has i.i.d. Gaussian entries, and $\cF_1$ depends on $\bW_1$ only through $\bg_1 = \bW_1 \bx$.  Invoking \cref{lemma:gaussian-conditioning}, we can write $\bW_1 = \bg_1 \bx^{\top} / d + \tilde{\bW}_1 \Pi_{\bx}^{\perp}$, where $\tilde{\bW}_1$ has the same marginal distribution as $\bW_1$ and is independent of $\cF_1$. Then we have
%
\begin{align*}
	\bg_1^s = \bg_1(1 - \tau s_d \bg_1^{\top} \boeta_1 / d) - \tau s_d \tilde{\bW}_1 \Pi_{\bx}^{\perp} \tilde{\bW}_1^{\top} \boeta_1.
\end{align*}
Furthermore, using the property of Gaussian distribution, we have $\tilde{\bW}_1 \Pi_{\bx}^{\perp} \tilde{\bW}_1^{\top} = \bar\bW_1 \bar\bW_1^{\top}$, where $\bar\bW_1 \in \RR^{d_1 \times (d - 1)}$ is a matrix that has i.i.d. Gaussian entries with mean zero and variance $1 / d$ that is further independent of $\cF_1$. By \cref{lemma:gaussian-conditioning}, $\bar\bW_1$ admits the decomposition $\bar\bW_1 {=} \Pi_{\boeta_1}^{\perp}\bW_1'  + \Pi_{\boeta_1}\bar\bW_1 $, where $\bW_1'$ has the same marginal distribution with $\bar\bW_1$ and is further independent of $\sigma\{\cF_1, \bar\bW_1^{\top} \boeta_1\}$. Therefore, 
%
\begin{align*}
	\bg_1^s{=} & \bg_1(1 -  \tau s_d \bg_1^{\top} \boeta_1 / d) - \tau s_d \frac{\|\bar\bW_1^{\top} \boeta_1\|_2^2 - \langle (\bW_1')^{\top} \boeta_1, \bar\bW_1^{\top} \boeta_1\rangle}{\|\boeta_1\|_2^2} \boeta_1 - \tau s_d \bW_1' \bar\bW_1^{\top} \boeta_1 \\
	= & (1 - \mu_1)\bg_1 - \beta_1 \boeta_1 - \gamma_1 \bu_1,
\end{align*}
where $\bu_1 := \sqrt{d} \tau \bW_1'\bar{\bW}_1^{\top} \boeta_1 / \|\bar{\bW}_1^{\top} \boeta_1\|_2$. Note that $\tau \in \cF_1$, and $\bW_1'$ is independent of $\sigma\{\cF_1, \bar\bW_1^{\top} \boeta_1\}$. We hence obtain that $\bu_1 \sim \normal(\mathbf{0}, \id_{d_1})$,                                                                                                                                                                                                                                                                                                                                                                                                                                                                                                                                                                                                                                                                                                                                                                                                                                                                                                                                                                                                                                                                                                                                                                                                                                                                                                                                                                                                                                                                                                                                                                                                                                                                                                                                                                                                                                                                                                                                                                                                                                                                                                                                                                                                                                                                                                                                                                                                                                                                                                                                                                                                                                                                                                                                                                                                                                                                                                                                                                                                                                                                                                                                                                                                                                                                                                                                                                                                                                                                                                                                                                                                                                                                                                                                                                                                                                                                                                                                                                                                                                                                                                                                                                                                                                                                                                                                                                                                                                                                                                                                                                                                                                                                                                                                                                                                                                                                                                                                                                                                                                                                                                                                                                                                                                                                                                                                                                                                                                                                                                                                                                                                                                                                                                                                                                                                                                                                                                                                                                                                                                                                                                                                                                                                                                                                                                                                                                                                                                                                                                                                                                                                                                                                                                                                                                                                                                                                                                                                                                                                                                                                                                                                                                                                                                                                                                                                                                                                                                                                                                                                                                                                                                                                                                                                                                                                                                                                                                                                                                                                                                                                                                                                                                                                                                                                                                                                                                                                                                                                                                                                                                                                                                                                                                                                                                                                                                                                                                                                                                                                                                                                                                                                                                                            and is independent of $\sigma\{\bar\bW_1^{\top}\boeta_1, \cF_1\}$. In the equation above, we have 
\begin{align*}
	\mu_1 := \frac{1}{d} \tau s_d \bg_1^{\top} \boeta_1, \qquad \beta_1 :=  \tau s_d \frac{\|\bar\bW_1^{\top} \boeta_1\|_2^2 - \langle (\bW_1')^{\top} \boeta_1, \bar\bW_1^{\top} \boeta_1\rangle}{\|\boeta_1\|_2^2}, \qquad \gamma_1 := \frac{s_d}{\sqrt{d}} \|\bar\bW_1^{\top} \boeta_1\|_2.
\end{align*}
We define the sigma algebra $\cG_1 := \sigma\{\mu_1, \bg_1, \beta_1, \boeta_1, \gamma_1, \bu_1, \bx, \bg_2, \{\bW_i\}_{3 \leq i \leq l + 1}\}$. Note that $\cG_1$ depends on $\bW_2$ only through $\bg_2 = \bW_2 \bh_1$ and $ \by_2 = \bW_2^{\top} \boeta_2$. 
In the following parts of the proof, we will assume $s_d \rightarrow S_0$ for some positive constant $S_0$. Under such choice of $s_d$, we can prove the following two lemmas:
\begin{lemma}\label{lemma:op1-base}
	Under the assumptions of \cref{thm:multi-layer}, if we further assume that $s_d \to S_0$ for some positive constant $S_0$, then as $d \to \infty$, we have $\mu_1 = o_P(1)$, $\beta_1 = O_P(1)$, $\gamma_1 = o_P(1)$. 
\end{lemma}
\begin{lemma}\label{lemma:hk-1_sigma}
	Under the assumptions of \cref{thm:multi-layer}, if we further assume that $s_d \to S_0$ for some positive constant $S_0$, then the following limits hold as $d \to \infty$:
	\begin{align*}
		 \frac{1}{d_{1}}\|\Pi_{\bh_{1}}^{\perp} \sigma(\bg_{1}^s)\|_2^2 \toP 0, \qquad  \frac{\langle \bh_{1}, \sigma(\bg_{1}^s) \rangle}{\|\bh_{1}  \|_2^2} \toP 1.
	\end{align*}
\end{lemma}
The proofs of \cref{lemma:op1-base,lemma:hk-1_sigma} are deferred to Appendices \ref{sec:proof-of-lemma:op1-base} and \ref{sec:proof-of-lemma:hk-1_sigma}, respectively. 

For $2 \leq i \leq l$, we define $\bg_i^s := \bW_i \sigma(\bg_{i - 1}^s) \in \RR^{d_i}$ as the output of an intermediate layer of the neural network with the input being the adversarial example. In summary, we have shown $\mathcal{H}_m$ holds for $m = 1$ with $\mathcal{H}_m$ stated below. Next, we proceed by induction and show that $\mathcal{H}_m$ holds for all $m \in [l]$. 
\subsection*{$\mathbf{\mathcal{H}_m}$:}
\begin{enumerate}
	\item[\emph{(i)}] There exists $\bu_{m}\sim \normal(\mathbf{0}, \id_{d_{m}})$, that is independent of $\cF_m := \sigma\{\bg_{m}, \boeta_{m}, \{\bW_i\}_{m + 1 \leq i \leq l + 1}, \bh_{m - 1}\}$, such that 
	\begin{align*}
		\bg_{m}^s=(1 - \mu_{m})\bg_{m} - \beta_{m} \boeta_{m} - \gamma_{m} \bu_{m}.
	\end{align*}
\item[\emph{(ii)}] Let
\begin{align*}
	\cG_m := \sigma\{\mu_{m}, \bg_{m}, \beta_{m}, \boeta_{m}, \gamma_{m}, \bu_{m}, \bh_{m - 1}, \bg_{m + 1}, \{\bW_i\}_{m + 2 \leq i \leq l + 1}\}.
\end{align*}
Then $\cG_m$ depends on $\bW_{m + 1}$ only through $\bg_{m + 1} =  \bW_{m + 1} \bh_{m}$ and $\by_{m + 1} = \bW_{m + 1}^{\top} \boeta_{m + 1}$.
In particular, $\cG_m \perp \Pi_{\boeta_{m + 1}}^{\perp}\bW_{m + 1} \Pi_{\bh_{m}}^{\perp}$.
\item[\emph{(iii)}] For $(\mu_{m}, \beta_{m}, \gamma_{m})$ in \emph{(i)}, the following results hold: $\mu_{m} = o_P(1)$, $\beta_{m} = O_P(1)$, $\gamma_{m} = o_P(1)$. 
\item[\emph{(iv)}] As $d \to \infty$, $\|\Pi_{\bh_{m}}^{\perp} \sigma(\bg_{m}^s)\|_2^2 / d_m \toP 0$ and ${\langle \bh_{m}, \sigma(\bg_{m}^s) \rangle} / {\|\bh_{m}\|_2^2} \toP 1$.
\item[\emph{(v)}] There exists a random variable $R_m$ and a positive constant $\alpha_m$, whose distribution and value depend only on $(\sigma, m, l)$. In particular, they are independent of the input dimension and the number of neurons. Furthermore, they satisfy $\beta_m \geq \alpha_m \beta_{m - 1} + R_m + o_P(1)$. 
\end{enumerate}
Note that claim \emph{(v)} does not apply for the base case $m = 1$. Next, we will show that if $\mathcal{H}_m$ holds for all $m \leq l - 1$, then this further implies that $\mathcal{H}_{m + 1}$ holds. 

\subsubsection*{Proofs of $\mathcal{H}_{m + 1}$ claims \emph{(i)} and \emph{(ii)}}

By $\mathcal{H}_m$ claim \emph{(i)}, we have
\begin{align}\label{eq:gks}
	\bg_{m + 1}^s = \bW_{m + 1}\sigma(\bg_m^s) =  \bW_{m + 1} \sigma((1 - \mu_{m})\bg_{m} - \beta_{m} \bD_{\sigma}^{m} \bW_{m + 1}^{\top} \boeta_{m + 1} - \gamma_{m} \bu_{m}).
\end{align}
From $\mathcal{H}_m$ claim \emph{(ii)}, we see that $\cG_m$ depends on $\bW_{m + 1}$ only through $\bg_{m + 1} = \bW_{m + 1} \bh_m$ and $\by_{m + 1} = \bW_{m + 1}^{\top} \boeta_{m + 1}$, $\bh_m$ is independent of $\bW_{m + 1}$, and $\boeta_{m + 1}$ depends on $\bW_{m + 1}$ only through $\bg_{m + 1} = \bW_{m + 1} \bh_m$. Therefore, invoking \cref{lemma:gaussian-conditioning}, we find that exists $\tilde{\bW}_{m + 1}$ that has the same marginal distribution as $\bW_{m + 1}$ and is independent of $\cG_m$, such that
\begin{align}\label{eq:Wk}
	\bW_{m + 1} = \frac{\bg_{m + 1} \bh_{m}^{\top}}{\|\bh_{m}\|_2^2} + \Pi_{\boeta_{m + 1}}^{\perp}\tilde \bW_{m + 1} \Pi_{\bh_{m}}^{\perp} + \frac{\boeta_{m + 1} \by_{m + 1}^{\top}\Pi_{\bh_{m}}^{\perp}}{\|\boeta_{m + 1}\|_2^2} ,
\end{align}
Next, we substitute \cref{eq:Wk} into \cref{eq:gks}, which leads to the following equality
\begin{align*}
	\bg_{m + 1}^s =& \frac{\langle \bh_{m}, \sigma(\bg_{m}^s) \rangle}{\|\bh_{m}\|_2^2} \bg_{m + 1} + \frac{\boeta_{m + 1}^{\top} \bW_{m + 1} \Pi_{\bh_{m}}^{\perp} \sigma(\bg_{m}^s) - \boeta_{m + 1}^{\top} \tilde\bW_{m + 1} \Pi_{\bh_{m}}^{\perp} \sigma(\bg_{m}^s)}{\|\boeta_{m + 1}\|_2^2} \boeta_{m + 1} + \tilde{\bW}_{m + 1} \Pi_{\bh_{m}}^{\perp} \sigma(\bg_{m}^s) \\
	= & (1 - \mu_{m + 1}) \bg_{m + 1} - \beta_{m + 1} \boeta_{m + 1} - \gamma_{m + 1}  \bu_{m + 1},
\end{align*}
where $\bu_{m + 1} := -\sqrt{d_m} \tilde{\bW}_{m + 1} \Pi_{\bh_{m}}^{\perp} \sigma(\bg_{m}^s) \|\Pi_{\bh_m}^{\perp} \sigma(\bg_m^s)\|_2^{-1}$. Since $\tilde{\bW}_{m + 1}$ is independent of $\cG_m$, and $\Pi_{\bh_{m}}^{\perp} \sigma(\bg_{m}^s) \in \cG_m$, we then conclude that $\bu_{m + 1}  \sim \normal(\mathbf{0}, \id_{d_{m + 1}})$, and $\bu_{m + 1}$ is independent of $\cG_m$. Since $\bh_{m} = \sigma(\bg_{m})$ and $\boeta_{m + 1}$ is a function of $\bg_{m + 1}$ and $\{\bW_i\}_{m + 2 \leq i \leq l + 1}$, we obtain that $\cF_{m + 1} \subseteq \cG_m$ and $\bu_{m + 1} \perp \cF_{m + 1}$. Thus, we have completed the proof of $\mathcal{H}_{m + 1}$, claim \emph{(i)}. Furthermore, $(\mu_{m + 1}, \beta_{m + 1}, \gamma_{m + 1})$ can be expressed as follows:
\begin{align}\label{eq:mubetagamma}
\begin{split}
	& \mu_{m + 1} = 1 - \frac{\langle \bh_{m}, \sigma(\bg_{m}^s) \rangle}{\|\bh_{m}\|_2^2}, \qquad \beta_{m + 1} = \frac{-\boeta_{m + 1}^{\top} \bW_{m + 1} \Pi_{\bh_{m}}^{\perp} \sigma(\bg_{m}^s) + \boeta_{m + 1}^{\top} \tilde\bW_{m + 1} \Pi_{\bh_{m}}^{\perp} \sigma(\bg_{m}^s)}{\|\boeta_{m + 1}\|_2^2}, \\
	& \gamma_{m + 1} = \frac{1}{\sqrt{d_{m}}} \|\Pi_{\bh_{m}}^{\perp} \sigma(\bg_{m}^s)\|_2.
\end{split}
\end{align}
Notice that $\cG_{m + 1}$ depends on $\bW_{m + 2}$ only through $\bg_{m + 2} = \bW_{m + 2} \bh_{m + 1}$ and $\by_{m + 2} = \bW_{m + 2}^{\top} \boeta_{m + 2}$, thus proving $\mathcal{H}_{m + 1}$ claim \emph{(ii)}. 

\subsubsection*{Proofs of $\mathcal{H}_{m + 1}$ claims \emph{(iii)} and \emph{(v)}}

The following lemma is a direct consequence of the induction hypothesis. 
\begin{lemma}\label{lemma:op1}
	Under the assumptions of \cref{thm:multi-layer}, if we further assume that $s_d \to S_0$ for some positive constant $S_0$, and $\mathcal{H}_m$ holds, then $\mu_{m + 1} = o_P(1)$ and $\gamma_{m + 1} = o_P(1)$. 
\end{lemma}
The proof of \cref{lemma:op1} is deferred to Appendix \ref{sec:proof-of-lemma:op1}. 
 We define the random object 
 $$\mathcal{V}_{m + 1} := (\bg_{m}, \bD_{\sigma}^{m}, \bh_{m}, \bh_{m - 1}, \boeta_{m + 1}, \bu_{m}, \bW_{m}).$$ 
 
 Note that $\mathcal{V}_{m + 1}$ depends on $\bW_{m + 1}$ only through $\bg_{m + 1} = \bW_{m + 1} \bh_{m}$. Since $\bh_m$ is independent of $\bW_{m + 1}$, by \cref{lemma:gaussian-conditioning}, we can write $\bW_{m + 1} = \bar\bW_{m + 1} \Pi_{\bh_{m}}^{\perp} + \bW_{m + 1} \Pi_{\bh_{m}}$, where $\bar\bW_{m + 1} \in \RR^{d_m \times d_{m + 1}}$ has the same marginal distribution with $\bW_{m + 1}$, and is independent of $\mathcal{V}_{m + 1}$. 
Next we prove that $\beta_{m + 1} = O_P(1)$. 

	We consider the first term in the enumerator of the definition of $\beta_{m + 1}$ in \cref{eq:mubetagamma}, and substitute in the decompositions we just obtained, which gives 
	\begin{align}\label{eq:7}
	& \langle \Pi_{\bh_{m}}^{\perp} \bW_{m + 1}^{\top} \boeta_{m + 1}, \sigma(\bg_{m}^s) \rangle \nonumber \\
	= & \langle \Pi_{\bh_{m}}^{\perp} \bW_{m + 1}^{\top} \boeta_{m + 1}, \sigma((1 - \mu_{m}) \bg_{m} - \beta_{m} \bD_{\sigma}^{m} \Pi_{\bh_{m}}^{\perp} \bW_{m + 1}^{\top} \boeta_{m + 1} -  \beta_{m} \bD_{\sigma}^{m} \Pi_{\bh_{m}} \bW_{m + 1}^{\top} \boeta_{m + 1} - \gamma_{m} \bu_{m}) \rangle \nonumber \\
	= & \langle \Pi_{\bh_{m}}^{\perp} \bar\bW_{m + 1}^{\top} \boeta_{m + 1}, \sigma((1 - \mu_{m}) \bg_{m} - \beta_{m} \bD_{\sigma}^{m} \Pi_{\bh_{m}}^{\perp} \bar\bW_{m + 1}^{\top} \boeta_{m + 1} -  \beta_{m} \bD_{\sigma}^{m} \Pi_{\bh_{m}} \bW_{m + 1}^{\top} \boeta_{m + 1} - \gamma_{m} \bu_{m}) \rangle  \nonumber\\
	= & \langle \bar\bW_{m + 1}^{\top} \boeta_{m + 1}, \sigma((1 - \mu_{m}) \bg_{m} - \beta_{m} \bD_{\sigma}^{m} \bar\bW_{m + 1}^{\top} \boeta_{m + 1} - \beta_{m} \delta_{m + 1} \bD_{\sigma}^{m} \bh_{m}  - \gamma_{m} \bu_{m})  \rangle \nonumber  \\
	& - \langle\bh_{m}, \sigma(\bg_{m}^s) \rangle \langle \bh_{m}, \bar\bW^{\top}_{m + 1} \boeta_{m + 1} \rangle / \|\bh_{m}\|_2^2,
\end{align}
where 
\begin{align*}
	\delta_{m + 1} := \frac{1}{ \|\bh_{m}\|_2^2}(\bh_{m}^{\top} \bW_{m + 1}^{\top} \boeta_{m + 1} - \bh_{m}^{\top} \bar\bW_{m + 1}^{\top} \boeta_{m + 1} ) .
\end{align*}
\cref{lemma:Pconvergence} together with the fact that $\bar{\bW}_{m + 1}$ is independent of $(\bh_m, \boeta_{m + 1})$ implies that  $\delta_{m + 1} = O_P(d_m^{-1})$.

From $\mathcal{H}_m$ claim \emph{(iv)}, we see that as $d \to \infty$, $\langle\bh_{m}, \sigma(\bg_{m}^s) \rangle / \|\bh_{m}\|_2^2 \toP 1$. Furthermore, since $\bar{\bW}_{m + 1} \perp \mathcal{V}_{m + 1}$, then conditioning on $(\bh_{m}, \boeta_{m + 1})$, we have $\langle \bh_{m}, \bar\bW^{\top}_{m + 1} \boeta_{m + 1} \rangle \overset{d}{=} \normal(0, \|\bh_{m}\|_2^2 \|\boeta_{m + 1}\|_2^2 / d_{m})$. Putting together these results and \cref{lemma:Pconvergence}, we conclude that $\langle \bh_{m}, \bar\bW^{\top}_{m + 1} \boeta_{m + 1} \rangle \overset{d}{\rightarrow} \normal(0, H_{m}^2E^2_{m + 1})$, where $H_{m} := \plim \|\bh_{m}\|_2 / {\sqrt{d_{m}}}$ and $E_{m + 1} := \plim \|\boeta_{m + 1}\|_2$. In summary, we have
\begin{align}\label{eq:7.5}
	\frac{ \langle\bh_{m}, \sigma(\bg_{m}^s) \rangle}{\|\bh_{m}\|_2^2} \langle \bh_{m}, \bar\bW^{\top}_{m + 1} \boeta_{m + 1} \rangle  \overset{d}{\to} \normal(0, H_m^2E_{m + 1}^2).
\end{align}
Thus, the second term in the last line of \cref{eq:7} is $O_P(1)$.

Next, we consider the first term in the last line of \cref{eq:7}. Conditioning on $\bh_{m - 1}$, we have $\bg_{m} = \bz_{m}\nu_{m}$, where $\nu_{m} := \sqrt{\|\bh_{m - 1}\|_2^2 / d_{m - 1}}$, $\bz_m = \nu_m^{-1}\bW_m \bh_{m - 1} \in \RR^{d_m}$. Since $\bW_m$ is independent of $\bh_{m - 1}$, we can conclude that $\bz_m \sim \normal(\mathbf{0}, \id_{d_{m}})$ and is independent of $\bh_{m - 1}$. By $\mathcal{H}_m$ claim \emph{(i)}, $\bu_{m}$ is independent of $\bg_{m}$ and $\bh_{m - 1}$. Therefore, $\bu_{m}$ is further independent of $(\bz_{m}, \bh_{m - 1})$. Since $\bar\bW_{m + 1}$ is independent of $(\bg_{m}, \bh_{m - 1}, \bu_{m}, \boeta_{m + 1})$, we can write $\bar\bW_{m + 1}^{\top} \boeta_{m + 1} = \bar\gamma_{m + 1}\bar\bu_{m + 1} / \sqrt{d_{m}}$, where $\bar\bu_{m + 1} \sim \normal(\mathbf{0}, \id_{d_{m}})$, is independent of $(\bz_{m}, \bu_{m}, \bh_{m - 1})$ and  $\bar\gamma_{m + 1} = \|\boeta_{m + 1}\|_2 $. In summary, we have
\begin{align}\label{eq:indep}
	\bz_{m}, \bu_{m}, \bar{\bu}_{m + 1} \iidsim \normal(\mathbf{0}, \id_{d_{m}}), \qquad (\bz_{m}, \bu_{m}, \bar{\bu}_{m + 1} ) \perp \bh_{m - 1}.
\end{align}
For $\btheta \in \RR^6$, we define
\begin{align*}
	& h_{\btheta}^{(m + 1)}(\bar u, z, u) := \theta_1 \bar{u} \sigma\big((1 - \theta_2) \theta_3 z - \theta_4\sigma'(\theta_3 z) \theta_1 \bar{u} - \theta_5 \sigma'(\theta_3 z) \sigma(\theta_3 z)  - \theta_6 u \big), \\
	& \bar h_{\btheta}^{(m + 1)}(\bar u, z, u) := \theta_1 \bar{u} \sigma\big((1 - \theta_2) \theta_3 z - \theta_4\sigma'(H_{m - 1} z) \theta_1 \bar{u} - \theta_5 \sigma'(H_{m - 1}  z) \sigma(\theta_3 z)  - \theta_6 u \big),
\end{align*}
where $H_{m - 1} = \plim \sqrt{\|\bh_{m - 1}\|_2^2 / d_{m - 1}}$. We further define the empirical processes $\GG_{d}^{(m + 1)}, \bar\GG_{d}^{(m + 1)}$ evaluated at $\btheta$ as
\begin{align*}
	& \GG_{d}^{(m + 1)}(\btheta) := \frac{1}{\sqrt{d_{m}}}\sum_{i = 1}^{d_{m}} (h_{\btheta}^{(m + 1)}(\bar{u}_{m + 1,i}, z_{m, i}, u_{m,i}) - \E[h_{\btheta}^{(m + 1)}(\bar{u}_{m + 1,i}, z_{m, i}, u_{m,i})]), \\
	& \bar\GG_{d}^{(m + 1)}(\btheta) := \frac{1}{\sqrt{d_{m}}}\sum_{i = 1}^{d_{m}} ( \bar h_{\btheta}^{(m + 1)}(\bar{u}_{m + 1,i}, z_{m, i}, u_{m,i}) - \E[\bar h_{\btheta}^{(m + 1)}(\bar{u}_{m + 1,i}, z_{m, i}, u_{m,i})]),
\end{align*}
where the expectations are taken over $\{(\bar{u}_{m + 1,i}, z_{m, i}, u_{m, i})\}_{i \leq d_m} \iidsim \normal(\mathbf{0}, \id_3)$. Here, $\bar{u}_{m + 1,i}$ is the $i$-th coordinate of $\bar{\bu}_{m + 1}$, $z_{m, i}$ is the $i$-th coordinate of $\bz_{m}$, and $u_{m, i}$ is the $i$-th coordinate of $\bu_{m}$. For $\btheta, \barbtheta \in \RR^6$, we define the covariance function $\bar c^{(m + 1)}(\btheta, \barbtheta)$ as
\begin{align}\label{eq:cov2}
	\bar c^{(m + 1)}(\btheta, \barbtheta) := \E[\bar h_{\btheta}^{(m + 1)}(\bar{u},z,u)\bar h_{\barbtheta}^{(m + 1)}(\bar{u},z,u)] - \E[\bar h_{\btheta}^{(m + 1)}(\bar{u},z,u)]\E[\bar h_{\barbtheta}^{(m + 1)}(\bar{u},z,u)],
\end{align}
where the expectations are taken over $(\bar{u}, z, u) \sim \normal(\mathbf{0}, \id_3)$. Since $\sigma'$ is almost everywhere continuous, and by assumption almost everywhere we have $|\sigma'(x)| \leq C_{\sigma}(1 + |x|^{k - 1})$, then standard application of the dominated convergence theorem shows that the covariance function $\bar c^{(m + 1)}(\cdot, \cdot)$ is continuous. Recall that $H_{m - 1} = \plim \sqrt{\|\bh_{m - 1}\|_2^2 / d_{m - 1}} = \plim \nu_m$. Furthermore, we define $E_{m + 1} := \plim \|\boeta_{m + 1}\|_2 = \plim \bar{\gamma}_{m + 1}$. The following lemma establishes a weak convergence result for $\bar \GG_{d}^{(m + 1)}$.

\begin{lemma}\label{lemma:Donsker2}
	Let $\Omega_{m + 1} := \{\bx \in \RR^6: \|\bx\|_{\infty} \leq H_{m - 1} + E_{m + 1}\}$, and $C(\Omega_{m + 1})$ be the space of continuous functions on $\Omega_{m + 1}$ endowed with the supremum norm. Under the conditions of \cref{thm:multi-layer}, if we further assume that there exists $S_0 > 0$ such that $s_d \to S_0$ as $d \to \infty$, and induction hypothesis $\mathcal{H}_m$ holds, then $\{\bar \GG_d^{(m + 1)}\}_{d \geq 1}$ converges weakly in $C(\Omega_{m + 1})$ to $\bar \GG^{(m + 1)}$ as $d \to \infty$, which is a Gaussian process with mean zero and covariance defined in \cref{eq:cov2}.
\end{lemma}
The proof of \cref{lemma:Donsker2} is deferred to Appendix \ref{sec:proof-of-lemma:Donsker2}.

\begin{lemma}\label{lemma:approx}
Under the conditions of \cref{thm:multi-layer}, if we further assume that there exists $S_0 > 0$ such that $s_d \to S_0$ as $d \to \infty$, and induction hypothesis $\mathcal{H}_m$ holds, then as $d \to \infty$ we have 
	\begin{align*}
		\GG_d^{(m + 1)}(\bar\gamma_{m + 1}, \mu_{m}, \nu_m, \beta_{m} / \sqrt{d_{m}}, \beta_{m} \delta_{m + 1}, \gamma_{m}) = \bar\GG_d^{(m + 1)}(\bar\gamma_{m + 1}, \mu_{m}, \nu_m, \beta_{m} / \sqrt{d_{m}}, \beta_{m} \delta_{m + 1}, \gamma_{m}) + o_P(1). 
	\end{align*}
\end{lemma}
The proof of \cref{lemma:approx} is deferred to Appendix \ref{sec:proof-of-lemma:approx}. Next, we will apply \cref{lemma:Donsker2,lemma:approx} to show that  $\beta_{m + 1} = O_P(1)$ (thus $\mathcal{H}_{m + 1}$ claim \emph{(iii)} holds) and $\mathcal{H}_{m + 1}$ claim \emph{(v)}. Note that 
\begin{align}\label{eq:6}
	 & \langle \bar\bW_{m + 1}^{\top} \boeta_{m + 1}, \sigma((1 - \mu_{m}) \bg_{m} - \beta_{m} \bD_{\sigma}^{m} \bar\bW_{m + 1}^{\top} \boeta_{m + 1} - \beta_{m} \delta_{m + 1} \bD_{\sigma}^{m} \bh_{m}  - \gamma_{m} \bu_{m})  \rangle \nonumber \\
	 = & \frac{1}{\sqrt{d_{m}}}\sum_{i = 1}^{d_{m}} \bar\gamma_{m + 1} \bar{u}_{m + 1,i} \times \nonumber \\
	 & \sigma\big( (1 - \mu_{m})\nu_m z_{m,i} - d_m^{-1/2}\beta_{m} \bar\gamma_{m + 1} \sigma'(\nu_m z_{m,i}) \bar{u}_{m + 1,i}  - \beta_{m} \delta_{m + 1} \sigma'(\nu_m z_{m,i}) \sigma(\nu_m z_{m,i}) - \gamma_{m} u_{m, i} \big) \nonumber \\
	 \overset{\emph{(a)}}{=} & \GG_d^{(m + 1)}(\bar\gamma_{m + 1}, \mu_{m}, \nu_m, \beta_{m} / \sqrt{d_{m}}, \beta_{m} \delta_{m + 1}, \gamma_{m}) -  \nonumber \\
	 & \beta_{m} \bar\gamma_{m + 1}^2\E[\sigma'(\nu_m z)\sigma'\big( (1 - \mu_{m})\nu_m z - d_m^{-1/2}\beta_{m} \bar\gamma_{m + 1} \sigma'(\nu_m z) \bar{u}  - \beta_{m} \delta_{m + 1} \sigma'(\nu_m z) \sigma(\nu_m z) - \gamma_{m} u\big)],
\end{align}
where the expectation is taken over $(z, u, \bar{u}) \sim \normal(\mathbf{0}, \id_3)$. In step \emph{(a)} we apply Stein's lemma to derive the equality.
By $\mathcal{H}_m$ claim \emph{(iii)}, we have $\mu_{m} = o_P(1)$, $\beta_{m} / \sqrt{d_{m}} = o_P(1)$, and $\gamma_{m} = o_P(1)$. Recall that we have shown $\delta_{m + 1} = O_P(d_m^{-1})$, thus $\beta_{m} \delta_{m + 1} = o_P(1)$. By \cref{lemma:Pconvergence}, we have $\nu_m= H_{m - 1} + o_P(1)$, $\bar\gamma_{m + 1} = E_{m + 1} + o_P(1)$. Therefore, combining \cref{lemma:Donsker2,lemma:approx}, we conclude that for any $\ep > 0$, 
\begin{align*}
	\big|\GG_d^{(m + 1)}(\bar\gamma_{m + 1}, \mu_{m}, \nu_m, \beta_{m} / \sqrt{d_{m}}, \beta_{m} \delta_{m + 1}, \gamma_{m})\big| \leq \sup_{\btheta \in A_{\ep}^{(m + 1)}} |\bar\GG_d^{(m + 1)}| + o_P(1),
\end{align*}
where $$A_{\ep}^{(m + 1)} := \{\btheta \in \RR^6: |\theta_1 - E_{m + 1}| \leq \ep, |\theta_2| \leq \ep, |\theta_3 - H_{m - 1}| \leq \ep, |\theta_4| \leq \ep, |\theta_5| \leq \ep, |\theta_6| \leq \ep\}.$$ Furthermore, invoking the dominated convergence theorem, we see that as $d \to \infty$,
\begin{align}\label{eq:6.5}
	& \bar\gamma_{m + 1}^2\E[\sigma'(\nu_m z)\sigma'\big( (1 - \mu_{m})\nu_m z - d_m^{-1/2}\beta_{m} \bar\gamma_{m + 1} \sigma'(\nu_m z) \bar{u}  - \beta_{m} \delta_{m + 1} \sigma'(\nu_m z) \sigma(\nu_m z) - \gamma_{m} u \big)] \nonumber \\
	= & E_{m + 1}^2 \E[\sigma'(H_{m - 1}z_m)^2] + o_P(1). 
\end{align}
We define $M_{\ep}^{(m + 1)}(\GG) := \sup_{\btheta \in A_{\ep}^{(m + 1)}}|\GG(\btheta)|$, then $M_{\ep}^{(m + 1)}$ is a continuous function with respect to the supremum norm $\ell^{\infty}(\Omega_{m + 1})$. Using \cref{lemma:Donsker2} together with the continuous mapping theorem, we see that $M_{\ep}^{(m + 1)}(\bar\GG_d^{(m + 1)})$ converges in distribution to $M_{\ep}^{(m + 1)}(\bar\GG^{(m + 1)})$. Notice that if we let $\ep = 1$, then $E_{m + 1}, H_{m - 1}, \bar c^{(m + 1)}, A_{1}^{(m + 1)}$ depend only on $(l, m + 1, \sigma)$, thus the distribution of $M_{1}^{(m + 1)}(\bar\GG^{(m + 1)})$ also depends uniquely on $(l, m + 1, \sigma)$. By $\mathcal{H}_m$ claim \emph{(iii)}, we have $\beta_{m} = O_P(1)$. Putting this together with \cref{eq:6,eq:6.5}, we obtain that there exists a random variable $R_{m + 1}^{(1)}$, the distribution of which depends only on $(l, m + 1, \sigma)$, such that as $d \to \infty$,
\begin{align}\label{eq:8}
	& \big|\langle \bar\bW_{m + 1}^{\top} \boeta_{m + 1}, \sigma((1 - \mu_{m}) \bg_{m} - \beta_{m} \bD_{\sigma}^{m} \bar\bW_{m + 1}^{\top} \boeta_{m + 1} - \beta_{m} \delta_{m + 1} \bD_{\sigma}^{m} \bh_{m}  - \gamma_{m} \bu_{m})  \rangle + \beta_{m} E_{m + 1}^2 \E[\sigma'(H_{m - 1}z_m)^2]  \big| \nonumber \\
	& \leq   R_{m + 1}^{(1)}  + o_P(1). 
\end{align} 
Finally, we consider the second term in the enumerator of the definition of $\beta_{m + 1}$ given in \cref{eq:mubetagamma}. Conditioning on $(\boeta_{m + 1}, \Pi_{\bh_{m}}^{\perp} \sigma(\bg_{m}^s))$, 
\begin{align}\label{eq:8.5}
	\boeta_{m + 1}^{\top} \tilde{\bW}_{m + 1} \Pi_{\bh_{m}}^{\perp} \sigma(\bg_{m}^s) \overset{d}{=} \normal(0, \|\boeta_{m + 1}\|_2^2 \|\Pi_{\bh_{m}}^{\perp} \sigma(\bg_{m}^s)\|_2^2 / d_{m}),
\end{align}
which is $o_P(1)$ by \cref{lemma:Pconvergence} and $\mathcal{H}_m$ claim \emph{(iv)}. Note that $\|\boeta_{m + 1}\|_2^{-2} = E_{m + 1}^{-2} + o_P(1)$. Taking this collectively with \cref{eq:mubetagamma,eq:7,eq:7.5,eq:8,eq:8.5}, we obtain that $\beta_{m + 1} = O_P(1)$. Furthermore, there exists a random variable $R_{m + 1}$ and a positive number $\alpha_{m + 1}$, the distribution (and the value) of which is a function of $(\sigma, m + 1, l)$ only, such that $\beta_{m + 1} \geq \beta_m \alpha_{m + 1} + R_{m + 1} + o_P(1)$. This completes the proof of $\mathcal{H}_{m + 1}$ claims \emph{(iii)}, \emph{(v)}. 
%

\subsubsection*{Proof of $\mathcal{H}_{m + 1}$ claim \emph{(iv)}}

Finally, we prove $\mathcal{H}_{m + 1}$ claim \emph{(iv)}, which is achieved by the following lemma:
\begin{lemma}\label{lemma:hk-sigma}
	Under the assumptions of \cref{thm:multi-layer}, if we further assume that $s_d \to S_0$ for some positive constant $S_0$,   $\mathcal{H}_m$ holds, and claims {(i),(ii),(iii),(v)} from  $\mathcal{H}_{m + 1}$  hold, then as $d \to \infty$ we have the following convergences:
	\begin{align*}
		 \frac{1}{d_{m + 1}}\|\Pi_{\bh_{m + 1}}^{\perp} \sigma(\bg_{m + 1}^s)\|_2^2 \toP 0, \qquad  \frac{\langle \bh_{m + 1}, \sigma(\bg_{m + 1}^s) \rangle}{\|\bh_{m + 1}  \|_2^2} \toP 1.
	\end{align*}
\end{lemma}
The proof of \cref{lemma:hk-sigma} is deferred to Appendix \ref{sec:proof-of-lemma:hk-sigma}. We note that $\mathcal{H}_{m + 1}$ claim \emph{(iv)} is a direct consequence of \cref{lemma:hk-sigma}. By induction, we have completed the proof of $\mathcal{H}_i$ for all $i \in [l]$. 

\subsubsection*{Back to the proof of the theorem}

Next, we will apply results from $\mathcal{H}_l$ to prove \cref{thm:multi-layer}. By $\mathcal{H}_l$ claim \emph{(i)}, 
\begin{align*}
	\bg_{l}^s = (1 - \mu_l) \bg_l - \beta_l \bD_{\sigma}^l \bW_{l + 1}^{\top} - \gamma_l \bu_l, 
\end{align*}
Using our modeling assumption and $\mathcal{H}_l$ claim \emph{(i)}, we obtain that $(\bg_l, \bD_{\sigma}^l, \bu_l)$ is independent of $\bW_{l + 1}$.
By $\mathcal{H}_l$ claim  \emph{(iii)}, we have $\mu_l = o_P(1)$, $\beta_l / \sqrt{d_l} = o_P(1)$ and $\gamma_l = o_P(1)$. We define $\zeta_l := \beta_l / \sqrt{d_l}$, $F_l: \RR^{d_l} \to \RR$, such that $F_l(\by) := \sum_{i = 1}^{d_l} W_{l + 1, i} \sigma(y_i)$. Then we have
\begin{align*}
	 F_l(\bg_l^s) - F_l(\bg_l) = \sum_{i = 1}^{d_l} W_{l + 1, i} \big( \sigma((1 - \mu_l) g_{l,i} - \zeta_l \sigma'(g_{l,i}) \sqrt{d_l}W_{l + 1, i} - \gamma_l u_{l,i}) - \sigma(g_{l,i}) \big).
\end{align*} 
Let $z_i := \sqrt{d_{l}} W_{l + 1, i}$, then $\{z_i\}_{i \in [d_l]} \iidsim \normal(0,1)$. We define $\bg_l = \nu_l \bz_l$, where $\nu_l := \sqrt{\|\bh_{l - 1}\|_2^2 / d_{l - 1}}$, $\bz_l = \nu_l^{-1}\bW_l \bh_{l - 1}$. Since $\bW_l$ is independent of $\bh_{l - 1}$, we have $\bz_l \sim \normal(\mathbf{0}, \id_{d_l})$ and is independent of $\bh_{l - 1}$. Since $\bW_{l + 1}$ is independent of $(\bg_l, \bh_{l - 1})$, we conclude that it is also independent of $(\bz_l, \bh_{l - 1})$. By $\mathcal{H}_l$ claim $\emph{(i)}$, we know that $\bu_l$ is independent of $(\bg_l, \bh_{l - 1}, \bW_{l + 1})$, thus we obtain that $\bu_l, \sqrt{d_l}\bW_{l + 1}, \bz_l \iidsim  \normal(\mathbf{0}, \id_{d_l})$. Furthermore, $(\bu_l, \sqrt{d_l} \bW_{l + 1}, \bz_l)$ is independent of $\bh_{l - 1}$. 

By \cref{lemma:Pconvergence}, $\nu_l$ converges in probability to a positive constant $H_{l - 1}$. For $\btheta \in \RR^4$, we define
\begin{align*}
	& h_{\btheta}^{(l + 1)}(z, z_l, u_l) := z\big(\sigma((1 - \theta_1) \theta_2 z_l - \theta_3 \sigma'(\theta_2 z_l) z - \theta_4 u) - \sigma(\theta_2 z_l)\big), \\
	& \bar h_{\btheta}^{(l + 1)}(z, z_l, u_l) := z\big(\sigma((1 - \theta_1) \theta_2 z_l - \theta_3 \sigma'(H_{l - 1} z_l) z - \theta_4 u) - \sigma(\theta_2 z_l)\big).
\end{align*}
For $\btheta \in \RR^4$, we define the empirical processes $\GG_d^{(l + 1)}, \bar\GG_d^{(l + 1)}$ indexed by $\btheta$ as
\begin{align*}
	& \GG_d^{(l + 1)}(\btheta) := \frac{1}{\sqrt{d_l}} \sum_{i = 1}^{d_l} \big( h_{\btheta}(z_i, z_{l,i}, u_{l,i}) - \E[h_{\btheta}(z_i, z_{l,i}, u_{l,i})] \big), \\
	& \bar\GG_d^{(l + 1)}(\btheta) := \frac{1}{\sqrt{d_l}} \sum_{i = 1}^{d_l} \big(\bar h_{\btheta}(z_i, z_{l,i}, u_{l,i}) - \E[\bar h_{\btheta}(z_i, z_{l,i}, u_{l,i})] \big),
\end{align*}
where the expectations are taken over $\{(z_i, z_{l,i}, u_{l,i})\}_{i \in [d_l]} \iidsim \normal(\mathbf{0}, \id_3)$. For $\btheta, \bar\btheta \in \RR^4$, we define the covariance function $\bar c^{(l + 1)}(\btheta, \barbtheta)$ as
\begin{align*}
	\bar c^{(l + 1)}(\btheta, \bar\btheta) := \E[\bar h_{\btheta}(z, z_l, u_l)\bar  h_{\barbtheta}(z, z_l, u_l)] - \E[\bar h_{\btheta}(z, z_l, u_l) ] \E[\bar h_{\barbtheta}(z, z_l, u_l)], 
\end{align*}
where the expectations are taken over $(z, z_l , u_l) \iidsim \normal(0,1)$. Using the assumptions imposed on $\sigma, \sigma'$, we can apply the dominated convergence theorem and conclude that $\bar c^{(l + 1)}(\cdot, \cdot)$ is a continuous function. We denote by $\bar \GG^{(l + 1)}$ the Gaussian process with mean zero and covariance function $\bar c^{(l + 1)}$. We define $\Omega_{l + 1} := \{\bx \in \RR^4: \|\bx\|_{\infty} \leq 2 H_{l - 1} \}$

Similar to the proof of \cref{lemma:Donsker2}, we can prove that equipped with the supremum norm $\ell^{\infty}(\Omega_{l + 1})$, $\{\bar \GG_d^{(l + 1)}\}_{d \geq 1}$ converges weakly in $C(\Omega_{l + 1})$ to $\bar \GG^{(l + 1)}$. We skip the detailed proof here for the sake of compactness. For $\btheta, \barbtheta \in \Omega_{l + 1}$, we define $\rho(\btheta, \barbtheta) := \E[(\bar \GG^{(l + 1)}(\btheta) - \bar \GG^{(l + 1)}(\barbtheta))^2]^{1/2}$. Then again by \cite[Lemma 18.15]{van2000asymptotic}, we can and will assume that $\bar \GG^{(l + 1)}$ almost surely has $\rho$-continuous sample path. 

For $H_{l - 1} > \ep > 0$, we let 
\begin{align*}
	B_{\ep} := \{\bx \in \RR^4: |x_1| < \ep, |x_2 - H_{l - 1}| < \ep, |x_3| < \ep, |x_4| < \ep\}, \qquad S^{(l + 1)}_{\ep}(\GG) := \sup_{\btheta \in B_{\ep}} |\GG(\btheta)|.
\end{align*}
Note that $S_{\ep}^{(l + 1)}$ is continuous with respect to $\ell^{\infty}(\Omega_{l + 1})$, thus $S_{\ep}^{(l + 1)}(\bar \GG_d^{(l + 1)}) \overset{d}{\to} S^{(l + 1)}_{\ep}(\bar \GG^{(l + 1)})$. Recall that we have proved $\mu_l = o_P(1), \nu_{l} = H_{l - 1} + o_P(1)$, $\zeta_l = O_P(d_l^{-1/2})$, $\gamma_l = o_P(1)$ and $z_i = \sqrt{d_l} W_{l + 1, i}$. Therefore, 
\begin{align*}
	& \big| F_l(\bg_l^s) - F_l(\bg_l) + \sqrt{d_l} \zeta_l  \E[\sigma'(\nu_l z_l) \sigma'((1 - \mu_l) \nu_l z_{l} - \zeta_l \sigma'(\nu_lz_{l}) z - \gamma_l u_{l} )] \big| \\
	\overset{(i)}{=} & \Big| \frac{1}{\sqrt{d_l}} \sum_{i = 1}^{d_l} z_i \big( \sigma((1 - \mu_l) \nu_l z_{l,i} - \zeta_l \sigma'(\nu_lz_{l,i}) z_i - \gamma_l u_{l,i} ) - \sigma(\nu_l z_{l,i}) \big) - \\
	&\sqrt{d_l} \E[z \big( \sigma((1 - \mu_l) \nu_l z_{l} - \zeta_l \sigma'(\nu_lz_{l}) z - \gamma_l u_{l} ) - \sigma(\nu_l z_{l}) \big)]  \Big| \\
	= & \Big| \frac{1}{\sqrt{d_l}} \sum_{i = 1}^{d_l} (h_{(\mu_l, \nu_l, \zeta_l, \gamma_l)}(z_i, z_{l,i}, u_{l,i}) - \E[h_{(\mu_l, \nu_l, \zeta_l, \gamma_l)}(z, z_{l}, u_{l})]) \Big| \\
	\overset{(ii)}{\leq} & S_{\ep}^{(l + 1)}(\bar\GG_d^{(l + 1)}) + \delta_{\ep}(d), 
\end{align*}
where $\delta_{\ep}(d) \toP 0$ as $d \to \infty$. In the above equations, the expectations are taken over $\{z_i, z_{l,i}, u_{l,i}, z, z_l, u_l\}_{i \in [d_l]} \iidsim \normal(0,1)$,  \emph{(i)} is by Stein's lemma, and \emph{(ii)} is by an argument that is similar to the proof of \cref{lemma:approx}. More precisely, we show that as $d \to \infty$
\begin{align*}
	\GG_d^{(l + 1)}((\mu_l, \nu_l, \zeta_l, \gamma_l)) = \bar\GG_d^{(l + 1)}((\mu_l, \nu_l, \zeta_l, \gamma_l)) + o_P(1).
\end{align*}
We ignore the proof of this part for the sake of simplicity, as it is basically the same as the proof of \cref{lemma:approx}. 

Since $\bar\GG^{(l + 1)}$ has $\rho$-continuous sample path, one can verify that $S_{\ep}^{(l + 1)}(\bar\GG^{(l + 1)}) \toP 0$ as $\ep \to 0^+$. 
For any $\ep' > 0$, we first choose $\ep$ small enough, such that $\P(S_{\ep}^{(l + 1)}(\bar\GG^{(l + 1)}) \geq \ep' / 3) \leq \ep' / 3$. Since $S_{\ep}^{(l + 1)}(\bar\GG_d^{(l + 1)}) \overset{d}{\to} S^{(l + 1)}_{\ep}(\bar\GG^{(l + 1)})$, and $\delta_{\ep}(d) \toP 0$ as $d \to \infty$, there exists $d_{\ep, \ep'} \in \NN_+$, such that for all $d \geq d_{\ep, \ep'}$, $\P(|\delta_{\ep}(d)| \geq \ep' / 3) \leq \ep' / 3$ and $\P(| S_{\ep}^{(l + 1)}(\bar\GG_d^{(l + 1)})| \geq \ep' / 3) \leq \P(|S_{\ep}^{(l + 1)}(\bar\GG^{(l + 1)})| \geq \ep' / 3) + \ep' / 3$. Combining these results, we conclude  that for all $d \geq d_{\ep, \ep'}$, 
\begin{align*}
	\P\big( \big| F_l(\bg_l^s) - F_l(\bg_l) + \sqrt{d_l} \zeta_l  \E[\sigma'(\nu_l z_l) \sigma'((1 - \mu_l) \nu_l z_{l} - \zeta_l \sigma'(\nu_lz_{l}) z - \gamma_l u_{l} )] \big| \geq \ep' \big) \leq \ep',
\end{align*}
Application of the dominated convergence theorem shows that $\E[\sigma'(\nu_l z_l) \sigma'((1 - \mu_l) \nu_l z_{l} - \alpha_l \sigma'(\nu_lz_{l}) z - \gamma_l u_{l} )] = \E[\sigma'(H_{l - 1} z)^2] + o_P(1)$ as $d \to \infty$. Recall that we have proved $\beta_l = O_P(1)$, thus 
\begin{align*}
	f(\bx^s) - f(\bx) = F_l(\bg_l^s) - F_l(\bg_l) = -\beta_l\E[\sigma'(H_{l - 1} z)^2] + o_P(1). 
\end{align*}
Since $\sigma$ is not a constant function, $\sigma'$ is almost everywhere continuous, and $H_{l - 1} > 0$, we then have $\E[\sigma'(H_{l - 1} z)^2] > 0$. Recall that we have shown that there exist random variables $\{\R_m\}_{2 \leq m \leq l}$, the distributions of which depend uniquely on $(\sigma, [l])$. In addition, we have shown that there exist positive constants $\{\alpha_m\}_{2 \leq m \leq l}$, with the values of which depend uniquely on $(\sigma, [l])$, such that $\beta_m \geq \alpha_m \beta_{m - 1} + R_m + o_P(1)$. Furthermore, using the law of large numbers, we have $\beta_1 = \tau S_0 + o_P(1)$. By \cref{lemma:Pconvergence}, $F_l(\bg_l)$ converges in distribution to a Gaussian distribution with mean zero and variance depending only on $(\sigma, l)$ (especially, independent of $\{s_d\}_{d \in \NN_+}$). Therefore, we deduce that
\begin{align*}
	\lim_{S_0 \to \infty}\liminf_{d \to \infty} \P(\sign(f(\bx)) \neq \sign(f(\bx^s))) = 1.
\end{align*}
Finally, we prove \cref{thm:multi-layer} via a standard diagonal argument. Note that for all $n \in \NN_+$, there exists $S_0^n > 0$ and $d_n \in \NN_+$, such that if we set $s_d = S_0^n$ for all $d \in \NN_+$, then for all $d \geq d_n$,
\begin{align*}
	\P(\sign(f(\bx^s)) \neq \sign(f(\bx))) \geq 1- n^{-1}. 
\end{align*} 
Without loss of generality, we assume that $d_{n + 1} \geq d_n$, $S_0^n / \sqrt{d_n} < n^{-1}$ and $S_0^n < \xi_{d_n}$. Indeed, to achieve these conditions, we simply need to take $d_n$ large enough. Then we set $s_d = S_0^n$ if and only if $d_n \leq d < d_{n + 1}$. Under such choice of $\{s_d\}_{d \in \NN_+}$, for all $d_{n + 1} > d \geq d_n$, we have 
\begin{align*}
	\frac{s_d}{\sqrt{d}} = \frac{S_0^n}{\sqrt{d}} \leq \frac{S_0^n}{\sqrt{d_n}} \leq \frac{1}{n}, \qquad \P(\sign(f(\bx^s)) \neq \sign(f(\bx))) \geq 1- n^{-1}, \qquad s_d \leq \xi_{d_n} \leq \xi_{d}.
\end{align*}
Note that $n$ is arbitrary, thus by combining the above results with the first claim of the theorem, we complete the proof of the second claim. 


\section*{Acknowledgements}
This work was supported by the NSF through award DMS-2031883, the Simons Foundation through 
Award 814639 for the Collaboration on the Theoretical Foundations of Deep Learning, 
the NSF grant CCF-2006489, the ONR grant N00014-18-1-2729.

\bibliographystyle{alpha}
\bibliography{bib}

\newpage
\begin{appendices}
\section{Proofs of the supporting lemmas}
\subsection{Proof of \cref{lemma:gaussian-conditioning}}\label{sec:proof-of-lemma:gaussian-conditioning}
We prove \cref{lemma:gaussian-conditioning} in this section. Note that $\bX = \Pi_{\bA_1} \bX + \Pi_{\bA_1}^{\perp} \bX$. Let $\bX'$ be an independent copy of $\bX$ that is independent of $(\bA_1, \bA_2, \bX, \bZ_1, \bZ_2)$. We consider the matrix $\bX_1 := \Pi_{\bA_1}^{\perp} \bX + \Pi_{\bA_1} \bX'$. Conditioning on any value of $(\bA_1, \bA_1 \bX, \bZ_1, \bZ_2)$, the conditional distribution of $\bX_1$ is equal to the distribution of $\bX$. Therefore, we conclude that $\bX_1 \overset{d}{=} \bX$ and $\bX_1$ is independent of $(\bA_1, \bA_1 \bX, \bZ_1, \bZ_2)$. Notice that $\bX$ admits the decomposition 
\begin{align*}
	\bX =& \Pi_{\bA_1} \bX + \Pi_{\bA_1}^{\perp} \bX \Pi_{\bA_2} + \Pi_{\bA_1}^{\perp} \bX \Pi_{\bA_2}^{\perp} \\
	= & \Pi_{\bA_1} \bX + \Pi_{\bA_1}^{\perp} \bX \Pi_{\bA_2} + \Pi_{\bA_1}^{\perp} \bX_1 \Pi_{\bA_2}^{\perp}.
\end{align*}
We let $\bX''$ be an independent copy of $\bX$ that is independent of $(\bA_1, \bA_2, \bX, \bX', \bZ_1, \bZ_2)$. Define
\begin{align*}
	\bX_2 = \Pi_{\bA_1}^{\perp} \bX_1 \Pi_{\bA_2}^{\perp} + \Pi_{\bA_1} \bX'' + \Pi_{\bA_1}^{\perp} \bX'' \Pi_{\bA_2}.
\end{align*}
Using distributional property of gaussian ensemble and the fact that $\bX_1$ is independent of $(\bA_1, \bA_1 \bX, \bZ_1, \bZ_2)$ (thus is independent of $(\bA_1, \bA_2, \bA_1 \bX, \bZ_1, \bZ_2)$), we can conclude that conditioning on any specific value of $(\bA_1, \bA_2, \bA_1 \bX, \bA_1 \bX_1,  \bX_1 \bA_2,  \bZ_1, \bZ_2)$, the conditional distribution of $\bX_2$ is equal to the distribution of $\bX$. Therefore, we deduce that $\bX_2 \perp (\bA_1, \bA_2, \bA_1 \bX, \bA_1 \bX_1,  \bX_1 \bA_2,  \bZ_1, \bZ_2)$. Notice that 
\begin{align*}
	\bX \bA_2 = &  \Pi_{\bA_1} \bX \bA_2 + \Pi_{\bA_1}^{\perp} \bX \bA_2 \\
	= & \Pi_{\bA_1} \bX \bA_2 + \bX_1 \bA_2 - \Pi_{\bA_1}\bX_1 \bA_2,
\end{align*}
thus $\bX_2 \perp (\bA_1 \bX, \bX \bA_2, \bZ_1)$. Combining the above analysis, we obtain that
\begin{align*}
	\bX = \Pi_{\bA_1} \bX + \Pi_{\bA_1}^{\perp} \bX \Pi_{\bA_2} + \Pi_{\bA_1}^{\perp} \bX_2 \Pi_{\bA_2}^{\perp},
\end{align*}
with $\bX_2$ independent of $\bY$. Thus, we have completed the proof of the lemma.

\subsection{Proof of \cref{lemma:control-gd-norm}}\label{sec:proof-of-lemma-control-gd-norm}
Conditioning on $\bW\bx$, the following two matrices are equal in distribution: 
\begin{align*}
	\bW \overset{d}{=} \frac{1}{d} \bW \bx \bx^{\top} + \tilde{\bW} \Pi_{\bx}^{\perp},
\end{align*}
where $\tilde{\bW}$ is an independent copy of $\bW$ which is independent of everything else, $\Pi_{\bx} \in \RR^{d \times d}$ is the projection operator projecting onto the linear subspace spanned by $\bx$, and $\Pi_{\bx}^{\perp} := \id_d - \Pi_{\bx}$. Using this decomposition, we can express the distribution of the gradient as
\begin{align*}
	\nabla f(\bx) \overset{d}{=} \Pi_{\bx}^{\perp} \tilde{\bW}^{\top} \bD_{\sigma} \ba + \frac{1}{d} \bx \bx^{\top} \bW^{\top} \bD_{\sigma} \ba. 
\end{align*}
By assumption, almost everywhere we have $|\sigma'(x)| \leq C_{\sigma}(1 + |x|^{k - 1}) $. Notice that $\bw_i^{\top} \bx \iidsim \normal(0,1)$ for $i \in [m]$. Therefore, we can apply Chebyshev's inequality and conclude that there exists a constant $C_1 > 0$ which depends only on $\sigma$, such that with probability at least $1 - \delta / 4$, 
\begin{align*}
	 \Big|\frac{1}{m}\sum_{i = 1}^m (\bw_i^{\top} \bx)^2\sigma'(\bw_i^{\top}\bx)^2 - \E[G^2\sigma'(G)^2] \Big| \leq \sqrt{\frac{C_1}{m\delta}},
\end{align*}
where $G \sim \normal(0,1)$. Conditioning on $\bg = \bW \bx$, we have $\bx^{\top} \bW^{\top}\bD_{\sigma} \ba \overset{d}{=} \normal(0,\sum_{i = 1}^m (\bw_i^{\top} \bx)^2\sigma'(\bw_i^{\top}\bx)^2 / m)$, then using Gaussian concentration, we conclude that there exists a numerical constant $C_2 > 0$, such that  with probability at least $1 - \delta / 4$, 
\begin{align*}
	|\bx^{\top} \bW^{\top}\bD_{\sigma} \ba| \leq \sqrt{ C_2 (\log (1 / \delta) + 1) \times \frac{1}{m} \sum_{i = 1}^m (\bw_i^{\top} \bx)^2\sigma(\bw_i^{\top}\bx)^2}.	
\end{align*}
Note that $\|\Pi_{\bx}^{\perp} \tilde{\bW}^{\top} \bD_{\sigma} \ba\|_2^2 \leq \|\tilde{\bW}^{\top} \bD_{\sigma} \ba\|_2^2$. Let $z_1, \cdots, z_d \iidsim \normal(0,1)$, then
\begin{align*}
	\|\tilde{\bW}^{\top} \bD_{\sigma} \ba\|_2^2 \overset{d}{=} \big(\sum_{i = 1}^m \sigma'(\bw_i^{\top} \bx)^2 a_i^2 \big) \times \big( \frac{1}{d}\sum_{j = 1}^d z_j^2 \big).
\end{align*}
Using Bernstein's inequality, with probability at least $1 - \delta / 4$, $|\frac{1}{d}\sum_{j = 1}^d z_j^2 - 1| \leq C_3(\log(1 / \delta) + 1) / \sqrt{ d}$ for some absolute constant $C_3$. Again by Chebyshev's inequality, with probability at least $1 - \delta / 4$, $|\sum_{i = 1}^m \sigma'(\bw_i^{\top} \bx)^2 a_i^2 - \E[\sigma'(G_1)^2G_2^2 ]| \leq \sqrt{C_4 / m\delta}$, where $G_1, G_2 \iidsim \normal(0,1)$ and $C_4 > 0$ is a constant depending only on $\sigma$. Then we combine the above results, and conclude that there exists $C > 0$ that depends only on $\sigma$, such that with probability at least $1 - \delta$, 
\begin{align*}
	\|\nabla f(\bx)\|_2 \leq C(1 + d^{-1/2}\log (1 / \delta) + (m\delta)^{-1/2} + (md\delta)^{-1/2}\log (1 / \delta) ),
\end{align*}  
thus concluding the proof of the lemma.

\subsection{Proof of \cref{lemma:coefficients-op1}}\label{sec:proof-of-lemma:coefficients-op1}

By assumption, almost everywhere we have $|\sigma'(x)| \leq C_{\sigma}(1 + |x|^{k - 1})$. Let $g_1$ be the first coordinate of $\bg$, then $\E[\sigma'(g_1)^2g_1^2]< \infty$. By the law of large numbers, $\|\bg^{\top} \bD_{\sigma}\|_2^2 / m = \E[g_1^2 \sigma'(g_1)^2] + o_P(1)$. Conditioning on $\bg^{\top} \bD_{\sigma}$, we have $\bg^{\top} \bD_{\sigma} \ba \overset{d}{=} \normal(\textbf{0}, \|\bg^{\top} \bD_{\sigma}\|_2^2 / m)$, thereby $\bg^{\top} \bD_{\sigma} \ba = O_P(1)$. By assumption we have $s_d \rightarrow S_0$, then we can conclude that $\tau s_d \bg^{\top} \bD_{\sigma} \ba / d= o_P(1)$. This proves $\mu = o_P(1)$.

Similarly, we apply the law of large numbers, and obtain that $\|\bD_{\sigma} \ba\|_2 = \E[\sigma'(g_1)^2]^{1/2} + o_P(1)$. By assumption, $\sigma$ is not a constant function and $\sigma'$ is almost everywhere continuous, thus we have $\E[\sigma'(g_1)^2  ] > 0$. Given $\bD_{\sigma} \ba$, the conditional distributions of $\bar\bW^{\top} \bD_{\sigma} \ba$ and $\bar\bW_c^{\top} \bD_{\sigma} \ba$ are both $\normal(\textbf{0}, \id_d\|\bD_{\sigma} \ba\|_2^2 / d)$, thus by the law of  large numbers and Cauchy–Schwarz inequality, we can conclude that
\begin{align*}
	 \|\bar\bW^{\top} \bD_{\sigma} \ba\|_2^2 = O_P(1), \qquad \langle \bar{\bW}_c^{\top} \bD_{\sigma} \ba, \bar\bW^{\top} \bD_{\sigma} \ba \rangle = O_P(1).
\end{align*}
Combining the equations above gives $\beta = o_P(1)$ and $\gamma = o_P(1)$ as $m,d \to \infty$. 

\subsection{Proof of \cref{lemma:Donsker}} \label{sec:proof-of-lemma-Donsker}
By our assumptions imposed on $\sigma$, we see that there exists a deterministic constant $C_0 > 0$, which is a function of the activation function $\sigma$ only, such that  for all $\btheta, \barbtheta \in \Omega$,
\begin{align*}
	|h_{\btheta}(b,g,u) - h_{\barbtheta}(b,g,u)| \leq  C_0\|\btheta - \barbtheta\|_2|b| \sqrt{g^2 + b^2 + b^2g^{2k} + u^2} \times \big(1 + |g|^k + |b|^k + |b|^k|g|^{k^2} + |u|^k \big).
\end{align*}
We define $m(b,g,u) := C_0|b| \sqrt{g^2 + b^2 + b^2g^{2k} + u^2} \times \big(1 + |g|^k + |b|^k + |b|^k|g|^{k^2} + |u|^k \big)$. One can verify that $\|m\|_2^2 := \E[m(b,g,u)^2] < \infty$, where the expectation is taken over $b,g,u \iidsim \normal(0,1)$. 

Let $\cF := \{h_{\btheta}: \btheta \in \Omega\}$. For the sake of completeness, we reproduce the definition of bracketing number introduced in \cite{van2000asymptotic}. Given two functions $e_1,e_2: \RR^3 \rightarrow \RR$, we define the bracket $[e_1,e_2]$ as the set of all functions $h$ such that $e_1(z) \leq h_z \leq e_2(z)$ for all $z \in \Omega$. An $\epsilon$-bracket in $L_2$ is a bracket $[e_1,e_2]$ such that $\E[(e_1(b,g,u) - e_2(b,g,u))^2]^{1/2} < \epsilon$, where again the expectation is taken over $b,g,u \iidsim \normal(0,1)$. The \emph{bracketing number} $N_{[\,]}(\ep, \cF, L_2)$ is the minimum number of $\ep$-brackets needed to cover $\cF$. We define the \emph{bracketing integral} as 
\begin{align*}
	J_{[\,]}(\delta, \cF, L_2) = \int_0^{\delta}\sqrt{\log N_{[\,]}(\ep, \cF, L_2)} \dd \ep.
\end{align*}
The following lemma is from \cite[Theorem 19.5]{van2000asymptotic}.
\begin{lemma}\label{lemma:vandervaart}
	If $J_{[\,]}(1, \cF, L_2) < \infty$, then $\GG_m$ converges weakly in $C(\Omega)$ to $\GG$. 
\end{lemma}
By \cref{lemma:vandervaart}, to prove the theorem, we only need to show that the bracketing integral is finite. By \cite[Example 19.7]{van2000asymptotic}, there exists a numerical constant $K > 0$, such that
\begin{align*}
	N_{[\,]}(\ep \|m\|_2, \cF, L_2) \leq K\ep^{-3}. 
\end{align*}
It is not hard to see that there exists another constant $K_0 > 0$, which depends only on $(K, \|m\|_2)$, such that
\begin{align*}
	J_{[\,]}(1, \cF, L_2) \leq K_0\int_0^1(\log(1 / \epsilon) + 1)^{1/2} \dd \ep  \leq K_0\int_0^1(\ep^{-1} + 1)^{1/2} \dd \ep < \infty,
\end{align*}
which concludes the proof of the lemma. 

\subsection{Proof of \cref{lemma:Pconvergence}}\label{sec:proof-of-lemma:Pconvergence}
\subsubsection*{Proof of claims 1 and 2}
We first prove claims 1 and 2 via induction over $m$. For the base case $m = 1$, the claims hold by the law of large numbers and the assumption that $\sigma$ is not a constant function. Suppose the claims hold for $1 \leq m \leq m_0$, then we prove it also holds for $m = m_0 + 1$ by induction. Conditioning on $\bh_{m_0}$, notice that
\begin{align*}
	\bg_{m_0 + 1} \overset{d}{=} \sqrt{\|\bh_{m_0}\|_2^2 / d_{m_0}} \bz,
\end{align*}
where $\bz \sim \normal(\mathbf{0}, \id_{d_{m_0 + 1}})$ and is independent of $\bh_{m_0}$. Therefore, 
$$\frac{1}{d_{m_0 + 1}}\|\bg_{m_0 + 1}\|_2^2 \overset{d}{=} \frac{1}{d_{m_0}} \|\bh_{m_0}\|_2^2 \times \frac{1}{d_{m_0 + 1}} \| \bz_{m_0 + 1}\|_2^2,$$ 
which converges to some positive deterministic constant by the law of large numbers and induction hypothesis. Similarly, conditioning on $\bh_{m_0}$, we have
\begin{align*}
	\frac{1}{d_{ m_0 \scriptstyle + 1}} \|\bh_{m_0 + 1}\|_2^2 \overset{d}{=} \frac{1}{d_{m_0 + 1}} \sum_{i = 1}^{d_{m_0 + 1}} \sigma\big( (\|\bh_{m_0}\|_2^2 / d_{m_0})^{1/2} z_i \big)^2,
\end{align*}
where $z_i$ is the $i$-th entry of $\bz$. 
By our induction hypothesis, $\|\bh_{m_0}\|_2^2 / d_{m_0}$ converges in probability to some constant $H_{m_0} > 0$. Since almost everywhere $|\sigma'(x)| \leq C_{\sigma}(1 + |x|^{k - 1})$ and $\sigma$ is continuous, we can conclude that 
\begin{align*}
	\frac{1}{d_{m_0 + 1}} \sum_{i = 1}^{d_{m_0 + 1}} \sigma\big( (\|\bh_{m_0}\|_2^2 / d_{m_0})^{1/2} z_i \big)^2 = \frac{1}{d_{m_0 + 1}} \sum_{i = 1}^{d_{m_0+ 1}} \sigma\big( H_{m_0}^{1/2} z_i \big)^2 + o_P(1),
\end{align*}
which further converges in probability to some positive constant by the law of large numbers and the non-degeneracy assumption on $\sigma$. Thus, we have completed the proof of the first two claims by induction. 

\subsubsection*{Proof of claims 3,4 and 5}
Then we prove claims 3,4 and 5, again via induction. 
We start with the base case $m = l$. Conditioning on $\bh_{l - 1}$, 
\begin{align*}
	\|\boeta_l\|_2^2 \overset{d}{=} \frac{1}{d_l}\sum_{i = 1}^{d_l} |z^{(l + 1)}_i|^2 \sigma'\big( (\|\bh_{l - 1}\|_2^2 / d_{l - 1})^{1/2} z_i^{(l)} \big)^2,
\end{align*}
where $\bz^{(l)}, \bz^{(l + 1)} \iidsim \normal(\mathbf{0}, \id_{d_l})$ and are independent of $\bh_{l - 1}$. Recall that almost everywhere we have $|\sigma'(x)| \leq C_{\sigma}(1 + |x|^{k - 1})$. By Chebyshev's inequality and claim 1 of \cref{lemma:Pconvergence},
\begin{align*}
	\frac{1}{d_l}\sum_{i = 1}^{d_l} |z^{(l + 1)}_i|^2 \sigma'\big( (\|\bh_{l - 1}\|_2^2 / d_{l - 1})^{1/2} z_i^{(l)} \big)^2 = \E_{z \sim \normal(0,1)}[\sigma'\big((\|\bh_{l - 1}\|_2^2 / d_{l - 1})^{1/2}z \big)^2] + o_P(1),
\end{align*}
where the expectation on the right hand side is taken over $z \sim \normal(0,1)$. By claim 2, there exists a constant $H_{l - 1} > 0$, such that $\|\bh_{l - 1}\|_2^2 / d_{l - 1} \toP H_{l - 1} $. Since $\sigma'$ is almost everywhere continuous, and $|\sigma'(x)| \leq C_{\sigma}(1 + |x|^{k - 1})$, we can apply the dominated convergence theorem and conclude that as $d \to \infty$,
$$ \E_{z \sim \normal(0,1)}[\sigma'\big((\|\bh_{l - 1}\|_2^2 / d_{l - 1})^{1/2}z \big)^2] \toP  \E_{z \sim \normal(0,1)}[\sigma'\big(H_{l - 1} ^{1/2}z \big)^2].$$ 
In summary, we have $\|\boeta_l\|_2^2 \toP \E_{z \sim \normal(0,1)}[\sigma'\big(H_{l - 1} ^{1/2}z \big)^2]$, thus completing the proof of claim 3 for the base case. 

Then we consider $\bh_{l - 1}^{\top} \bW_l^{\top} \boeta_l$ and the proof of claim 5. Since $\bW_{l + 1}$ is independent of $\bD_{\sigma}^l \bW_l \bh_{l - 1}$, conditioning on $\bD_{\sigma}^l\bW_l \bh_{l - 1}$, we have
\begin{align*}
	\bh_{l - 1}^{\top} \bW_l^{\top} \boeta_l = \langle \bD_{\sigma}^l\bW_l \bh_{l - 1},   \bW_{l + 1}^{\top} \rangle \overset{d}{=} \normal(0, \|\bD_{\sigma}^l\bW_l \bh_{l - 1}\|_2^2 / d_l). 
\end{align*}
Conditioning on $\bh_{l - 1}$, we have
$$\frac{1}{d_l}\|\bD_{\sigma}^l\bW_l \bh_{l - 1}\|_2^2 \overset{d}{=} \frac{1}{d_l}\sum_{i = 1}^{d_l} \sigma'((\|\bh_{l - 1}\|_2^2 / d_{l - 1})^{1/2}z_i^{(l)})^2((\|\bh_{l - 1}\|_2^2 / d_{l - 1})^{1/2}z_i^{(l)})^2,$$ 
which converges in probability to $\E_{z \sim \normal(0,1)} [H_{l - 1}\sigma'(H_{l - 1}^{1/2}z)^2z^2]$ as $d \to \infty$, again by Chebyshev's inequality and the dominated convergence theorem. Therefore, as $d \to \infty$, 
$$\bh_{l - 1}^{\top} \bW_l^{\top} \boeta_l \overset{d}{\rightarrow} \normal(0, \E_{z \sim \normal(0,1)} [H_{l - 1}\sigma'(H_{l - 1}^{1/2}z)^2z^2]),$$ 
which implies that $\bh_{l - 1}^{\top} \bW_l^{\top} \boeta_l = O_P(1)$. We have completed the proof of claim 5 for the base case.

As the last step towards proving the base case, we consider the Euclidean norm of $\by_l$, i.e., claim 4. Notice that $\boeta_l$ depends on $\bW_l$ only through $\bW_l \bh_{l - 1}$. Therefore, conditioning on $\bW_l\bh_{l - 1}$
\begin{align*}
	\by_l =& \bW_l^{\top} \boeta_l \\
	\overset{d}{=} &  \Pi_{\bh_{l - 1}}^{\perp} \tilde{\bW}_l^{\top} \bD_{\sigma}^l \bW_{l + 1}^{\top} + \frac{\bh_{l - 1}}{\|\bh_{l - 1}\|_2^2} \langle \bW_l \bh_{l - 1}, \bD_{\sigma}^l \bW_{l + 1}^{\top} \rangle \\
	= & \tilde{\bW}_l^{\top} \bD_{\sigma}^l \bW_{l + 1}^{\top}+ \frac{\bh_{l - 1}}{\|\bh_{l - 1}\|_2^2}( \langle \bW_l \bh_{l - 1}, \boeta_l \rangle - \langle \tilde\bW_l \bh_{l - 1}, \boeta_l \rangle ),
\end{align*}
where $\tilde{\bW}_l$ is an independent copy of $\bW_l$ and is independent of everything else. By claim 2 of the lemma, the vector ${\bh_{l - 1}} /{\|\bh_{l - 1}\|_2^2}$ has Euclidean norm $O_P(d_{l - 1}^{-1/2})$. Conditioning on $(\bh_{l - 1}, \boeta_l)$, we have $\langle \tilde\bW_l \bh_{l - 1}, \boeta_l \rangle \overset{d}{=} \normal(0, \|\boeta_{l}\|_2^2 \|\bh_{l - 1}\|_2^2 / d_{l - 1})$, which is $O_P(1)$ as $\|\boeta_{l}\|_2^2$, $\|\bh_{l - 1}\|_2^2 / d_{l - 1}$ are both $O_P(1)$ by claim 1 and claim 3. As proven above, we have $\langle \bW_l \bh_{l - 1}, \boeta_l \rangle = O_P(1)$. Finally, conditioning on $\boeta_l = \bD_{\sigma}^l \bW_{l + 1}^{\top}$, $\|\tilde{\bW}_l^{\top} \bD_{\sigma}^l \bW_{l + 1}^{\top}\|_2^2 \overset{d}{=} \sum_{i = 1}^{d_{l - 1}} z_i^2 \|\boeta_l\|_2^2 / d_{l - 1}$, where $\{z_i\}_{i \in [d_{l - 1}]} \iidsim \normal(0,1)$. By the law of large numbers, $\|\tilde{\bW}_l^{\top} \bD_{\sigma}^l \bW_{l + 1}^{\top}\|_2^2$ converges in probability to the same limit of $\|\boeta_l\|_2^2$ as $d \to \infty$. Combining the results above, we can conclude that $\|\by_l\|_2^2$ converges in probability to a positive constant as $d \to \infty$, thus concluding the proof of claim 4 of the base case.

 Suppose claims 3 to 5 hold for all $m_0 + 1 \leq m \leq l$, then we prove that they also hold for $m = m_0$. First notice that $(\boeta_{m_0 + 1}, \bD_{\sigma}^{m_0})$ depends on $\bW_{m_0 + 1}$ only through $\bW_{m_0 + 1} \bh_{m_0}$, thus conditioning on $(\bD_{\sigma}^{m_0}, \boeta_{m_0 + 1})$, we have
 \begin{align}\label{eq:etak}
 	\boeta_{m_0} &= \bD_{\sigma}^{m_0} \bW_{m_0 + 1}^{\top} \boeta_{m_0 + 1} \nonumber \\
 	& \overset{d}{=} \bD_{\sigma}^{m_0} \tilde\bW_{m_0 + 1}^{\top} \boeta_{m_0 + 1} + \frac{\bh_{m_0}^{\top} \bW_{m_0 + 1}^{\top} \boeta_{m_0 + 1} - \bh_{m_0}^{\top} \tilde{\bW}_{m_0 + 1}^{\top} \boeta_{m_0 + 1}}{\|\bh_{m_0}\|_2^2} \bD_{\sigma}^{m_0} \bh_{m_0},
 \end{align}
 where $\tilde{\bW}_{m_0 + 1}$ has the same marginal distribution as $\bW_{m_0 + 1}$, and is independent of everything else.
 By claim 1, $\|\bh_{m_0}\|_2^{-2} = O_P(d_{m_0}^{-1})$. By induction hypothesis, $\bh_{m_0}^{\top} \bW_{m_0 + 1}^{\top}\boeta_{m_0 + 1} = O_P(1)$. Conditioning on $(\bh_{m_0}, \boeta_{m_0 + 1})$, we have $\bh_{m_0}^{\top} \tilde{\bW}_{m_0 + 1}^{\top} \boeta_{m_0 + 1} \overset{d}{=} \normal(0, \|\bh_{m_0}\|_2^2 \|\boeta_{m_0 + 1}\|_2^2 / d_{m_0})$, which is $O_P(1)$ by induction hypothesis and claim 1. Conditioning on $\bh_{m_0 - 1}$, 
 \begin{align*}
 	\frac{1}{d_{m_0}}\|\bD_{\sigma}^{m_0} \bh_{m_0}\|_2^2 \overset{d}{=} & \frac{1}{d_{m_0}}\sum_{i = 1}^{d_{m_0}} \sigma'\big( (\|\bh_{m_0 - 1}\|_2^2 / d_{m_0 - 1})^{1/2} z_i^{(m_0)}  \big)^2 \sigma((\|\bh_{m_0 - 1}\|_2^2 / d_{m_0 - 1})^{1/2} z_i^{(m_0)})^2,
 \end{align*}
 where $\{z_i^{(m_0)} \}_{i \in [d_{m_0}]}\iidsim \normal(0,1)$ and are independent of $\bh_{m_0 - 1}$. By claim 2 of the lemma, there exists a constant $H_{m_0 - 1} > 0$, such that $\|\bh_{m_0 - 1}\|_2^2 / d_{m_0 - 1} \toP H_{m_0 - 1}$. By the assumption that $\sigma'$ is almost everywhere continuous, and $|\sigma'(x)| \leq C_{\sigma}(1 + |x|^{k - 1})$, we can apply Chebyshev's inequality and dominated convergence theorem and conclude that $\|\bD_{\sigma}^{m_0} \bh_{m_0}\|_2^2 / d_{m_0}$ converges in probability to some constant. In summary, the second term in \cref{eq:etak} has Euclidean norm $O_P(d_{m_0}^{-1/2})$. 
 
 Then we consider the first term in \cref{eq:etak}. Notice that conditioning on $\bh_{m_0 - 1}$, 
 \begin{align*}
 	\|\bD_{\sigma}^{m_0} \tilde{\bW}_{m_0 + 1}^{\top} \boeta_{m_0 + 1} \|_2^2 \overset{d}{=} \frac{1}{d_{m_0}}\sum_{i = 1}^{d_{m_0}} \sigma'\big( (\|\bh_{m_0 - 1}\|_2^2 / d_{m_0 - 1})^{1/2} z_i^{(m_0)}\big)^2 z_i^2 \|\boeta_{m_0 + 1}\|_2^2,
 \end{align*}
where $\bz, \bz^{(m_0)} \iidsim \normal(\mathbf{0}, \id_{d_{m_0}})$ and are independent of $\bh_{m_0 - 1}$. Again by Chebyshev's inequality and dominated convergence theorem, we have 
$$\frac{1}{d_{m_0}}\sum_{i = 1}^{d_{m_0}} \sigma'\big( (\|\bh_{m_0 - 1}\|_2^2 / d_{m_0 - 1})^{1/2} z_i^{(m_0)}\big)^2 z_i^2 \toP \E_{z \sim \normal(0,1)}[\sigma'(H_{m_0 - 1}^{1/2} z)^2].$$ By induction hypothesis, $\|\boeta_{m_0 + 1}\|_2^2$ converges in probability to a positive constant. Therefore, we conclude that $\|\bD_{\sigma}^{m_0} \tilde{\bW}_{m_0 + 1}^{\top} \boeta_{m_0 + 1} \|_2^2 $ converges in probability to a positive constant. Then we plug the above results into \cref{eq:etak}, and conclude that $\|\boeta_{m_0}\|_2^2$ converges in probability to a positive constant which depends only on $(\sigma, l)$. Thus, we have completed the proof of claim 3 for $m = m_0$. 

We next show that $\bh_{m_0 - 1}^{\top} \bW_{m_0}^{\top} \boeta_{m_0} = O_P(1)$. Notice that $(\boeta_{m_0 + 1}, \bD_{\sigma}^{m_0} \bW_{m_0} \bh_{m_0 - 1})$ depends on $\bW_{m_0 + 1}$ only through $\bW_{m_0 + 1} \bh_{m_0}$. Then conditioning on $(\boeta_{m_0 + 1}, \bD_{\sigma}^{m_0} \bW_{m_0} \bh_{m_0 - 1})$ we have $\bW_{m_0 + 1} \overset{d}{=}\tilde\bW_{m_0 + 1}\Pi_{\bh_{m_0}}^{\perp} + \bW_{m_0 + 1} \Pi_{\bh_{m_0}}$, where $\tilde{\bW}_{m_0 + 1}$ is an independent copy of $\bW_{m_0 + 1}$ and is independent of everything else. Therefore, 
\begin{align}\label{eq:Wheta}
	& \langle \bW_{m_0} \bh_{m_0 - 1}, \boeta_{m_0} \rangle \nonumber \\
	= & \langle \bD_{\sigma}^{m_0} \bW_{m_0} \bh_{m_0 - 1}, \bW_{m_0 + 1}^{\top} \boeta_{m_0 + 1} \rangle \nonumber \\
	\overset{d}{=} & \frac{\langle \bD_{\sigma}^{m_0}\bW_{m_0} \bh_{m_0 - 1}, \bh_{m_0}  \rangle}{\|\bh_{m_0}\|_2^2} (\bh_{m_0}^{\top} \bW_{m_0 + 1}^{\top} \boeta_{m_0 + 1} - \bh_{m_0}^{\top} \tilde{\bW}_{m_0 + 1} \boeta_{m_0 + 1}) + \langle \bD_{\sigma}^{m_0} \bW_{m_0} \bh_{m_0 - 1}, \tilde\bW_{m_0 + 1}^{\top} \boeta_{m_0 + 1} \rangle.
\end{align}
By claim 2 of the lemma, $\|\bh_{m_0}\|_2^{-2} = O_P(d_{m_0}^{-1})$. Conditioning on $\bh_{m_0 - 1}$, 
\begin{align*}
	& \frac{1}{d_{m_0}}\langle \bD_{\sigma}^{m_0}\bW_{m_0} \bh_{m_0 - 1}, \bh_{m_0}  \rangle \\
 \overset{d}{=} & \frac{1}{d_{m_0}} \sum_{i = 1}^{d_{m_0}} \sigma'\big((\|\bh_{m_0 - 1}\|_2^2 / d_{m_0 - 1})^{1/2} z_i\big)^2\sigma\big((\|\bh_{m_0 - 1}\|_2^2 / d_{m_0 - 1})^{1/2} z_i\big)^2 z_i^2  \|\bh_{m_0 - 1}\|_2^2   / d_{m_0 - 1},
\end{align*}
where $\bz \sim \normal(\mathbf{0}, \id_{d_{m_0}})$ and is independent of $\bh_{m_0 - 1}$. Again we apply Chebyshev's inequality and dominated convergence theorem, and conclude that 
$$\frac{1}{d_{m_0}}\langle \bD_{\sigma}^{m_0}\bW_{m_0} \bh_{m_0 - 1}, \bh_{m_0}  \rangle  \toP H_{m_0 - 1}\E_{z \sim \normal(0,1)}\Big[\sigma'\big(H_{m_0 - 1}^{1/2} z\big)^2 \sigma\big(H_{m_0 - 1}^{1/2}z\big)^2 z^2\Big]$$
as $d \to \infty$. By induction hypothesis, $\bh_{m_0}^{\top} \bW_{m_0 + 1}^{\top} \boeta_{m_0 + 1} = O_P(1)$. Conditioning on $(\bh_{m_0}, \boeta_{m_0 + 1})$, we have $\bh_{m_0}^{\top} \tilde\bW_{m_0 + 1}^{\top} \boeta_{m_0 + 1} \overset{d}{=} \normal(0, \|\bh_{m_0}\|_2^2 \|\boeta_{m_0 + 1}\|_2^2 / d_{m_0})$, which is $O_P(1)$ since by claim 1 of the lemma we have $\|\bh_{m_0}\|_2^2 / d_{m_0} = O_P(1)$, and by induction hypothesis we have $\|\boeta_{m_0 + 1}\|_2^2 = O_P(1)$. Combining the above analysis, we can concludet that the first summand in \cref{eq:Wheta} is $O_P(1)$. 

Then we consider the second summand in \cref{eq:Wheta}. Conditioning on $(\bD_{\sigma}^{m_0} \bW_{m_0} \bh_{m_0 - 1}, \boeta_{m_0 + 1})$, 
$$\langle \bD_{\sigma}^{m_0} \bW_{m_0} \bh_{m_0 - 1}, \tilde\bW_{m_0 + 1}^{\top} \boeta_{m_0 + 1} \rangle \overset{d}{=} \normal(0, \|\boeta_{m_0 + 1}\|_2^2 \|\bD_{\sigma}^{m_0} \bW_{m_0} \bh_{m_0 - 1}\|_2^2 / d_{m_0}). $$
Again we apply the conditioning technique. Conditioning on $\bh_{m_0 - 1}$ we have
\begin{align*}
	\frac{1}{d_{m_0}}\|\bD_{\sigma}^{m_0} \bW_{m_0} \bh_{m_0 - 1}\|_2^2 \overset{d}{=} \frac{1}{d_{m_0}} \sum_{i = 1}^{d_{m_0}} \sigma'\big((\|\bh_{m_0 - 1}\|_2^2 / d_{m_0 - 1})^{1/2}z_i \big)^2 (\|\bh_{m_0 - 1}\|_2^2 / d_{m_0 - 1})z_i^2 = O_P(1),
\end{align*}
where $\bz \sim \normal(\mathbf{0}, \id_{d_{m_0}})$ and is independent of $\bh_{m_0 - 1}$. By induction hypothesis, $\|\boeta_{m_0 + 1}\|_2^2 = O_P(1)$, thus $\langle \bD_{\sigma}^{m_0} \bW_{m_0} \bh_{m_0 - 1}, \tilde\bW_{m_0 + 1}^{\top} \boeta_{m_0 + 1} \rangle = O_P(1)$. Next, we plug the above analysis into \cref{eq:Wheta} and conclude that
 $\langle \bW_{m_0} \bh_{m_0 - 1}, \boeta_{m_0} \rangle = O_P(1)$, thus proving claim 5 for $m = m_0$. 

As the last step of our induction proof, we show that $\|\by_{m_0}\|_2^2$ converges in probability to some positive constant. Note that $\boeta_{m_0}$ depends on $\bW_{m_0}$ only through $\bW_{m_0} \bh_{m_0 - 1}$. As a result, conditioning on $\boeta_{m_0}$, 
\begin{align*}
	\by_{m_0} = \bW_{m_0}^{\top} \boeta_{m_0}  \overset{d}{=} \tilde{\bW}_{m_0}^{\top} \boeta_{m_0} + \frac{\bh_{m_0 - 1}^{\top} \bW_{m_0}^{\top} \boeta_{m_0} - \bh_{m_0 - 1}^{\top} \tilde\bW_{m_0}^{\top} \boeta_{m_0}}{\|\bh_{m_0 - 1}\|_2^2} \bh_{m_0 - 1},
\end{align*}
where $\tilde{\bW}_{m_0}$ is an independent copy of $\bW_{m_0}$ and is independent of everything else. By induction hypothesis, $\|\boeta_{m_0}\|_2^2$ converges in probability to some positive constant. By the law of large numbers,  we have $\|\tilde{\bW}_{m_0}^{\top} \boeta_{m_0}\|_2^2$ converges in probability to the same limit of $\|\boeta_{m_0}\|_2^2$ as $d \to \infty$. By claim 2 of the lemma, $\|\bh_{m_0 - 1}\|_2^{-2} = O_P(d_{m_0 - 1}^{-1})$ and $\|\bh_{m_0 - 1}\|_2 = O_P(d_{m_0 - 1}^{1/2})$. Note that we have proved $\bh_{m_0 - 1}^{\top} \bW_{m_0}^{\top} \boeta_{m_0} = O_P(1)$. Conditioning on $(\bh_{m_0 - 1}, \boeta_{m_0})$, we have $\bh_{m_0 - 1}^{\top} \tilde\bW_{m_0}^{\top} \boeta_{m_0} \overset{d}{=} \normal(0,\|\bh_{m_0 - 1}\|_2^2 \|\boeta_{m_0}\|_2^2 / d_{m_0 - 1})$ which is $O_P(1)$ by claim 1 and induction hypothesis. In summary, we can conclude that $\|\by_{m_0}\|_2^2$ and $\|\boeta_{m}\|_2^2$ converges in probability to the same limit, thus completing the proof of the lemma by induction.

\subsection{Proof of \cref{lemma:finite-sample}}\label{sec:proof-of-lemma:finite-sample}

We first provide finite sample upper bounds on the Euclidean norms of $\{\bg_i, \bh_i, \bD_{\sigma}^i \bg_i\}_{i \in [l]}$. 

\begin{lemma}\label{lemma:upper-bound-1}
	Under the conditions of \cref{thm:multi-layer}, there exist positive constants $\{ Q_i\}_{1 \leq i \leq l}$, which depend only on $(\sigma, l)$, such that for all $1 \leq i \leq l$, with probability at least $1 - \delta$,  we have 
	\begin{align}
		& \frac{1}{\sqrt{d_i}} \|\bh_i\|_2 \leq Q_i \prod_{j = 1}^i \Big(1 + \delta^{-1/2}d_j^{-1/2} \Big)^{k^{i - j}} , \label{eq:13} \\
		& \frac{1}{\sqrt{d_i}} \|\bg_i\|_2 \leq Q_i \Big(1 + \log(1 / \delta) d_i^{-1/2} \Big) \prod_{j = 1}^{i - 1} \Big(1 + \delta^{-1/2}d_j^{-1/2} \Big)^{k^{i - 1 - j}} , \label{eq:14} \\
		& \frac{1}{\sqrt{d_i}}\frac{ \|\bD_{\sigma}^i \bg_i\|_2}{\|\bh_{i - 1}\|_2} \leq \frac{Q_i }{\sqrt{d_{i - 1}}}\prod_{j = 1}^i \Big(1 + \delta^{-1/2}d_j^{-1/2} \Big)^{k^{i - j}}. \label{eq:15}
	\end{align}
\end{lemma}
\begin{proof}
	We prove this lemma by induction over $i$. For the base case $i = 1$, $\bg_1 \sim \normal(\mathbf{0}, \id_{d_1})$, thus, \cref{eq:14} follows from Bernstein's inequality. Since almost everywhere we have $|\sigma'(x)| \leq C_{\sigma}(1 + |x|^{k - 1})$, then there exists $C_{\sigma}' > 0$ which is a function of $\sigma$ only, such that almost everywhere $|\sigma(x)| \leq C_{\sigma}'(1 + |x|^k)$ and $|\sigma'(x)x| \leq C_{\sigma}'(1 + |x|^k)$.
 Then with probability 1,  we have  
	\begin{align*}
		& \frac{1}{d_1}\|\bh_1\|_2^2 \leq \frac{2}{d_1} \sum_{i = 1}^{d_1} (C_{\sigma}')^2(1 + |g_{1,i}|^{2k}), \\
		& \frac{1}{d_1\|\bx\|_2^2} \|\bD_{\sigma}^1 \bg_1 \|_2^2 \leq \frac{2}{d_1d} \sum_{i = 1}^{d_1} (C_{\sigma}')^2(1 + |g_{1,i}|^{2k}),
	\end{align*}
	thus, \cref{eq:13,eq:15} for the base case follow from Chebyshev's inequality. Now suppose the lemma holds for all $i \leq m$, then we prove it also holds for $i = m + 1$ by induction. Conditioning on $\bh_m$, we have $\bg_{m + 1} \overset{d}{=} \normal(0, (\|\bh_m\|_2^2 / d_m) \id_{d_{m + 1}})$. Thus, by Bernstein's inequality, there exists an absolute constant $C > 0$, such that with probability at least $1 - \delta / 3$, $\|\bg_{m + 1}\|_2^2 / d_{m + 1} \leq C\|\bh_m\|_2^2(1 + \log(1 / \delta) / \sqrt{d_{m + 1}}) / d_m $. By induction hypothesis, with probability at least $1 - \delta / 3$, 
	\begin{align*}
		\frac{1}{\sqrt{d_m}} \|\bh_m\|_2 \leq 3^{mk^{m}}Q_m \prod_{j = 1}^m \Big(1 + \delta^{-1/2}d_j^{-1/2} \Big)^{k^{m - j}},
	\end{align*}
	thus, \cref{eq:14} for $i = m + 1$ follows. To prove \cref{eq:13,eq:15}, note that
	\begin{align*}
		\frac{1}{d_{m + 1}}\|\bh_{m + 1}\|_2^2 \leq &  \frac{1}{d_{m + 1}} \sum_{i = 1}^{d_{m + 1}} (C_{\sigma}')^2(1 + |g_{m + 1, i}|^{k})^2 \\
		\overset{d}{=} & \frac{1}{d_{m + 1}} \sum_{i = 1}^{d_{m + 1}} 2(C_{\sigma}')^2(1 + |z_{i}|^{2k} \|\bh_m\|_2^{2k} / d_m^{k}), \\
		\frac{1}{d_{m + 1}} \frac{\|\bD_{\sigma}^{m + 1} \bg_{m + 1}\|_2^2}{\|\bh_m\|_2^2} \leq  & \frac{1}{d_{m + 1} \|\bh_m\|_2^2} \sum_{i = 1}^{d_{m + 1}} \sigma'(g_{m + 1, i})^2 g_{m + 1, i}^2 \\
		\overset{d}{=} &  \frac{1}{d_{m + 1}d_m } \sum_{i = 1}^{d_{m + 1}} \sigma'(z_i\|\bh_m\|_2 / \sqrt{d_m})^2 z_i^2  \\
		\leq &  \frac{1}{d_{m + 1}d_m} \sum_{i = 1}^{d_{m + 1}} 4(C_{\sigma})^2z_i^2 (1 + z_i^{2k - 2} \|\bh_m\|_2^{2k - 2} / d_m^{k - 1}) , 
	\end{align*}
	where $\bz \sim \normal(\mathbf{0}, \id_{d_{m + 1}})$, and is independent $\bh_m$. Thus, \cref{eq:13,eq:15} for $i = m + 1$ follows from Chebyshev's inequality, and we complete the proof of the lemma by induction.
\end{proof}
Next, we control the Euclidean norms of $\{\boeta_i\}_{i \in [l]}$.
 \begin{lemma}\label{lemma:upper-bound-2}
 	Under the conditions of \cref{thm:multi-layer}, there exist constants $\{\tilde{Q}_i\}_{1 \leq i \leq l}$, which depend only on $(\sigma, l)$, such that for all $1 \leq i \leq l$, with probability at least $1 - \delta$,  we have
 	\begin{align}
 		& \|\boeta_i\|_2  \leq \tilde{Q}_i ( \sqrt{\log(1 / \delta)} + 1)^{l - i} \prod_{m = i}^l  \prod_{j = 1}^m \big(1 + \delta^{-1/2} d_j^{-1/2} \big)^{k^{m - j}} , \label{eq:16} \\
 		& \frac{|\bh_{i - 1}^{\top} \bW_{i}^{\top} \boeta_{i}|}{\|\bh_{i - 1}\|_2} \leq \frac{\tilde{Q}_i ( \sqrt{\log(1 / \delta)} + 1)^{l - i}}{\sqrt{d_{i - 1}}}\prod_{m = i}^l \prod_{j = 1}^m \big(1 + \delta^{-1/2} d_j^{-1/2} \big)^{k^{m - j}}.  \label{eq:17}
 	\end{align}
 \end{lemma}
 \begin{proof}
 	We prove this lemma by induction over $i$. We start with the base case $i = l$. Conditioning on $\bh_{l - 1}$, we have
 	\begin{align*}
 		\|\boeta_l\|_2^2  \overset{d}{=} &  \frac{1}{d_l}\sum_{i = 1}^{d_l} z_{l + 1, i}^2 \sigma'(\|\bh_{l - 1}\|_2 / \sqrt{d_{l - 1}} z_{l, i})^2 \\
 		\leq & \frac{8C_{\sigma}^2}{d_l}\sum_{i = 1}^{d_l} z_{l + 1, i}^2 + \frac{8C_{\sigma}^2}{d_l}\sum_{i = 1}^{d_l} z_{l + 1, i}^2 z_{l,i}^{2k - 2} (\|\bh_{l - 1}\|_2 / \sqrt{d_{l - 1}})^{2k - 2}, 
 	\end{align*}  
 	where $\bz_l, \bz_{l + 1} \iidsim \normal(\mathbf{0}, \id_{d_l})$ and is independent of $\bh_{l - 1}$. Then \cref{eq:16} for $i = l$ follows from \cref{lemma:upper-bound-1} and Chebyshev's inequality. Note that $\bW_{l + 1}$ is independent of $(\bD_{\sigma}^l, \bg_l)$, then
 	\begin{align*}
 		\frac{|\bh_{l - 1}^{\top} \bW_l^{\top} \boeta_l|^2}{\|\bh_{l - 1}\|_2^2} \overset{d}{=} &  \frac{1}{d_l d_{l - 1}} \sum_{i = 1}^{d_l} z_{l + 1, i}^2  z_{l,i}^2 \sigma'(z_{l,i}\|\bh_{l - 1}\|_2 / \sqrt{d_{l - 1}})^2 \\
 		\leq &   \frac{8C_{\sigma}^2}{d_l d_{l - 1}} \sum_{i = 1}^{d_l} z_{l + 1, i}^2  z_{l,i}^2 \big( 1 + z_{l,i}^{2k - 2} \|\bh_{l - 1}\|_2^{2k - 2}d_{l - 1}^{-k + 1} \big),
 	\end{align*} 
 	As a result, \cref{eq:17} for $i = l$ follows from \cref{lemma:upper-bound-1} and Chebyshev's inequality. Thus, we have completed the proof for the base case. 
 	
 	Suppose the lemma is true for all $m + 1 \leq i \leq l$, then we prove it also holds for $i = m$ via induction hypothesis. By \cref{eq:etak}, 
 	\begin{align}\label{eq:etam}
 		\boeta_m \overset{d}{=}  \bD_{\sigma}^m  \tilde{\bW}_{m + 1}^{\top} \boeta_{m + 1} + \frac{\bh_m^{\top} \bW_{m + 1}^{\top} \boeta_{m + 1} - \bh_m^{\top} \tilde\bW_{m + 1}^{\top} \boeta_{m + 1}}{\|\bh_m\|_2^2} \bD_{\sigma}^m \bh_m,
 	\end{align}
 	where $\tilde{\bW}_{m + 1}$ has the same distribution as $\bW_{m + 1}$, and is independent of everything else. The right hand side of \cref{eq:etam} has Euclidean norm no larger than 
 	\begin{align*}
 		\|\bD_{\sigma}^m  \tilde{\bW}_{m + 1}^{\top} \boeta_{m + 1}\|_2 + \frac{|\bh_m^{\top} \bW_{m + 1}^{\top} \boeta_{m + 1}| + |\bh_m^{\top} \tilde\bW_{m + 1}^{\top} \boeta_{m + 1}|}{\|\bh_m\|_2^2} \|\bD_{\sigma}^m \bh_m\|_2.
 	\end{align*}
 	Note that 
 	\begin{align*}
 		\|\bD_{\sigma}^m  \tilde{\bW}_{m + 1}^{\top} \boeta_{m + 1}\|_2^2 \overset{d}{=} \frac{1}{d_m}\sum_{i = 1}^{d_m}\sigma'(z_{m,i}\|\bh_{m - 1}\|_2 / \sqrt{d_{m - 1}})^2 z_{m + 1, i}^2 \|\boeta_{m + 1}\|_2^2,
 	\end{align*}
 	where $\bz_m, \bz_{m + 1} \iidsim \normal(\mathbf{0}, \id_{d_m})$, and are independent of $\bh_{m - 1}$. Using the fact that almost everywhere $|\sigma'(x)| \leq C_{\sigma}(1 + |x|^{k - 1})$, together with \cref{lemma:upper-bound-1} and Chebyshev's inequality, we conclude that there exists a constant $\tilde{Q}_m^{(1)} > 0$, depending only on $(\sigma, l)$, such that with probability at least $1 - \delta / 6$, 
 	\begin{align}\label{eq:19}
 		\|\bD_{\sigma}^m  \tilde{\bW}_{m + 1}^{\top} \boeta_{m + 1}\|_2 \leq \tilde{Q}_m^{(1)} \prod_{j = 1}^m  \big(1 + \delta^{-1/2} d_j^{-1/2} \big)^{k^{m - j}} \|\boeta_{m + 1}\|_2. 
 	\end{align}
 	Conditioning on $(\bh_m, \boeta_{m + 1})$, we have $\bh_m^{\top} \tilde\bW_{m + 1}^{\top} \boeta_{m + 1} / \|\bh_m\|_2 \overset{d}{=} \normal(0, \|\boeta_{m + 1}\|_2^2 / d_m)$. Therefore, with probability at least $1 - \delta / 6$, we have
 	\begin{align}\label{eq:20}
 		\frac{|\bh_m^{\top} \tilde\bW_{m + 1}^{\top} \boeta_{m + 1}|}{\|\bh_m\|_2} \leq d_m^{-1/2}\tilde{Q}_m^{(2)} \sqrt{\log (1 / \delta)} \|\boeta_{m + 1}\|_2,
 	\end{align}
 	where $\tilde{Q}_m^{(2)} > 0$ is a numerical constant.		
 
	Conditioning on $\bh_{m - 1}$, $\|\bD_{\sigma}^m \bh_m\|_2^2 / \|\bh_m\|_2^2 \leq \|\bD_{\sigma}^m\|_F^2 \overset{d}{=} \sum_{i = 1}^{d_m} \sigma'(\|\bh_{m - 1}\|_2 / \sqrt{d_{m - 1}}z_{m,i})^2$, where $\bz_m \sim \normal(\mathbf{0}, \id_{d_m})$ and is independent of $\bh_{m - 1}$. Since
 	\begin{align*}
 		 \frac{1}{d_m}\sum_{i = 1}^{d_m} \sigma'(\|\bh_{m - 1}\|_2 / \sqrt{d_{m - 1}}z_{m,i})^2 \leq \frac{8C_{\sigma}^2}{d_m}\sum_{i = 1}^{d_m}\big(1 + |z_{m,i}|^{2k - 2}\|\bh_{m - 1}\|_2^{2k - 2}d_{m - 1}^{-k + 1} \big),
 	\end{align*}
 	then by \cref{lemma:upper-bound-1} and Chebyshev's inequality, we conclude that there exists $\tilde{Q}_m^{(3)} > 0$, depending only on $(\sigma, l)$, such that with probability at least $1 - \delta / 6$, 
 	\begin{align}\label{eq:21}
 		\frac{\|\bD_{\sigma}^m \bh_m\|_2}{\|\bh_m\|_2} \leq \sqrt{d_m} \tilde{Q}_m^{(3)}\prod_{j = 1}^m  \big(1 + \delta^{-1/2} d_j^{-1/2} \big)^{k^{m - j}}.
 	\end{align}
 	Combining \cref{eq:etam,eq:19,eq:20,eq:21}, we conclude that there exists $\tilde{Q}_m^{(4)} > 0$, depending only on $(\sigma, l)$, such that with probability at least $1 - \delta / 2$, 
 	\begin{align}\label{eq:22}
 		\small\|\boeta_m\|_2 \leq \tilde{Q}_m^{(4)} ( \sqrt{\log(1 / \delta)} + 1) \times \Big\{ \prod_{j = 1}^m  \big(1 + \delta^{-1/2} d_j^{-1/2} \big)^{k^{m - j}} \Big\} \times \Big\{ \|\boeta_{m + 1}\|_2 + \frac{\sqrt{d_m} \|\bh_m^{\top} \bW_{m + 1}^{\top} \boeta_{m + 1}\|}{\|\bh_m\|_2}\Big\},
 	\end{align}
 	thus \cref{eq:16} for $i = m$ follows from \cref{eq:22} and induction hypothesis. Then we proceed to prove \cref{eq:17} for $i = m$. By \cref{eq:Wheta}, there exists $\tilde{\bW}_{m + 1}$ that has the same marginal distribution as $\bW_{m + 1}$ and is independent of everything else, such that
 	\begin{align}\label{eq:23}
 		& \frac{|\bh_{m - 1}^{\top} \bW_{m}^{\top} \boeta_{m}|}{\|\bh_{m - 1}\|_2} \nonumber  \\ \overset{d}{= } & \Big|\frac{\langle \bD_{\sigma}^m\bW_m \bh_{m - 1}, \bh_m  \rangle}{\|\bh_m\|_2^2\|\bh_{m - 1}\|_2} (\bh_m^{\top} \bW_{m + 1}^{\top} \boeta_{m + 1} - \bh_m^{\top} \tilde{\bW}_{m + 1} \boeta_{m + 1}) + \frac{\langle \bD_{\sigma}^m \bW_m \bh_{m - 1}, \tilde\bW_{m + 1}^{\top} \boeta_{m + 1} \rangle}{\|\bh_{m - 1}\|_2} \Big| \nonumber  \\
 		\leq &  \frac{|\langle \bD_{\sigma}^m\bW_m \bh_{m - 1}, \bh_m  \rangle|}{\|\bh_m\|_2^2\|\bh_{m - 1}\|_2} |\bh_m^{\top} \bW_{m + 1}^{\top} \boeta_{m + 1} - \bh_m^{\top} \tilde{\bW}_{m + 1} \boeta_{m + 1}| + \frac{|\langle \bD_{\sigma}^m \bW_m \bh_{m - 1}, \tilde\bW_{m + 1}^{\top} \boeta_{m + 1} \rangle|}{\|\bh_{m - 1}\|_2}.  	
 	\end{align} 
 	By \cref{lemma:upper-bound-1}, with probability at least $1 - \delta / 12$, 
 	\begin{align}\label{eq:24}
 		\frac{ \|\bD_{\sigma}^m \bg_m\|_2}{\|\bh_{m - 1}\|_2} \leq \frac{6^{mk^m}Q_m \sqrt{d_m}}{\sqrt{d_{m - 1}}}\prod_{j = 1}^m \Big(1 + \delta^{-1/2}d_j^{-1/2} \Big)^{k^{m - j}}.
 	\end{align}
 	Conditioning on $(\bh_m, \boeta_{m + 1})$, $\bh_m^{\top} \tilde{\bW}_{m + 1} \boeta_{m + 1} / \|\bh_{m}\|_2 \overset{d}{=} \normal(0, \| \boeta_{m + 1}\|^2_2 / d_m)$. Therefore, with probability at least $1 - \delta / 6$, we have
 	\begin{align}\label{eq:25}
 		\frac{|\bh_m^{\top} \tilde{\bW}_{m + 1} \boeta_{m + 1}|} {\|\bh_{m}\|_2} \leq \frac{\tilde{Q}_m^{(5)} \sqrt{\log(1 / \delta)} \|\boeta_{m + 1}\|_2}{ \sqrt{d_m}},
 	\end{align}
 where $\tilde{Q}_m^{(5)} > 0$ is a numerical constant. 

Given $(\bD_{\sigma}^m \bW_m \bh_{m - 1}, \boeta_{m + 1})$, the conditional distribution of $\langle \bD_{\sigma}^m \bW_m \bh_{m - 1}, \tilde\bW_{m + 1}^{\top} \boeta_{m + 1} \rangle / \|\bh_{m - 1}\|_2$  is $\normal(0, d_m^{-1}\|\bD_{\sigma}^m \bg_m\|_2^2 \|\boeta_{m + 1}\|_2^2 / \|\bh_{m - 1}\|_2^2)$. Therefore, using \cref{eq:24}, with probability at least $1 - \delta / 6$, 
\begin{align}\label{eq:25.5}
	\frac{|\langle \bD_{\sigma}^m \bW_m \bh_{m - 1}, \tilde\bW_{m + 1}^{\top} \boeta_{m + 1} \rangle|}{\|\bh_{m - 1}\|_2} \leq \frac{\tilde Q_m^{(6)} \|\boeta_{m + 1}\|_2 }{\sqrt{d_{m - 1}}} \times \Big\{\prod_{j = 1}^m \Big(1 + \delta^{-1/2}d_j^{-1/2} \Big)^{k^{m - j}} \Big\} \times \sqrt{\log(1 / \delta)}, 
\end{align}
where $\tilde Q_m^{(6)} > 0$ is a constant depending only on $(\sigma, l)$.
  Using \cref{eq:23,eq:24,eq:25.5,eq:25} and induction hypothesis, we conclude that with probability at least $1 - \delta / 2$, the last line in \cref{eq:23} is no larger than
  \begin{align}\label{eq:26}
  	\frac{\tilde{Q}_m^{(7)}( \sqrt{\log(1 / \delta)} + 1)}{\sqrt{d_{m - 1}}} \times \Big\{\prod_{j = 1}^m \Big(1 + \delta^{-1/2}d_j^{-1/2} \Big)^{k^{m - j}} \Big\} \times \Big\{ \|\boeta_{m + 1}\|_2 + \frac{\sqrt{d_m} \|\bh_m^{\top} \bW_{m + 1}^{\top} \boeta_{m + 1}\|}{\|\bh_m\|_2} \Big\},
  \end{align}
where $\tilde Q_m^{(7)} > 0$ is a constant depending only on $(\sigma, l)$. Thus, \cref{eq:17} for $i = m$ follows from \cref{eq:26} and induction hypothesis. Therefore, we have completed the proof of the lemma by induction. 
 \end{proof}
Finally, with \cref{lemma:upper-bound-1,lemma:upper-bound-2}, we are ready to prove \cref{lemma:finite-sample}. Note that $\boeta_1$ depends on $\bW_1$ only through $\bW_1 \bx$, thus by property of Gaussian distribution, we have 
\begin{align*}
	\nabla f(\bx) = \bW_1^{\top} \boeta_1 \overset{d}{=} \frac{\bx^{\top} \bW_1^{\top} \boeta_1 - \bx^{\top} \tilde\bW_1^{\top} \boeta_1}{\|\bx\|_2^2} \bx + \tilde{\bW}_1^{\top} \boeta_1,
\end{align*} 
where $\tilde{\bW}_1$ is an independent copy of $\bW_1$, and is independent of everything else. We let $\bz \sim \normal(\mathbf{0}, \id_{d})$, independent of everything else, then we have
\begin{align*}
	\|\tilde{\bW}_1^{\top}  \boeta_1\|_2 \overset{d}{=} \frac{\|\boeta_1\|_2 \|\bz\|_2}{\sqrt{d}}, \qquad \frac{|\bx^{\top} \bW_1^{\top} \boeta_1|}{\|\bx\|_2} = \frac{|\bh_0^{\top} \bW_1^{\top} \boeta_1|}{\|\bh_0\|_2}, \qquad  \frac{|\bx^{\top} \tilde\bW_1^{\top} \boeta_1|}{\|\bx\|_2} \overset{d}{=} \normal(0, \|\boeta_1\|_2^2 / d).
\end{align*}
Therefore, by \cref{lemma:upper-bound-2} and Bernstein's inequality, with probability at least $1 - \delta$, we have 
$$\|\nabla f(\bx)\|_2 \leq  Q( \sqrt{\log(1 / \delta)} + 1)^{l - 1}(1 + \log(1 / \delta)d^{-1/2}) \prod_{i = 1}^l  \prod_{j = 1}^i \big(1 + \delta^{-1/2} d_j^{-1/2} \big)^{k^{i - j}},$$
where $Q > 0$ is a constant depending only on $(\sigma, l)$, thus completing the proof of \cref{lemma:finite-sample}.

\subsection{Proof of \cref{lemma:op1-base}}\label{sec:proof-of-lemma:op1-base}
By definition, $\mu_1 = \tau s_d\bx^{\top} \by_1 / d $, thus by Cauchy–Schwarz inequality we have $|\mu_1| \leq s_d \|\by_1\|_2 / \sqrt{d}$. By \cref{lemma:Pconvergence}, as $d \to \infty$ we have $\|\by_1\|_2 = O_P(1)$, thus $\mu_1 = o_P(1)$. Next, we consider $\gamma_1$. Since $\bar\bW_1$ is independent of $\boeta_1$, by the law of large numbers, as $d \to \infty$, we have
\begin{align*}
	\plim_{d \to \infty} \|\bar\bW_1^{\top}\boeta_1\|_2 = \plim_{d \to \infty}  \|\boeta_1\|_2,
\end{align*}
which is finite and positive. As a result, $\gamma_1 = o_P(1)$. Finally, we consider $\beta_1$. Notice that $|\langle (\bW_1')^{\top} \boeta_1, \bar\bW_1^{\top} \boeta_1\rangle| \leq \|(\bW_1')^{\top} \boeta_1\|_2 \|\bar\bW_1^{\top} \boeta_1\|_2$, which further converges to $(\plim \|\boeta_1\|_2)^2$ as $d \to \infty$. Similarly, we can show $\|\bar{\bW}_1 \boeta_1\|_2^2$ converges to $(\plim \|\boeta_1\|_2)^2$ as $d \to \infty$. Therefore, we have $\beta_1 = O_P(1)$, thus completing the proof of the lemma. 

\subsection{Proof of \cref{lemma:hk-1_sigma}}\label{sec:proof-of-lemma:hk-1_sigma}

By definition, $\boeta_1 = \bD_{\sigma}^1 \bW_2^{\top} \boeta_2$. Recall that $\bu_1$ is independent of $\sigma\{\bar{\bW}_1^{\top} \boeta_1, \cF_1\}$, then we can conclude that $(\bg_1, \bD_{\sigma}^1, \boeta_2, \bu_1)$ depends on $\bW_2$ only through $\bW_2 \bh_1$. Therefore, by \cref{lemma:gaussian-conditioning}, there exists $\bar\bW_2 \in \RR^{d_2 \times d_1}$, which has the same marginal distribution with $\bW_2$ and is independent of $(\bg_1, \bD_{\sigma}^1, \boeta_2, \bu_1)$, such that 
\begin{align*}
	\bg_1^s =  & \bg_1(1 - \mu_1) - \beta_1 \bD_{\sigma}^1 \Pi_{\bh_1} \bW_2^{\top} \boeta_2 - \beta_1 \bD_{\sigma}^1  \Pi_{\bh_1}^{\perp} \bar\bW_2^{\top} \boeta_2 - \gamma_1 \bu_1\\
	 =&  \bg_1(1 - \mu_1) + \beta_1 \bD_{\sigma}^1\bh_1 \frac{\bh_1^{\top} \bar\bW_2^{\top} \boeta_2 - \bh_1^{\top} \bW_2^{\top} \boeta_2}{\|\bh_1\|_2^2} - \gamma_1 \bu_1 - \beta_1 \bD_{\sigma}^1 \bar\bW_2^{\top} \boeta_2 \\
	= & \bg_1(1 - \mu_1) + \zeta_1 \bD_{\sigma}^1 \sigma(\bg_1) - \gamma_1 \bu_1 - \beta_1\bar\gamma_1 \bD_{\sigma}^1 \bar\bu_1 / \sqrt{d_1}, 
\end{align*}
where $\bar\bu_1 \sim \normal(\mathbf{0}, \id_{d_1})$, is independent of $(\bg_1, \bD_{\sigma}^1, \boeta_2, \bu_1)$ and
\begin{align*}
	\zeta_1 := \beta_1 \frac{\bh_1^{\top} \bar\bW_2^{\top} \boeta_2 - \bh_1^{\top} \bW_2^{\top} \boeta_2}{\|\bh_1\|_2^2}, \qquad \bar\gamma_1 := \|\boeta_2\|_2.
\end{align*}
Recall that in \cref{lemma:op1-base}, we have shown that $\mu_1 = o_P(1)$, $\beta_1  = O_P(1)$, $\gamma_1 = o_P(1)$ as $d \to \infty$. By \cref{lemma:Pconvergence} claim 5, 1 and 3, $\bh_1^{\top} \bW_2^{\top} \boeta_2 = O_P(1)$, $\|\bh_1\|_2^{-2} = O_P(d_1^{-1})$, and $\bar\gamma_1 = O_P(1)$. By controlling the second moment then applying Chebyshev's inequality, we can conclude that $\bh_1^{\top} \bar\bW_2^{\top} \boeta_2 = O_P(1)$. Therefore, $\zeta_1 = o_P(1)$ and $\beta_1 \bar{\gamma}_1 / \sqrt{d_1} = o_P(1)$.

For $\btheta \in \RR^4$, we define
\begin{align*}
	m_{\btheta} (g, u, \bar u) := \sigma(g(1 - \theta_1) + \theta_2 \sigma'(g) \sigma(g) - \theta_3 u - \theta_4  \sigma'(g)\bar u)\sigma(g).
\end{align*} 
By assumption $|\sigma'(x)| \leq C_{\sigma}(1 + |x|^{k - 1})$. Using this assumption, we can conclude that for any $\btheta, \bar\btheta \in \RR^4$ satisfying $\|\btheta\|_{\infty}, \|\bar\btheta\|_{\infty} \leq 1$,  
\begin{align*}
	|m_{\btheta}(g, u, \bar{u}) - m_{\bar\btheta}(g, u, \bar{u})| \leq C_{\sigma}' \| \btheta - \barbtheta\|_2 (1 + |g|^{n(k)} + |u|^{n(k)} + |\bar{u}|^{n(k)}), 
\end{align*}
where $n(k) \in \NN_+$ is a function of $k$ and $C_{\sigma}' > 0$ is a function of $\sigma$. Notice that $ \E_{(g,u,\bar{u})\sim \normal(\mathbf{0},\id_3)}[1 + |g|^{n(k)} + |u|^{n(k)} + |\bar{u}|^{n(k)}] < \infty$. Then by Example 19.7 and Theorem 19.4 in \cite{van2000asymptotic}, we know that $\{m_{\btheta}: \|\btheta\|_{\infty} \leq 1\}$ is a Glivenko-Cantelli class, thus
\begin{align*}
	 \sup_{\|\btheta\|_{\infty} \leq 1} \Big| \frac{1}{d_1}\sum_{i = 1}^{d_1} m_{\btheta}(g_i, u_i, \bar{u}_i) - \E_{(g,u,\bar{u}) \sim \normal(\mathbf{0}, \id_3)}[m_{\btheta}(g, u, \bar{u})] \Big| = o_P(1).
\end{align*}
Using the equation above and dominated convergence theorem, we have as $d \to \infty$
\begin{align*}
	\frac{1}{d_1}\langle \bh_1, \sigma(\bg_1^s)\rangle  = & \frac{1}{d_1}\sum_{i = 1}^{d_1} m_{(\mu_1, \zeta_1, \gamma_1, \beta_1 \bar{\gamma}_1 / \sqrt{d_1})}(g_i, u_i, \bar{u}_i)\\
	= & \E_{(g,u,\bar{u}) \sim \normal(\mathbf{0}, \id_3)}[ m_{(\mu_1, \zeta_1, \gamma_1, \beta_1 \bar{\gamma}_1 / \sqrt{d_1})}(g, u, \bar{u}) ] + o_P(1) \\
	 = & \E_{(g,u,\bar{u}) \sim \normal(\mathbf{0}, \id_3)}[m_{\mathbf{0}}(g, u, \bar{u})] + o_P(1).
\end{align*}
By the law of large numbers, $\|\bh_1\|_2^2 = \E_{(g,u,\bar{u}) \sim \normal(\mathbf{0}, \id_3)}[m_{\mathbf{0}}(g, u, \bar{u})] + o_P(1)$, thus proving the second claim of the lemma.
 Again via proving uniform convergence type result we can conclude that $\|\sigma(\bg_1^s)\|_2^2 / d_1 = \|\bh_1\|_2^2 / d_1 + o_P(1)$, thus completing the proof of the first result.

\subsection{Proof of \cref{lemma:op1}}\label{sec:proof-of-lemma:op1}
Given \cref{eq:mubetagamma}, $\mu_{m + 1} = o_P(1)$ and $\gamma_{m + 1} = o_P(1)$ follow from  claim \emph{(iv)} of  induction hypothesis $\mathcal{H}_m$. 

\subsection{Proof of \cref{lemma:Donsker2}}\label{sec:proof-of-lemma:Donsker2}

By assumption, almost everywhere we have $|\sigma'(x)| \leq  C_{\sigma}(1 + |x|^{k - 1})$. Therefore, there exists $C_{\sigma}' > 0$ which is a constant depending only on $\sigma$, such that $|\sigma(x)| \leq C_{\sigma}'(1 + |x|^k)$. 

Using these facts, we conclude that there exists $n(k) \in \NN_+$ which is a function of $k$, and $C_{\sigma}'' > 0$ which is a function of $(\sigma, \Omega_{m + 1})$ only, such that  for all $\btheta, \bar\btheta \in \Omega_{m + 1}$, we have 
\begin{align*}
	|\bar h^{(m + 1)}_{\btheta}(\bar{u}, z, u) - \bar h^{(m + 1)}_{\bar\btheta}(\bar{u}, z, u)| \leq {C}_{\sigma}''\|\btheta - \bar\btheta\|_2(1 + |\bar{u}|^{n(k)} + |z|^{n(k)} + |u|^{n(k)}).
\end{align*}
Note that $\E_{(\bar{u}, z, u) \sim \normal(\mathbf{0}, \id_3)}[(1 + |\bar{u}|^{n(k)} + |z|^{n(k)} + |u|^{n(k)})^2] < \infty$. The rest of the proof is almost identical to that of \cref{lemma:Donsker}: We apply the well-known results regarding Donsker class in \cite{van2000asymptotic} and prove that $\{\bar h^{(m + 1)}_{\btheta}: \btheta \in \Omega_{m + 1}\}$ is a Donsker class via proving the corresponding bracketing integral is finite. Here, we skip the details to avoid duplication.

\subsection{Proof of \cref{lemma:approx}} \label{sec:proof-of-lemma:approx}
We define $\cS, \bar{\cS}: \RR^{3 \times d_m} \to \RR$ as follows:
\begin{align*}
	& \cS(\bar{\bu}_{m + 1}, \bz_m, \bu_m)  := \frac{1}{\sqrt{d_m}} \sum_{i = 1}^{d_m} h^{(m + 1)}_{(\bar\gamma_{m + 1}, \mu_{m}, \nu_m, \beta_{m} / \sqrt{d_{m}}, \beta_{m} \delta_{m + 1}, \gamma_{m})} (\bar{u}_{m + 1, i}, z_{m,i}, u_{m,i}), \\
	& \bar\cS(\bar{\bu}_{m + 1}, \bz_m, \bu_m) := \frac{1}{\sqrt{d_m}} \sum_{i = 1}^{d_m} \bar{h}^{(m + 1)}_{(\bar\gamma_{m + 1}, \mu_{m}, \nu_m, \beta_{m} / \sqrt{d_{m}}, \beta_{m} \delta_{m + 1}, \gamma_{m})} (\bar{u}_{m + 1, i}, z_{m,i}, u_{m,i}).
\end{align*}
%
By assumption, almost everywhere we have $|\sigma'(x)| \leq  C_{\sigma}(1 + |x|^{k - 1})$. Therefore, there exists $C_{\sigma}' > 0$ which is a constant depending only on $\sigma$, such that 
\begin{align*}
	|\sigma(x)| \leq C_{\sigma}'(1 + |x|^{k}). 
\end{align*}
Using these facts, via some computation we can conclude that there exist $n_0 \in \NN_+$ and $\bar C_{\sigma} > 0$ depending only on $(\sigma, \Omega_{m + 1})$, such that 
\begin{align*}
	& \big|\cS(\bar{\bu}_{m + 1}, \bz_m, \bu_m) - \bar\cS(\bar{\bu}_{m + 1}, \bz_m, \bu_m)\big| \\
	\leq &  \frac{\bar C_{\sigma}}{d_m}\sum_{i = 1}^{d_m} |\sigma'(\nu_m z_{m,i}) - \sigma'(H_{m - 1}z_{m,i})| \times \big(|\beta_m| \bar{\gamma}_{m + 1}^2\bar{u}_{m + 1, i}^2 + |d_m^{1/2}\beta_m \delta_{m + 1}\bar{\gamma}_{m + 1}\bar{u}_{m + 1, i} |(1 + \nu_m^k|z_{m,i}|^k) \big) \times \\
	& (1 + |\bar{u}_{m + 1, i}|^{n_0} + |z_{m,i}|^{n_0} + |u_{m,i}|^{n_0}) \times (1 + |\mu_m|^{n_0} + |\nu_m|^{n_0} + |\beta_m|^{n_0} + |\bar{\gamma}_{m + 1}|^{n_0} + |\gamma_m|^{n_0} + |\delta_{m + 1}|^{n_0}).
\end{align*}
By \cref{eq:indep}, $\nu_m$ is independent of $(\bar{\bu}_{m + 1}, \bz_m, \bu_m)$. Recall that we have proved $\beta_m = O_P(1)$, $\delta_{m + 1} = O_P(d_m^{-1})$, $\nu_m = H_{m - 1} + o_P(1)$, $\bar{\gamma}_{m + 1} = E_{m + 1} + o_P(1)$, $\gamma_m = o_P(1)$, $\mu_m = o_P(1)$. Note that $\sigma'$ is almost everywhere continuous and $|\sigma'(x)| \leq C_{\sigma}(1 + |x|^{k - 1})$. Then by applying Chebyshev's inequality and dominated convergence theorem to the above equation, we can conclude that as $d \to \infty$
\begin{align}\label{eq:34}
	\big|\cS(\bar{\bu}_{m + 1}, \bz_m, \bu_m) - \bar\cS(\bar{\bu}_{m + 1}, \bz_m, \bu_m)\big| = o_P(1).
\end{align}
Similarly, we have
\begin{align}\label{eq:35}
	\big|\E[\cS(\bar{\bu}_{m + 1}, \bz_m, \bu_m)] - \E[\bar\cS(\bar{\bu}_{m + 1}, \bz_m, \bu_m)]\big| = o_P(1),
\end{align}
where in the above equation the expectations are taken over $(\bar{\bu}_{m + 1}, \bz_m, \bu_m)$, assuming 
$$(\mu_m, \nu_m, \beta_m, \bar{\gamma}_{m + 1}, \gamma_m, \delta_{m + 1})$$ are fixed. Note that
\begin{align*}
	& \GG_d^{(m + 1)}(\bar\gamma_{m + 1}, \mu_{m}, \nu_m, \beta_{m} / \sqrt{d_{m}}, \beta_{m} \delta_{m + 1}, \gamma_{m}) = \cS(\bar{\bu}_{m + 1}, \bz_m, \bu_m) - \E[\cS(\bar{\bu}_{m + 1}, \bz_m, \bu_m)], \\
	& \bar\GG_d^{(m + 1)}(\bar\gamma_{m + 1}, \mu_{m}, \nu_m, \beta_{m} / \sqrt{d_{m}}, \beta_{m} \delta_{m + 1}, \gamma_{m}) = \bar\cS(\bar{\bu}_{m + 1}, \bz_m, \bu_m) - \E[\bar\cS(\bar{\bu}_{m + 1}, \bz_m, \bu_m)],
\end{align*}
thus we have completed the proof of the lemma using \cref{eq:34,eq:35}.

\subsection{Proof of \cref{lemma:hk-sigma}}\label{sec:proof-of-lemma:hk-sigma}
By definition, $\boeta_{m + 1} = \bD_{\sigma}^{m + 1} \bW_{m + 2}^{\top} \boeta_{m + 2}$. Note that $(\bg_{m + 1}, \bD_{\sigma}^{m + 1}, \boeta_{m + 2}, \bu_{m + 1})$ depends on $\bW_{m + 2}$ only through $\bg_{m + 2} = \bW_{m + 2} \bh_{m + 1}$ and $\bh_{m + 1}$ is independent of $\bW_{m + 2}$. By \cref{lemma:gaussian-conditioning}, there exists $\bar\bW_{m + 2} \in \RR^{d_{m + 2} \times d_{m + 1}}$ that has the same marginal distribution as $\bW_{m + 2}$ and  
$\bar\bW_{m + 2}  \perp(\bg_{m + 1}, \bD_{\sigma}^{m + 1}, \boeta_{m + 2}, \bu_{m + 1})$, such that 
\begin{align*}
	&\bg_{m + 1}^s \\
	 =&  (1 - \mu_{m + 1})\bg_{m + 1} + \beta_{m + 1} \bD_{\sigma}^{m + 1} \bh_{m + 1}\frac{\bh_{m + 1}^{\top} \bar{\bW}_{m + 2}^{\top} \boeta_{m + 2} - \bh_{m + 1}^{\top} {\bW}_{m + 2}^{\top} \boeta_{m + 2}}{\|\bh_{m + 1}\|_2^2} - \\
	 & \beta_{m + 1} \bD_{\sigma}^{m + 1} \bar{\bW}_{m + 2}^{\top} \boeta_{m + 2} - \gamma_{m + 1} \bu_{m + 1}, \\
	= & (1 - \mu_{m + 1})\bg_{m + 1} + \kappa_{m + 1} \bD_{\sigma}^{m + 1} \sigma(\bg_{m + 1}) - \gamma_{m + 1} \bu_{m + 1} - \beta_{m + 1} \bar{\gamma}_{m + 2} \bD_{\sigma}^{m + 1} \bar{\bu}_{m + 2} / \sqrt{d_{m + 1}},
\end{align*}
where $\bar{\bu}_{m + 2} \sim \normal(\mathbf{0}, \id_{d_{m + 1}})$ is independent of $(\bg_{m + 1}, \bD_{\sigma}^{m + 1}, \boeta_{m + 2}, \bu_{m + 1})$, and
\begin{align*}
	\kappa_{m + 1} := \beta_{m + 1} \frac{\bh_{m + 1}^{\top} \bar{\bW}_{m + 1}^{\top} \boeta_{m + 2} - \bh_{m + 1}^{\top} {\bW}_{m + 2}^{\top} \boeta_{m + 2}}{\|\bh_{m + 1}\|_2^2}, \qquad \bar{\gamma}_{m + 2} := \|\boeta_{m + 2}\|_2. 
\end{align*}
By $\mathcal{H}_{m + 1}$ claim \emph{(iii)}, $\mu_{m + 1} = o_P(1)$, $\gamma_{m + 1} = o_P(1)$ and $\beta_{m + 1} / \sqrt{d_{m + 1}} = o_P(1)$. By \cref{lemma:Pconvergence}, $\|\bh_{m + 1}\|_2^{-2} = O_P(d_{m + 1}^{-1})$, $\bar{\gamma}_{m +  1} = O_P(1)$, $\bh_{m + 1}^{\top} \bW_{m + 2}^{\top} \boeta_{m + 2} = O_P(1)$. Then we compute the corresponding second moment and apply Chebyshev's inequality, and conclude that $\bh_{m + 1}^{\top} \bar{\bW}_{m + 2}^{\top} \boeta_{m + 2} = O_P(1)$. Combining these analysis, we have $\kappa_{m + 1} = o_P(1)$. 

We can write $\bg_{m + 1} = \nu_{m + 1} \bz_{m + 1}$, where $\nu_{m + 1} := \sqrt{\|\bh_{m}\|_2^2 / d_{m}}$ and $\bz_{m + 1} \sim \normal(\mathbf{0}, \id_{d_{m + 1}})$ is independent of $\bh_m$. Adopting similar arguments we applied to obtain \cref{eq:indep}, we can conclude that $\bg_{m + 1}, \bar{\bu}_{m + 2}, \bu_{m + 1} \iidsim \normal(\mathbf{0}, \id_{d_{m + 1}})$. Recall that we have proved $\mu_{m + 1} = o_P(1)$, $\kappa_{m + 1} = o_P(1)$, $\gamma_{m + 1} = o_P(1)$ and $\beta_{m + 1} \bar{\gamma}_{m + 2} / \sqrt{d_{m + 1}} = o_P(1)$. By assumption, for all $x,y \in \RR$, 
\begin{align}\label{eq:36}
	|\sigma(x) - \sigma(y)| \leq C_{\sigma}|x - y|(1 + |x|^{k - 1} + |y|^{k - 1}).
\end{align}
Then we apply \cref{eq:36} to bound the difference between $\bh_{m + 1} = \sigma(\bg_m)$ and $\sigma(\bg_m^s)$, and the lemma follows from simple application of the law of large numbers. 

\section{Proofs for the non-asymptotic results}

\subsection{Proof of \cref{lemma:control-mu-beta-gamma}}\label{sec:proof-of-lemma:control-mu-beta-gamma}

We will heavily rely on the Bernstein's inequality to prove the lemma, and we state it here for readers' convenience. 

\begin{lemma}\label{lemma:bernstein}
	Let $X_1, \cdots, X_N$ be independent, mean zero, sub-exponential random variables. Then, for every $t \geq 0$, we have
	\begin{align*}
		\P\left( \left| \sum_{i = 1}^N X_i \right| \geq t \right) \leq 2\exp \left( -c \min \left\{ \frac{t^2}{\sum_{i = 1}^N \|X_i\|_{\Psi_1}^2}, \frac{t}{\max_{i \in [N]}\|X_i\|_{\Psi_1}} \right\} \right),
	\end{align*}
	where $c > 0$ is a numerical constant, and $\|\cdot\|_{\Psi_1}$ is the Orlicz norm. 
\end{lemma}

\subsubsection*{Proof of the first result}

It suffices to prove that with probability at least $1 - C\eta_1^{-2}(\sigma(0)^2 + L^2)$ for some positive numerical constant $C > 0$, the following inequality holds:  
\begin{align*}
	|\bg^{\top} \bD_{\sigma} \ba| \leq \eta_1.
\end{align*}
Notice that $\bg^{\top} \bD_{\sigma} \ba = m^{-1/2}\sum_{i = 1}^m g_i \sigma(g_i) b_i$. Therefore, $\Var[\bg^{\top} \bD_{\sigma} \ba] \leq \E_{b,g \sim_{i.i.d.} \normal(0,1)}[b^2 g^2 \sigma(g)^2] \leq C(\sigma(0)^2 + L^2)$ for some numerical constant $C > 0$. By Chebyshev's inequality, with probability at least $1 - C\eta_1^{-2}(\sigma(0)^2 + L^2)$, $|\bg^{\top} \bD_{\sigma} \ba| \leq \eta_1$, thus completing the proof for this part.

\subsubsection*{Proof of the second result}

By the definitions of $\bar\bW$ and $\bar\bW_c$, we see that 
\begin{align*}
	\frac{\|\bar\bW^{\top} \bD_{\sigma} \ba\|_2^2 - \langle \bar{\bW}_c^{\top} \bD_{\sigma} \ba, \bar\bW^{\top} \bD_{\sigma} \ba \rangle}{\|\bD_{\sigma} \ba\|_2^2} \overset{d}{=} \frac{1}{d} \sum_{i = 1}^d (z_i^2 - z_iz_i'),
\end{align*}
where $z_i, z_i' \iidsim \normal(0,1)$. By \cref{lemma:bernstein}, we see that there exists a numerical constant $C > 0$, such that for all $\eta_2 \geq 1$,  with probability at least $1 - 4\exp(-C\eta_2)$  the following inequalities hold:
\begin{align*}
	\left| \frac{1}{d} \sum_{i = 1}^d z_i^2  - 1\right| \leq \frac{\eta_2}{\sqrt{d}}, \qquad \left| \frac{1}{d} \sum_{i = 1}^d z_i z_i' \right| \leq \frac{\eta_2}{\sqrt{d}}. 
\end{align*}
When the event described above occurs, we see that
\begin{align*}
	\left|\beta - \frac{\tau s_d}{\sqrt m}\right| \leq \frac{2s_d \eta_2}{\sqrt{md}},
\end{align*}
which completes the proof of the second result. 

\subsubsection*{Proof of the third result}

Notice that
\begin{align*}
	\gamma^2 = \frac{s_d^2}{d}\cdot \|\bar\bW^{\top} \bD_{\sigma} \ba\|_2^2 \overset{d}{=} \frac{s_d^2}{d} \cdot \left( \frac{1}{d	} \sum_{i = 1}^d z_i^2 \right) \cdot \left( \frac{1}{m} \sum_{i = 1}^m b_i^2 \sigma(g_i)^2 \right),
\end{align*}
where $z_i, b_i, g_i \iidsim \normal(0,1)$. Since $|\sigma(x)| \leq |\sigma(0)| + L|x|$, we then conclude that there exists a numerical constant $C > 0$, such that 
\begin{align*}
	\Var\left[ \frac{1}{m} \sum_{i = 1}^m b_i^2 \sigma(g_i)^2\right] \leq \frac{C(\sigma(0)^4 + L^4)}{m}.
\end{align*}
By Chebyshev's inequality, with probability at least $1 - Cm^{-1}(\sigma(0)^4 + L^4)$
\begin{align*}
	\left| \frac{1}{m} \sum_{i = 1}^m b_i^2 \sigma(g_i)^2 - \E_{g \sim \normal(0,1)}[\sigma(g)^2] \right| \leq 1.
\end{align*}
By Bernstein's inequality (\cref{lemma:bernstein}), with probability at least $1 - 2 \exp(-cd)$ for some absolute constant $c > 0$, we have $d^{-1}\sum_{i = 1}^d z_i^2 \leq 2$. In summary, with probability at least $1 - 2 \exp(-cd) - Cm^{-1}(\sigma(0)^4 + L^4)$, 
\begin{align*}
	\gamma \leq \frac{2s_d}{\sqrt d} \cdot \sqrt{ 1 + \E_{g \sim \normal(0,1)}[\sigma(g)^2] },
\end{align*}
thus concluding the proof of the third result.

\subsection{Proof of \cref{lemma:upper-bound-sub-gaussian-norm}}\label{sec:proof-of-lemma:upper-bound-sub-gaussian-norm}

Recall that the sub-Gaussian norm of a random variable $X$ is defined as 
\begin{align*}
	\|X\|_{\Psi_2} := \inf\left\{ t > 0: \E[\exp(X^2 / t^2 )] \leq 2 \right\}.
\end{align*}
Standard computation implies that for all $\lambda \in \RR$,
\begin{align*}
	& \E\left[ \exp\left( \lambda (\LL_m(\btheta) - \LL_m(\btheta'))  \right)  \right] \\
	\leq & \exp \left( \frac{\lambda^2}{2m} \sum_{i = 1}^m b_i^2\left( \sigma((1 - \theta_1) g_i - \theta_2 b_i \sigma'(g_i) - \theta_3 u_i) - \sigma((1 - \theta_1') g_i - \theta_2' b_i \sigma'(g_i) - \theta_3' u_i) \right)^2 \right) \\
	\leq & \exp \left( \frac{\lambda^2 L^2}{2m} \sum_{i = 1}^m b_i^2\left(|\theta_1 - \theta_1'|\cdot |g_i| + L \cdot |\theta_2 - \theta_2'| \cdot |b_i| + |\theta_3 - \theta_3'| \cdot |u_i| \right)^2 \right) \\
	\leq & \exp \left( \frac{\lambda^2\|\btheta - \btheta'\|_2^2}{2} \cdot \frac{1}{m} \sum_{i = 1}^m M(b_i, g_i, u_i)^2\right).
\end{align*}
Hence by the sub-Gaussian property, $\|\LL_m(\btheta) - \LL_m(\btheta')\|_{\Psi_2} \leq C\|\btheta - \btheta'\|_2 \cdot \sqrt{\frac{1}{m} \sum_{i = 1}^m M(b_i, g_i, u_i)^2}$ for some positive numerical constant $C$. 

\subsection{Proof of \cref{lemma:upper-bound-Fg}} \label{sec:proof-of-lemma:upper-bound-Fg}

We observe that $F(\bg) \overset{d}{=}  m^{-1}z\sum_{i = 1}^m \sigma(g_i)^2$, where $z \sim \normal(0,1)$ is independent of $\{g_i\}_{i \leq m}$. Notice that there exists a numerical constant $C' > 0$, such that $\|\sigma(g)^2\|_{\Psi_1} \leq C'(\sigma(0)^2 + L^2)$. Then by \cref{lemma:bernstein}, there exist numerical constants $c, C > 0$, such that 
\begin{align*}
	& \P\left( \left| \frac{1}{m}\sum_{i = 1}^m \sigma(g_i)^2 - \E_{g \sim \normal(0,1)}[\sigma(g)^2] \right| \geq 1 \right) \leq C \exp\left( -\frac{cm}{L^4 + \sigma(0)^4} \right), \\
	& \P\left(|\,z\,| \geq \eta_3  \right) \leq C \exp(-c\eta_3^2),
\end{align*}
which concludes the proof of the lemma.

\end{appendices}

\end{document}

%% file: macro.tex
\usepackage{float} 
\usepackage{amsmath}
\usepackage{amssymb}
\usepackage{array}
\usepackage{bbm}
\usepackage{makeidx}
\usepackage{hyperref}
\usepackage[capitalise]{cleveref}
\usepackage{amsthm}
\usepackage{color}
\usepackage{mathtools}
\usepackage{graphicx} 
\usepackage[LGR,T1]{fontenc}
\usepackage{siunitx}
\usepackage{mathrsfs}
\usepackage{scalerel}
\usepackage{forest}
\usepackage{algorithm,algpseudocode}
\usepackage{lipsum}
\usepackage[inline]{}
\usepackage{xcolor}
\usepackage[title]{appendix}
\hypersetup{
    colorlinks,
    linkcolor={blue!80!black},
    citecolor={green!50!black},
}
\usepackage{accents}
\usepackage{apptools}
\usepackage{tikz}
\usetikzlibrary{arrows, positioning, automata}
\usepackage{comment}
\usepackage[shortlabels]{enumitem}

\usepackage[mathscr]{euscript}
\DeclareSymbolFont{rsfs}{U}{rsfs}{m}{n}
\DeclareSymbolFontAlphabet{\mathscrsfs}{rsfs}

\AtAppendix{\counterwithin{lemma}{section}}

\usepackage{mathtools}

\theoremstyle{plain}
\renewenvironment{proof}{\noindent{\bfseries Proof.}}{\begin{flushright} \qedsymbol\end{flushright}}

\newtheorem{theorem}{Theorem}[section]

\newtheorem{lemma}{Lemma}[section]

\newtheorem{remark}{Remark}[section]
\let\olddefinition\remark
\renewcommand{\remark}{\olddefinition\normalfont}

\allowdisplaybreaks
\usepackage[top=1in,bottom=1in,left=1in,right=1in]{geometry}
\usepackage[utf8]{inputenc}

\def\dd{\mathrm{d}}

\def\top{\intercal}

\def\barbtheta{\bar{\vtheta}}

\DeclareMathOperator*{\plim}{p-lim}

\def\<{\langle}
\def\>{\rangle}
\def\sadv{\mbox{\tiny\rm adv}}
\def\reals{{\mathbb R}}

\def\cF{\mathcal{F}}
\def\cG{\mathcal{G}}

\def\cS{\mathcal{S}}

\def\vtheta{\boldsymbol{\theta}}

\def\cF{\mathcal{F}}

\def\ep{\epsilon}

\def\GG{\mathbb{G}}

\def\LL{\mathbb{L}}

\def\NN{\mathbb{N}}

\def\RR{\mathbb{R}}

\def\bA{\mathbf{A}}

\def\bD{\mathbf{D}}

\def\bW{\mathbf{W}}
\def\bX{\mathbf{X}}
\def\bY{\mathbf{Y}}
\def\bZ{\mathbf{Z}}

\def\ba{\boldsymbol{a}}
\def\bb{\boldsymbol{b}}

\def\bg{\boldsymbol{g}}
\def\bh{\boldsymbol{h}}

\def\bu{\boldsymbol{u}}
\def\bv{\boldsymbol{v}}
\def\bw{\boldsymbol{w}}
\def\bx{\boldsymbol{x}}
\def\by{\boldsymbol{y}}
\def\bz{\boldsymbol{z}}

\def\bA{\boldsymbol{A}}

\def\bD{\boldsymbol{D}}

\def\bW{\boldsymbol{W}}
\def\bX{\boldsymbol{X}}
\def\bY{\boldsymbol{Y}}
\def\bZ{\boldsymbol{Z}}

\def\normal{{\mathsf{N}}}

\def\btheta{\boldsymbol{\theta}}

\def\boeta{\boldsymbol{\eta}}

\def\id{{\boldsymbol I}}

\renewcommand{\P}{\mathbb{P}}
\newcommand{\E}{\mathbb{E}}
\newcommand{\R}{\mathbb{R}}

\newcommand{\eps}{\varepsilon}
\newcommand{\Var}{\operatorname{Var}}

\newcommand{\toP}{\overset{\P}{\to}}

\newcommand{\sign}{\operatorname{sign}}

\newcommand{\diag}{\operatorname{diag}}

\newcommand{\Unif}{\operatorname{Unif}}

\newcommand{\RN}[1]{%
  \textup{\uppercase\expandafter{\romannumeral#1}}%
}

\newcommand\iidsim{\stackrel{\mathclap{iid}}{\sim}}

\newcommand{\RNum}[1]{\uppercase\expandafter{\romannumeral #1\relax}}



\makeatletter
\newcommand*{\rom}[1]{\expandafter\@slowromancap\romannumeral #1@}
\makeatother